\documentclass{article}

\usepackage{arxiv}
\usepackage{authblk}
\usepackage[utf8]{inputenc} 
\usepackage[T1]{fontenc}    
\usepackage{hyperref}       
\usepackage{url}            
\usepackage{booktabs}       
\usepackage{amsfonts}       
\usepackage{nicefrac}       
\usepackage{microtype}      
\usepackage{lipsum}		
\usepackage{graphicx}
\usepackage{doi}

\usepackage{amsmath,amsthm}
\usepackage{graphicx}       
\usepackage{mathtools}
\usepackage{csquotes}
\usepackage{xcolor}
\usepackage{bbm}
\usepackage{bm}
\usepackage{subcaption}
\usepackage{caption}
\usepackage{enumitem}

\usepackage{multirow}

\DeclareMathOperator*{\lab}{\mathcal{Y}}

\DeclareMathOperator*{\Var}{\text{Var}}
\DeclareMathOperator*{\Cov}{\text{Cov}}

\newtheorem{theorem}{Theorem}[section]
\newtheorem*{theorem*}{Theorem}
\newtheorem{lemma}{Lemma}[section]
\newtheorem{prop}{Proposition}[section]
\newtheorem{corollary}{Corollary}[section]

\newtheorem{definition}{Definition}[section]

\newcommand{\TU}{\text{TU}}
\newcommand{\EU}{\text{EU}}
\newcommand{\AU}{\text{AU}}

\renewcommand{\vec}[1]{\boldsymbol{#1}}
\newcommand{\vtheta}{{\vec{\theta}}}
\newcommand{\given}{\, | \,}
\newcommand{\fromto}{\longrightarrow}

\newcommand{\ksimplex}{\Delta_K}
\newcommand{\ksimplextwo}[1][K]{\Delta_{#1}^{(2)}}

\newcommand*{\defeq}{\mathrel{\vcenter{\baselineskip0.5ex \lineskiplimit0pt
			\hbox{\footnotesize.}\hbox{\footnotesize.}}}%
=}

\usepackage[round]{natbib}

\title{Second-Order Uncertainty Quantification: \\ Variance-Based Measures}

\date{December 30, 2023}	

\author[1,3]{Yusuf~Sale}
\author[1,3]{Paul~Hofman}
\author[2,3]{Lisa~Wimmer}
\author[1,3]{Eyke~Hüllermeier}
\author[2,3]{Thomas~Nagler}

\affil[1]{Institute of Informatics, LMU Munich}
\affil[2]{Department of Statistics, LMU Munich}
\affil[3]{Munich Center for Machine Learning (MCML)} 
\affil[ ]{\href{mailto:<yusuf.sale@ifi.lmu.de>?Subject=Your Paper}{yusuf.sale@ifi.lmu.de}}

\renewcommand{\shorttitle}{Second-Order Uncertainty Quantification: Variance-Based Measures}

\hypersetup{
pdftitle={A template for the arxiv style},
pdfsubject={q-bio.NC, q-bio.QM},
pdfauthor={Yusuf~Sale, Paul~Hofman, Lisa~Wimmer, Eyke~Hüllermeier, Thomas~Nagler},
pdfkeywords={uncertainty quantification, variance decomposition, second-order distributions},
}

\begin{document}
\maketitle

\begin{abstract}
Uncertainty quantification is a critical aspect of machine learning models, providing important insights into the reliability of predictions and aiding the decision-making process in real-world applications. This paper proposes a novel way to use variance-based measures to quantify uncertainty on the basis of second-order distributions in classification problems. A distinctive feature of the measures is the ability to reason about uncertainties on a class-based level, which is useful in situations where nuanced decision-making is required. Recalling some properties from the literature, we highlight that the variance-based measures satisfy important (axiomatic) properties. In addition to this axiomatic approach, we present empirical results showing the measures to be effective and competitive to commonly used entropy-based measures. 
\end{abstract}


\section{Introduction}

Thanks to novel methods of impressive predictive power, machine learning (ML) is becoming more and more ingrained into peoples' lives.
It increasingly supports human decision-making processes in fields ranging from healthcare \citep{lambrou2010reliable, senge_2014_ReliableClassificationLearning, yang2009using, mobiny_2021_DropConnectEffectiveModeling} and autonomous driving \citep{michelmore_2018_EvaluatingUncertaintyQuantification} to socio-technical systems \citep{varshney2016engineering,varshney2017safety}.
The safety requirements of such applications trigger an urgent need to report \textit{uncertainty} alongside model predictions \citep{hullermeier2021aleatoric}.
Meaningful uncertainty estimates are indispensable for trust in ML-assisted decisions as they signal when a prediction is not confident enough to be relied upon.

In order to address predictive uncertainty about a query instance $\boldsymbol{x}_q \in \mathcal{X}$, it is often crucial to identify its source.
For one, uncertainty can arise from the data through inherent stochasticity of the data-generating process, omitted variables or measurement errors \citep{gruber2023sources}.
As such, this \emph{aleatoric} uncertainty (AU) is a fixed but unknown quantity.
In addition, a lack of knowledge about the best way to model the data-generating process induces \emph{epistemic} uncertainty (EU).
Under the assumption that the model class is correctly specified, collecting enough information will reduce the EU until it vanishes in the limit of infinite data \citep{hullermeier2021aleatoric}.
The attribution of uncertainty to its sources can inform decisions in various ways.
For instance, it might help practitioners realize that gathering more data will be useless when only AU is present, or guide sequential learning processes like active learning \citep{shelmanov_2021_ActiveLearningSequence, nguyen2022measure} and Bayesian optimization \citep{HOFFER2023109279, stanton2023bayesian} by seeking out promising parts of the search space that can be explored to reduce EU while avoiding uninformative areas with high AU.

Quantifying both AU and EU necessitates a meaningful uncertainty \emph{representation}.
In supervised learning, we consider hypotheses $h \in \mathcal{H}$ that map a query instance $\boldsymbol{x}_q \in \mathcal{X}$ to a probability measure $p$ on $(\lab, \sigma(\lab))$, where $\lab$ denotes the outcome space and $\sigma(\mathcal{Y})$ a suitable $\sigma$-algebra on $\lab$.
For each $\boldsymbol{x}_q$, $p$ is an estimate of the ground-truth probability measure $p^*$ on $(\lab, \sigma(\lab))$.
When predicting a single numerical value or class label, $p$ will be a Dirac measure.
The case of probabilistic classification, which we study in this work, is more informative in the sense that a posterior probability is associated with all possible class labels, giving rise to a natural notion of AU around the outcome $y \in \mathcal{Y}$. 
However, such probabilistic expressions are point predictions derived from a single hypothesis $h$ learned on the training data.
Since all other candidates in the hypothesis space $\mathcal{H}$ are discarded in the process, $p$ cannot, by design, represent EU \citep{hullermeier2021aleatoric}.

Expressing EU requires a further level of uncertainty representation. 
A straightforward approach is to impose a \emph{second-order} distribution over $(\lab, \sigma(\lab))$, effectively assigning a probability to each candidate first-order distribution $p$, and equate the dispersion of this distribution with EU.
Both the classical Bayesian paradigm \citep{gelman2013bayesian} and evidential deep learning (EDL) methods \citep{UlmerHF23} follow this idea.
As an alternative, methods founded on more general theories of probability, such as imprecise probabilities or credal sets \citep{walley1991,augustin2014introduction}, have been considered \citep{corani2012bayesian}.

It has become the \emph{de facto} standard in probabilistic classification to rely on a bi-level distributional approach and use measures based on Shannon entropy to quantify the different uncertainty components (e.g., \citet{kendall2017uncertainties, smithGal2018, Charpentier2022NaturalPN}).
In this narrative, the entropy of the categorical output distribution over class labels is associated with the total predictive uncertainty for a query instance $\boldsymbol{x}_q$.
By a well-known result from information theory, this quantity decomposes additively into conditional entropy (representing AU) and mutual information (representing EU).
While this set of measures may seem neat and intuitive, \citet{wimmer2023quantifying} recently pointed out that it does not fulfill certain properties that one would actually expect.
A key criticism concerns the upper-boundedness of entropy-based measures, which, in combination with the additive relation instilled by information theory, leads to a systematic estimation bias for AU as soon as EU is present.
In this paper, we propose to employ measures based on variance instead, and show that they overcome the drawbacks of the entropy-based approach without sacrificing practical applicability.

The idea of using variance-based uncertainty measures is not completely new.
In particular, a decomposition derived from the law of total variance has been used in regression problems for quite some time \citep{depeweg2018decomposition}.
Moreover, \citet{duan2023evidential} introduce variance-based uncertainty measures for classification, yet they motivate this from the EDL paradigm and do not discuss any theoretical properties.
The present work approaches variance-based uncertainty quantification from a more fundamental point of view.
Our contributions are as follows: 
\begin{itemize}
    \item We introduce variance-based measures for second-order uncertainty quantification, providing an alternative to entropy-based measures that have recently faced criticism in the literature. Further, we demonstrate that our proposed measures satisfy a  set of desirable properties, enhancing their theoretical appeal. 
    \item We outline a label-wise perspective enabling reasoning about uncertainty at the individual class level, aiding decision-making especially in scenarios where the stakes of incorrect predictions vary across classes.
    \item We present empirical results that underscore the practical effectiveness of our measures, demonstrating their effectiveness in a variety of scenarios and highlighting their potential for application in downstream tasks. 
\end{itemize}

\section{Quantifying Second-Order Uncertainty}
    In the following, we will be exclusively concerned with the supervised classification scenario. 
    We refer to $\mathcal{X}$ as \textit{instance} space, and we assume categorical target variables from a finite \textit{label} space $\lab = \{ y_1, \ldots, y_K \}$, where $K \in \mathbb{N}_{\geq 2}$. 
    Thus, each instance $\vec{x} \in \mathcal{X}$ is associated with a conditional distribution on the measurable space $(\lab, 2^{\lab})$, such that $\theta_k \defeq p(y_k \given \boldsymbol{x} )$ is the probability to observe label $y_k \in \mathcal{Y}$ given $\vec{x} \in \mathcal{X}$.
    Further, we note that the set of all probability measures on $(\lab, 2^{\lab})$ can be identified with the $(K-1)$-simplex
    $\ksimplex \defeq \left \{ \vtheta = (\theta_1, \ldots , \theta_K) \in [0,1]^K ~ \given~ \| \vtheta \|_1 = 1 \right \}$.
    Consequently, for each $\vtheta \in \ksimplex$, an associated degree of aleatoric uncertainty (AU) can be calculated. To effectively represent epistemic uncertainty (EU), it is necessary for the learner to express its uncertainty regarding $\vtheta$. This is achievable through a second-order probability distribution over the first-order distributions $\vtheta$.

    As already stated in the introductory section, two popular methods to obtain a second-order (predictive) distribution are Bayesian inference and Evidential Deep Learning (EDL). 
    In both approaches, we arrive at a second-order predictor $h_2: \mathcal{X} \fromto \ksimplextwo,$ where $\ksimplex^{(2)}$ denotes the set of all probability measures on $(\ksimplex, \sigma(\ksimplex))$; we call $Q \in \ksimplextwo$ a second-order distribution. For the sake of simplicity, we subsequently omit the conditioning on the query instance $\vec{x}_q$. Hence, given an instance $\vec{x}_q \in \mathcal{X}$, $Q \in \ksimplextwo$ denotes our current probabilistic knowledge about $\vtheta$, i.e., $Q(\vtheta)$ is the probability (density) of $\vtheta \in \ksimplextwo$.
    In the remainder of this paper, we assume that a second-order distribution $Q$ is already provided.

    Given an uncertainty representation in terms of a second-order distribution $Q \in \ksimplextwo$, the subsequent question is how to suitably quantify total, aleatoric and epistemic uncertainty associated with such a second-order uncertainty assessment. Prevalent approaches to uncertainty quantification in the literature \citep{houlsby_2011_BayesianActiveLearning, gal_2016_UncertaintyDeepLearning, depeweg2018decomposition, smithGal2018, mobiny_2021_DropConnectEffectiveModeling} rely on information-theoretic measures derived from Shannon entropy \citep{shannon1948mathematical}. In the following section, we will revisit these commonly accepted entropy-based uncertainty measures and discuss crucial properties that any uncertainty measure should possess.

\subsection{Entropy-Based Measures}
\label{subsec:entropy}
	We begin by revisiting the arguably most common entropy-based approach in machine learning for quantifying predictive uncertainty represented by a second-order probability $Q \in \ksimplextwo$.
    This approach leverages (Shannon) entropy and its connection to mutual information and conditional entropy to quantify total, aleatoric, and epistemic uncertainties associated with $Q$.

    Shannon entropy for a (first-order) probability distribution $\vtheta \in \ksimplex$ is given by
    \begin{align}
		H(\vtheta) \defeq - \sum_{k = 1}^{K} \theta_k \log_2 \theta_k \, .
		\label{eq:entropy}
	\end{align}
    Now, let $Y: \Omega \fromto \lab$ be a (discrete) random variable, and denote by $\vtheta_{Y}$ its corresponding distribution on the measurable space $(\lab, 2^{\lab})$. Then, we can analogously define the entropy of the random variable $Y$ by simply replacing $\theta_k$ in \eqref{eq:entropy} by the respective push-forward measure of $Y$.
    Thus, entropy has established itself as an accepted uncertainty measure due to both appealing theoretical properties it possesses and the intuitive interpretation as a measure of uncertainty. Particularly, it measures the degree of uniformity of the distribution of a random variable. 

	Subsequently, following the notation convention in \citet{wimmer2023quantifying}, we assume that $\Theta \sim Q.$ Therefore, $\Theta: \Omega \fromto \ksimplex$ is a random first-order distribution distributed according to a second-order distribution $Q$, and consequently taking values $\Theta(\omega) = \vtheta$ in the $(K-1)$-simplex $\ksimplex$.

 Given a second-order distribution $Q$, we can consider its expectation given by 
	\begin{align}
		\bar{\vtheta} = \mathbb{E}_Q[\Theta] = \int_{\ksimplex} \vtheta \; \mathrm{d}Q(\vtheta) \, ,
		\label{eq:aggregation}
	\end{align} 
	which yields a probability measure $\bar{\vtheta}$ on $(\lab, 2^{\lab})$. This measure corresponds to the distribution of $Y$ when we view it as generated from first sampling $\Theta \sim Q$ and then $Y$ according to $\Theta$. Then, it is natural to define the measure of total uncertainty (TU) as the entropy \eqref{eq:entropy} of $\bar{\vtheta}$:
	\begin{align}
		\TU(Q) = H\left( \mathbb{E}_Q[\Theta]   \right) \, ,
		\label{tu:entropy}
	\end{align}
    with $H(\cdot)$ as defined in \eqref{eq:entropy}.
    Similarly, aleatoric uncertainty (AU) can be defined in terms of \textit{conditional entropy} $H(Y|\Theta)$:
	\begin{align}
		\AU(Q) = \mathbb{E}_Q[ H(Y| \Theta) ] = \int_{\ksimplex} H(\vtheta) \; \mathrm{d}Q(\vtheta) \, .
		\label{au:entropy}
	\end{align}
    By fixing a first-order distribution $\vtheta \in \ksimplex$, all EU is essentially removed and only AU remains. However, as $\vtheta$ is not precisely known, we take the expectation with respect to the second-order distribution.
    The measure of epistemic uncertainty is particularly inspired by the widely known additive decomposition of entropy into conditional entropy and \textit{mutual information} (see also Section 2.4 in \cite{cover1999elements}). This is expressed as follows:
	\begin{align}
		\underbrace{H(Y)}_{\textnormal{entropy}} = \underbrace{H(Y\, | \, \Theta)}_{\textnormal{conditional entropy}}+ \underbrace{I(Y, \Theta)}_{\textnormal{mutual information}} 
		\label{eu:entropy}
	\end{align}
	Rearranging \eqref{eu:entropy} for mutual information yields a measure of epistemic uncertainty
	\begin{align} \label{eu:entropy_2}
		\begin{split}
			\EU(Q) = I(Y, \Theta) &= H(Y)- H(Y\, | \, \Theta) \\[0.2cm]
                    &= \mathbb{E}_Q[D_{KL}(\Theta \, \| \, \bar{\vtheta})] \, ,
		\end{split}
	\end{align}
	where $D_{KL}(\cdot \, \| \, \cdot)$ denotes the Kullback-Leibler (KL) divergence \citep{kullback1951information}.
    While entropy, mutual information, and conditional entropy provide meaningful interpretations for quantifying uncertainties  within first-order predictive distributions, the suitability of these entropy-based measures for second-order quantification has been critically challenged by \cite{wimmer2023quantifying}.
    This criticism was substantiated on the basis of a set of desirable properties, which will be discussed in detail next.

\subsection{Desirable Properties}
\label{subsection:axioms}
    Before introducing variance-based uncertainty measures, it is important to first consider desirable properties that any suitable uncertainty measure ought to fulfill. In the (uncertainty) literature it is standard practice to establish measures based on a set of axioms \citep{pal1993uncertainty, bronevich2008axioms}. Such an axiomatic approach was also adapted in the recent machine learning literature \citep{hullermeier2022quantification, sale2023secondorder, sale2023volume}.
    To this end, we revisit the axioms outlined by \cite{wimmer2023quantifying}, while also taking into account recently proposed properties that further refine the understanding of what constitutes a suitable measure of second-order uncertainty \citep{sale2023secondorder}.
    Before discussing the proposed axioms, we first provide some mathematical preliminaries.
    \begin{definition} \label{def:shifts}
    Let $\vec{X} \sim Q,\, \vec{X}^{\prime} \sim Q^{\prime}$ be two random vectors, where we have that $Q, Q^{\prime} \in \ksimplextwo$. Denote by $\sigma(\vec{X})$ the $\sigma$-algebra generated by the random vector $\vec{X}$.
    Then we call $Q^{\prime}$
    \begin{itemize}
        \item[(i)] a mean-preserving spread of $Q$, iff $\vec{X}^\prime \overset{d}{=} \vec{X} + \vec{Z}$, for some random variable $\vec{Z}$ with $\mathbb{E}[\vec{Z} \given \sigma(\vec{X})] = 0$ almost surely (a.s.) and $ \max_k \Var(Z_k) > 0$. 
        \item[(ii)] a spread-preserving location shift of $Q$, iff $\vec{X}^\prime \overset{d}{=} \vec{X} + \vec{z}$, where $\vec{z} \neq 0$ is a constant. 
        \item[(iii)] a spread-preserving center-shift of $Q$, iff it is a spread-preserving location shift with $\mathbb E[\vec{X}'] = \lambda \mathbb E[\vec{X}] + (1 - \lambda) (1/K, \dots, 1/K)^\top$ for some $\lambda \in (0, 1)$.
    \end{itemize}
    Note that for definitions (ii) and (iii) it should be guaranteed that the shifted probability measure $Q^{\prime}$ remains valid within its support. 
    \end{definition}
    Now, let $\TU$, $\AU$, and $\EU$ denote, respectively, measures $\ksimplextwo \to \mathbb{R}_{\geq 0}$ of total, aleatoric, and epistemic uncertainties associated with a second-order uncertainty representation $Q \in \ksimplextwo$.
    \cite{wimmer2023quantifying} propose that any uncertainty measure should fulfill (at least) the following set of axioms:

\begin{itemize}
    \item[A0] $\TU$, $\AU$, and $\EU$ are non-negative.
    \item[A1] $\EU(Q) = 0$, iff $Q = \delta_{\vtheta}$, where $\delta_{\vtheta}$ denotes the Dirac measure on $\vtheta \in \ksimplex$.
    \item[A2] $\EU$ and $\TU$ are maximal for $Q$ being the uniform distribution on $\ksimplex$.
    \item[A3] If $Q'$ is a mean-preserving spread of $Q$, then $\EU(Q') \geq \EU(Q)$ (weak version) or $\EU(Q') > \EU(Q)$ (strict version), the same holds for TU.
    \item[A4] If $Q'$ is a center-shift of $Q$, then $\AU(Q') \geq \AU(Q)$ (weak version) or $\AU(Q') > \AU(Q)$ (strict version), the same holds for $\TU$.
    \item[A5] If $Q'$ is a spread-preserving location shift of $Q$, then $\EU(Q') = \EU(Q)$.
\end{itemize}
Axiom A0 is an obvious requirement, ensuring that such measures reflect a degree of uncertainty without implying the absence of information or negative uncertainty, which would be conceptually unsound. 
Axiom A1 addresses the behavior of $\EU$ in the context of Dirac measures, where a Dirac measure $\delta_{\vtheta}$ represents a scenario of complete certainty about $\vtheta \in \ksimplex$. The vanishing of $\EU$ in this context aligns with the intuitive understanding that epistemic uncertainty should be zero when there is absolute certainty about the true underlying model.
Further, Axiom A2 considers the condition under which $\EU$ and $\TU$ attain their maximal values, specifically when $Q$ is the uniform distribution on $\ksimplex$. This reflects situations of maximum uncertainty or ignorance, where the lack of knowledge about any specific outcome $\vtheta \in \ksimplex$ leads to the highest level of uncertainty. As we will discuss later, this axiom is not without controversy, particularly in the fields of statistics and decision theory.
Axiom A3 encapsulates the idea that spreading a distribution while preserving its mean should not reduce, and might increase, the epistemic (and thus, total) uncertainty. It underscores the notion that increased dispersion (while maintaining the mean) is associated with higher uncertainty, a concept that is central in statistics.\footnote{According to \cite{wimmer2023quantifying}, this axiom is violated by the entropy-based approach. As we show in Proposition \ref{corrigendum}, however, this claim seems to be incorrect.}
Conversely, leaving the dispersion constant but shifting the distribution closer to the barycenter of the simplex, thereby expressing a belief about $\vtheta$ that is closer to uniform, should be reflected by an increase in AU (Axiom A4).
Lastly, Axiom A5 asserts that a spread-preserving location shift, which alters the distribution's location without affecting its spread, should leave the epistemic uncertainty unchanged. This property highlights the distinct nature of epistemic uncertainty, which is sensible to the spread of the distribution rather than its location \citep{eyke_new}. 

Taking into consideration recently proposed criteria for measures of second-order uncertainty \citep{sale2023secondorder}, we expand the existing set of Axioms A0--A5 by introducing two additional properties. For the set of all mixtures of second-order Dirac measures on first-order Dirac measures we write
 \begin{align*} 
\Delta_{\delta_m} = \Big\{ \delta_m \in  \ksimplextwo \, : \,  
\delta_m =   \sum_{y \in \lab} \lambda_y \cdot \delta_{\delta_{y}}, \, \sum_{y \in \lab} \lambda_y = 1 \Big\} \,,
\end{align*}
where $\delta_{\delta_y}$ denotes the second-order Dirac measure on $\delta_y \in \ksimplex$ for $y \in \lab$. 
Each element in this set should arguably have no aleatoric uncertainty, such that we postulate the following Axiom A6. 
\begin{itemize}
    \item[A6] $\AU(\delta_{m}) = 0$ holds for any $\delta_m \in \Delta_{\delta_m}$ \, . 
\end{itemize}
Now, let $\lab_1$ and $\lab_2$ be partitions of $\lab$ and $Q \in \ksimplextwo$, and further denote by $Q_{|\lab_i}$ the corresponding marginalized distribution for $i \in \{1, 2\}$. 

\begin{itemize}
\item[A7] $\TU_{\lab}(Q) \leq \TU_{\lab_1}(Q_{|\lab_1}) + \TU_{\lab_2}(Q_{|\lab_2})$, the same holds for $\AU$ and $\EU$.
\end{itemize}

Axiom A7 guarantees that the total uncertainty of a second-order distribution is bounded by the sum of total uncertainties of its corresponding marginalizations.

\section{Variance-Based Measures}
\label{sec:proposal}
In this section 
we introduce variance-based uncertainty measures for classification tasks. 
To this end, we leverage the \textit{law of total variance}: for any random variable  $X \in L^2(\Omega, \mathcal{A}, P)$ and sub-$\sigma$-algebra  $\mathcal{F} \subseteq \mathcal{A}$, 
\begin{align}
    \Var(X) = \mathbb{E}[ \underbrace{\mathbb{E}[(X - \mathbb{E}[X \given \mathcal{F}])^2 \given \mathcal{F}]}_{\eqqcolon \Var(X \given \mathcal{F}) }]   + \Var(\mathbb{E}[X \given \mathcal{F}])   \, .
    \label{eq:lawvariance}
\end{align}

Unlike entropy, variance is not directly applicable to the case of a categorical variable. Therefore, we propose to measure uncertainty in a label-wise manner, and to obtain the overall uncertainty associated with a prediction in terms of the sum of the uncertainties on the individual labels \citep{duan2023evidential}. 
Let us denote by $Y: \Omega \fromto \{0,1\}^K$ the $K$-dimensional random vector indicating the presence or absence of a particular label $y_k \in \mathcal{Y}$ for $k \in \{1, \dots, K\}$. Further, define $\Theta_k \coloneqq P(Y_k = 1)$ and assume that the random vector $\Theta = (\Theta_1, \dots, \Theta_K)$ is distributed according to a second-order distribution $Q \in \ksimplextwo$, i.e., $\Theta \sim Q$. Moreover, let $Q_k$ denote the marginal distribution of the random variable $Y_k$, such that $Y_k \sim Q_k$ for $k \in \{1, \dots, K\}$.
Then, by the law of total variance \eqref{eq:lawvariance} and observing that $\sigma(\Theta_k) \subseteq \mathcal{F}$ for any $k \in \{1,\dots,K\}$, we get
\begin{align*}
 \Var(Y_k)&= \mathbb{E}[\Var(Y_k \given \sigma(\Theta_k)) ] + \Var(\mathbb{E}[Y_k \given \sigma(\Theta_k)] ) \\[0.2cm]
   &= \mathbb{E}[\Theta_k \cdot (1 - \Theta_k)] + \Var(\Theta_k) \, .
\end{align*}
This equality suggests an alternative definition of total uncertainty and its (additive) decomposition into an aleatoric and an epistemic part:
\begin{itemize}
    \item The (label-wise) \textbf{total uncertainty} $(\TU_k)$ is given by $\Var(Y_k)$. It takes the role of Shannon entropy in the entropy-based approach. Just like the latter corresponds to the expected log-loss (of the risk-minimizing prediction $\mathbb{E}[\Theta_k]$), $\Var(Y_k)$ is the expected squared-error loss. 
    
    \item  Label-wise \textbf{aleatoric uncertainty} $(\AU_k)$ is given by $ \mathbb{E}[\Theta_k \cdot (1 - \Theta_k)]$, capturing the inherent randomness in the outcome of each $Y_k$; it reflects how much the occurrence of each label is subject to chance, particularly in cases where predicting the label is inherently uncertain due to the variability of the data-generating process itself. Just like conditional entropy, it can be seen as the (expected) ``conditional variance'' of $Y_k$ and corresponds to the expected squared-error loss provided the true value of $\Theta_k$ is given. 

    \item Label-wise \textbf{epistemic uncertainty} $(\EU_k)$ is quantified by $\Var(\Theta_k)$. It captures the dispersion of $\Theta_k$, indicating a level of uncertainty due to limited data or imperfect modeling. Just like mutual information corresponds to the expected reduction in log-loss achieved by the knowledge of $\Theta_k$, $\Var(\Theta_k)$ is the expected reduction of squared-error loss. 
\end{itemize}

We note that the label-wise perspective\,---\,which is of course not reserved for variance but could in principle also be realized for other measures of uncertainty, including entropy\,---\,is particularly useful in settings where decisions following the prediction of different labels are associated with unequal costs. For instance, when predicting the sub-type of a certain medical condition, with costly treatment administered at occurrence of one of the sub-types, the marginal uncertainty about this category might be of particular interest. We show some illustrative experimental results in Section~\ref{sub:classwise}.

Nevertheless, certain scenarios call for a global perspective on predictive uncertainty. 
To obtain corresponding measures, the most obvious idea is to define total, aleatoric, and epistemic uncertainty associated with a second-order distribution $Q \in \ksimplextwo$ by summing over all label-wise uncertainties, taking their relative importance into account:
\begin{align}
\TU(Q) &\coloneqq \sum_{k = 1}^{K} w_k \Var(Y_k)
\label{tu:variance}\\
\AU(Q) &\coloneqq \sum_{k = 1}^{K} w_k \mathbb{E}[\Theta_k \cdot (1 - \Theta_k)]
\label{au:variance}\\
\EU(Q) &\coloneqq \sum_{k = 1}^{K} w_k \Var(\Theta_k),
\label{eu:variance}
\end{align}
with importance weights $w_1, \dots, w_K \ge 0$.

This global view is crucial in scenarios where understanding the overall uncertainty is key to making informed decisions. For instance, $\TU$ serves as an indicator for the overall reliability of the model. Meanwhile, $\AU$ and 
$\EU$ distinguish between the uncertainty arising from the data's inherent variability and that stemming from the model's knowledge limitations, respectively. 

We now demonstrate that variance-based measures \eqref{tu:variance}, \eqref{au:variance}, and \eqref{eu:variance} satisfy a number of important properties. The proof of Theorem \ref{thm:axioms} is provided in Appendix \ref{appendix:proofs}.

\begin{theorem}
\label{thm:axioms}
Variance-based measures \eqref{tu:variance}, \eqref{au:variance}, and \eqref{eu:variance} satisfy Axioms A0, A1, A3 (strict version), A5--A7 for any $w_1, \dots, w_K > 0$, and A4 (strict version) if additionally $w_1 = \dots = w_K$. 
\end{theorem}
We will provide insights into how variance-based measures behave with respect to the properties outlined in Section \ref{subsection:axioms}, using the illustrative examples presented in Figure \ref{fig:examples}. Here, we assume different second-order distributions $Q$ over the parameter $\theta$ of a Bernoulli distribution.
\begin{figure}[t]
    \centering
    \begin{subfigure}{0.32\textwidth}
        \centering
        \includegraphics[width=0.8\linewidth]{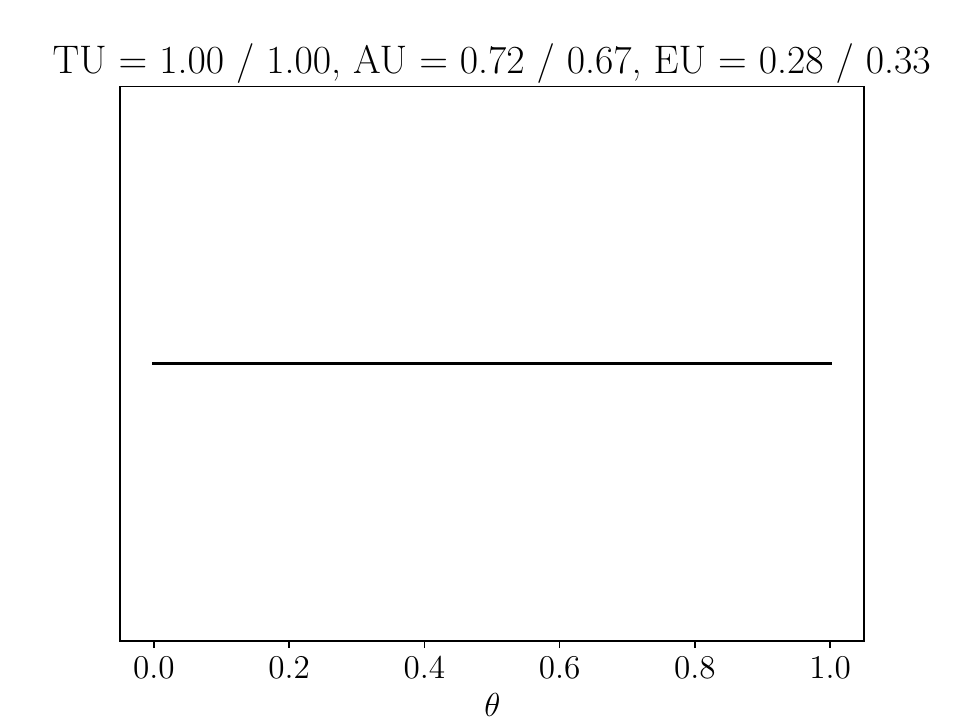}
        \caption{$\mathcal{U}[0,1]$}
    \end{subfigure}
    \begin{subfigure}{0.32\textwidth}
        \centering
        \includegraphics[width=0.8\linewidth]{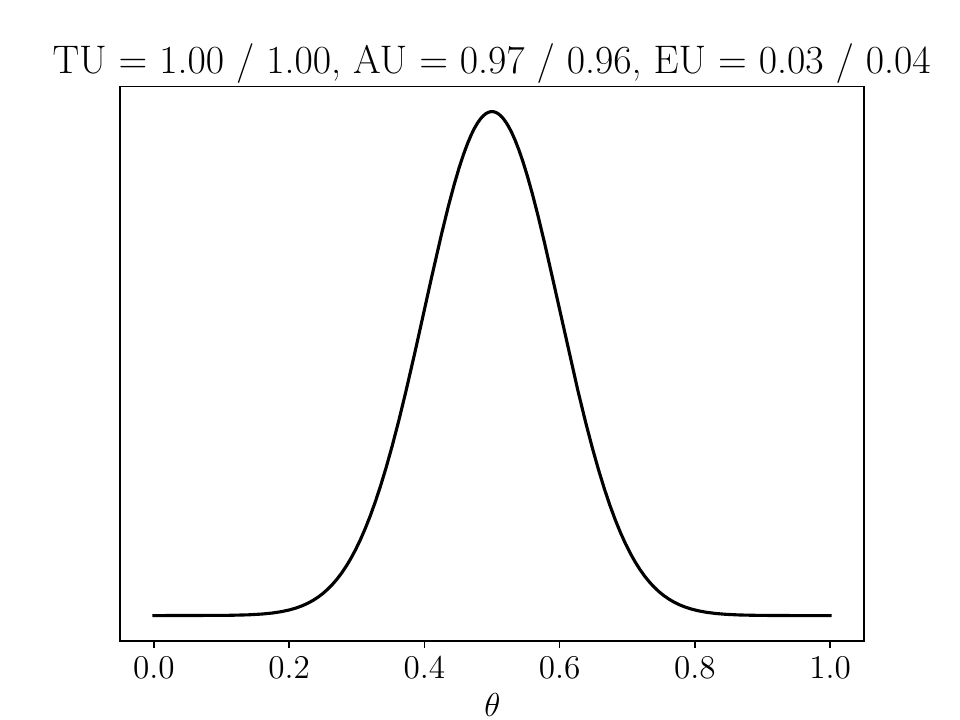}
        \caption{$\mathcal{N}(0.5,0.1)$}
    \end{subfigure}
    \begin{subfigure}{0.32\textwidth}
        \centering
        \includegraphics[width=0.8\linewidth]{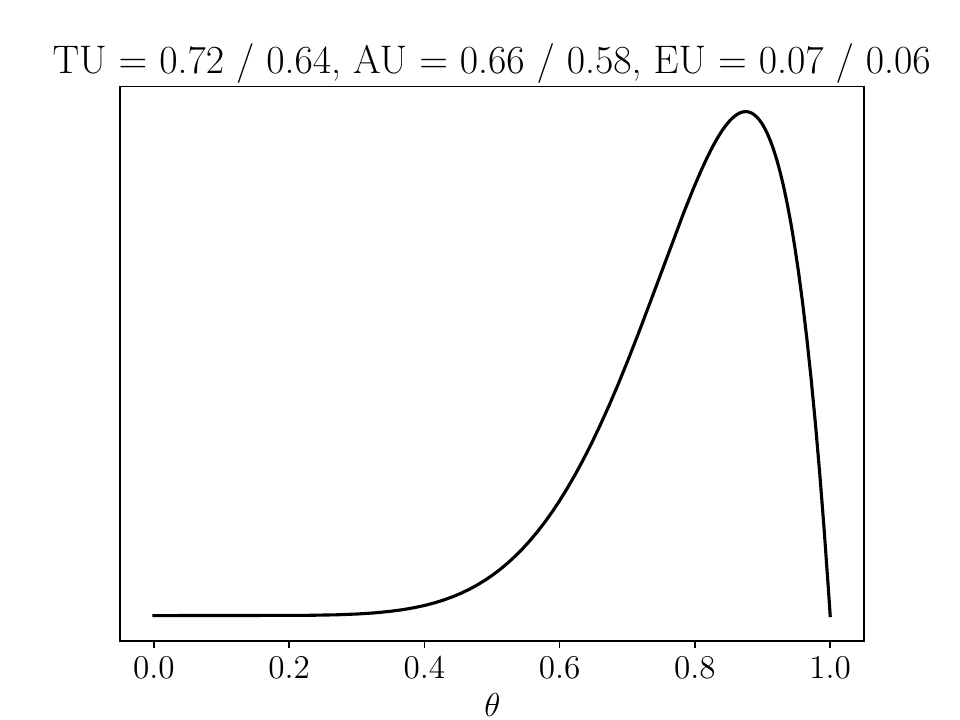}
        \caption{$Beta(8,2)$}
    \end{subfigure}
    \vskip\baselineskip
    \begin{subfigure}{0.32\textwidth}
        \centering
        \includegraphics[width=0.8\linewidth]{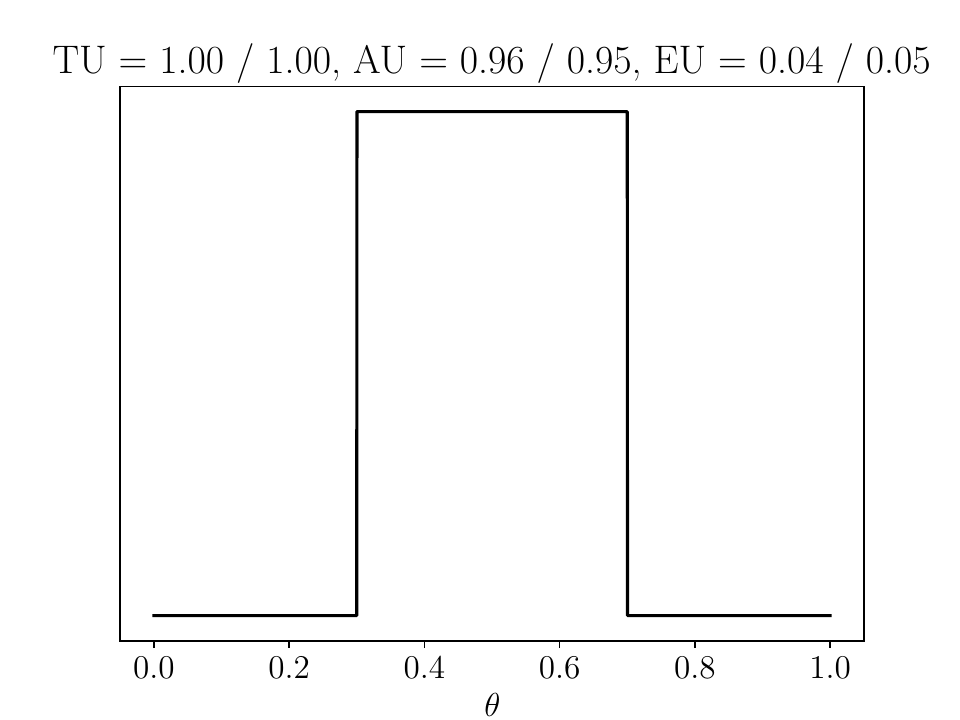}
        \caption{$\mathcal{U}[0.3, 0.7]$}
        \label{subfig:d}
    \end{subfigure}
    \begin{subfigure}{0.32\textwidth}
        \centering
        \includegraphics[width=0.8\linewidth]{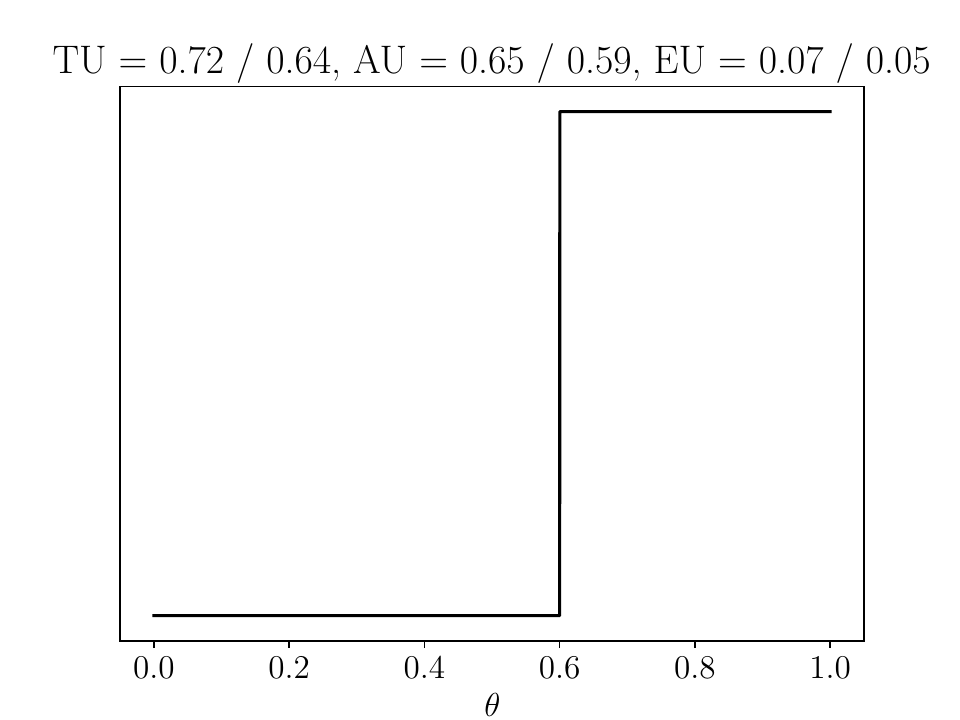}
        \caption{$\mathcal{U}[0.6,1.0]$}
        \label{subfig:e}
    \end{subfigure}
    \begin{subfigure}{0.32\textwidth}
        \centering
        \includegraphics[width=0.8\linewidth]{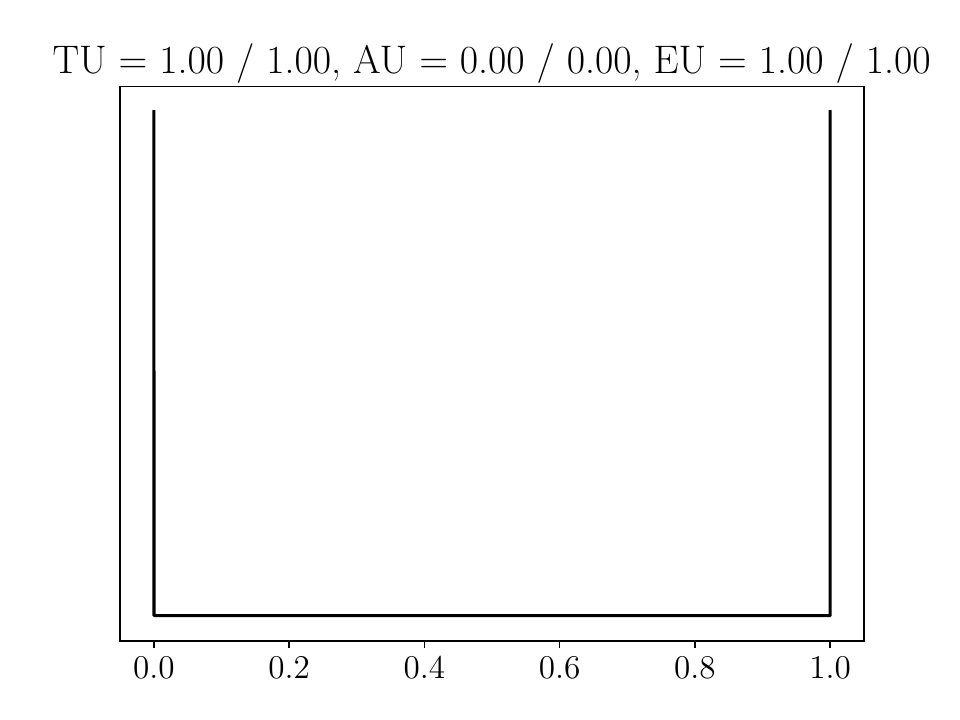}
        \caption{$\frac{1}{2}\delta_0 + \frac{1}{2}\delta_1$}
    \end{subfigure}
    \caption{Second-order distributions over the parameter $\theta$ of a (first-order) Bernoulli distribution with the respective uncertainties (entropy-based / variance-based). Note that the variance-based measures are normalized to $[0,1]$.}
    \label{fig:examples}
\end{figure}

First we note that the variance-based measures, like the entropy-based measures, do not satisfy A2\footnote{This property is in general hard to satisfy and not without controversy.}.
However, A5 is satisfied by the former but not by the latter. This axiom is important as it ensures that EU strictly quantifies the amount of knowledge an agent has, and does not conflate this with the AU implied by this knowledge. Consider, for example, Figures \ref{subfig:d}, \ref{subfig:e}: although the knowledge of the agent does not change between these two situations, the EU does change for entropy\,---\,even in a counter-intuitive direction (if at all, then the uncertainty should decrease when shifting from the center to the boundary). 

Axiom A4 is only satisfied for equal weights. This is a feature, not a deficiency: the `center' $(1 / K, \dots, 1 / K)$ of the simplex loses its meaning as the canonical point of maximal AU when we weigh uncertainty about individual labels differently. For potentially unequal weights, we can prove a more general property, however.
\begin{lemma} \label{lem:concavity}
    For any $w_1, \dots, w_K > 0$, the function $V(\vec{\theta}) = \sum_{k = 1}^K w_k \theta_k(1 - \theta_k)$ is strictly concave on $\Delta_K$ with unique maximizer $\vec{\beta}$ defined by
    \begin{align*}
        \beta_k = \frac 1 2\biggl(1 - \frac{(K - 2) / w_k}{ \sum_{k = 1}^K 1/ w_k }\biggr).
    \end{align*}
\end{lemma}

\begin{corollary} \label{cor:A4}
    Let $w_1, \dots, w_K > 0$ and $\vec{\beta} = \arg\max_{\vec{\theta} \in \Delta_K} \sum_{k = 1} w_k \theta_k(1 - \theta_k)$ and define $Q'$ as a spread-preserving location shift such that $\mathbb E[\vec{X}'] = \lambda \mathbb E[\vec{X}] + (1 - \lambda) \vec{\beta}$ for some $\lambda \in (0, 1)$.
    Then $\mathrm{AU}(Q') > \mathrm{AU}(Q)$ and $\mathrm{TU}(Q') > \mathrm{TU}(Q)$.
\end{corollary}
The inequality for TU follows directly from Lemma \ref{lem:concavity}. The inequality for AU is then implied by $\mathrm{EU}(Q') = \mathrm{EU}(Q)$ (Axiom A5).
One can verify that for $w_1 = \dots = w_K$, we have $\vec{\beta} = (1, \dots, 1)/ K$, so $Q'$ is a spread-preserving center-shift in the sense of Definition \ref{def:shifts}.

\section{Experiments}
In this section we perform experiments to showcase the effectiveness of the proposed measures. Due to the fundamental lack of a ground truth in studying uncertainty (as opposed to predictive performance where ground-truth labels are usually available), it is challenging to assess the quality of the uncertainty estimates. As such, we study the effectiveness of the proposed measures in two different tasks: prediction with abstention and out-of-distribution (OoD) detection. The code for the experimental section is available in a GitHub repository. We compare variance-based measures with the virtually gold-standard entropy uncertainty measures. Further experiments as well as details on model architecture and training setup can be found in Appendix \ref{appendix:exp_details} and Appendix \ref{appendix:exp}, respectively. 

\subsection{Accuracy-Rejection Curves}
\label{subsec:arcs}
We generate Accuracy-Rejection Curves (ARCs) by rejecting the predictions for instances on which the predictor is most uncertain and computing the accuracy on the remaining subset \citep{huhn2009}. With a good uncertainty quantification method, the accuracy should monotonically increase as the percentage of instances for which the model is allowed to abstain increases because it misclassifies instances with low uncertainty less often than instances with high uncertainty.
To approximate the second-order distribution, we train an ensemble of five neural networks on the CIFAR10 \citep{krizhevsky2009learning} and SVHN \citep{netzer2011reading} data sets. Further experiments using additional data sets can be found in Appendix \ref{app:arc}. 
We compare the proposed variance-based uncertainty measures to the entropy-based baselines (cf.\ Section \ref{subsec:entropy}) as used in the Bayesian setting. Figure~\ref{fig:arcs} shows the accuracy-rejection curves for the CIFAR10 data set (left) and the SVHN data set (right). The accuracies are reported as the mean over five independent runs, and the standard deviation is depicted by the shaded area. The ARCs for all uncertainty measures closely align with the respective entropy-based measures and exhibit similar qualitative behaviors. 
Note that our goal is not to demonstrate that variance-based measures always perform better than their entropy equivalents. Instead, our focus is on illustrating that
they fulfill many desirable properties and are highly competitive in downstream task applications. 

\begin{figure}[htbp]
\centering

\begin{subfigure}{0.45\textwidth}
\includegraphics[width=\linewidth]{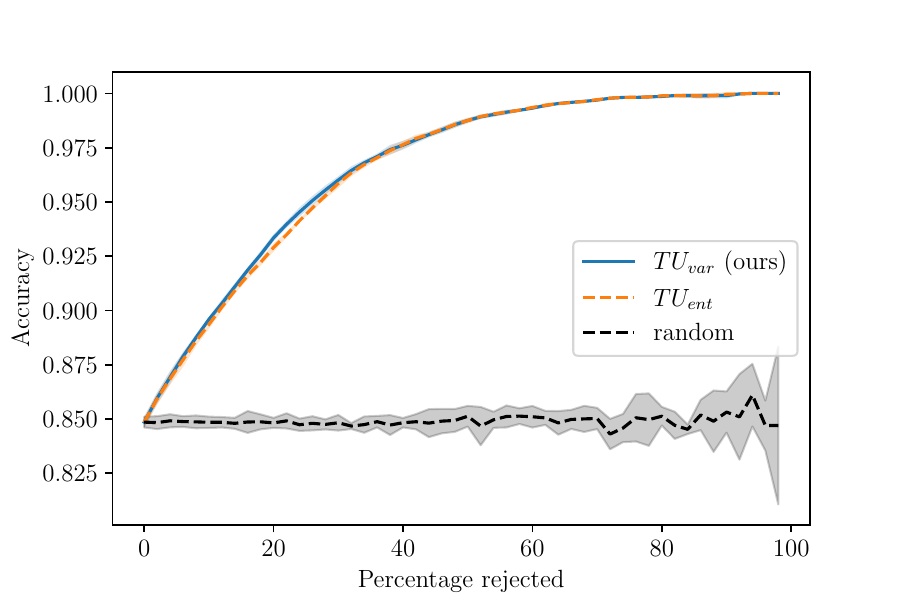}
\caption{CIFAR10 (TU)}
\end{subfigure}
\hfill
\begin{subfigure}{0.45\textwidth}
\includegraphics[width=\linewidth]{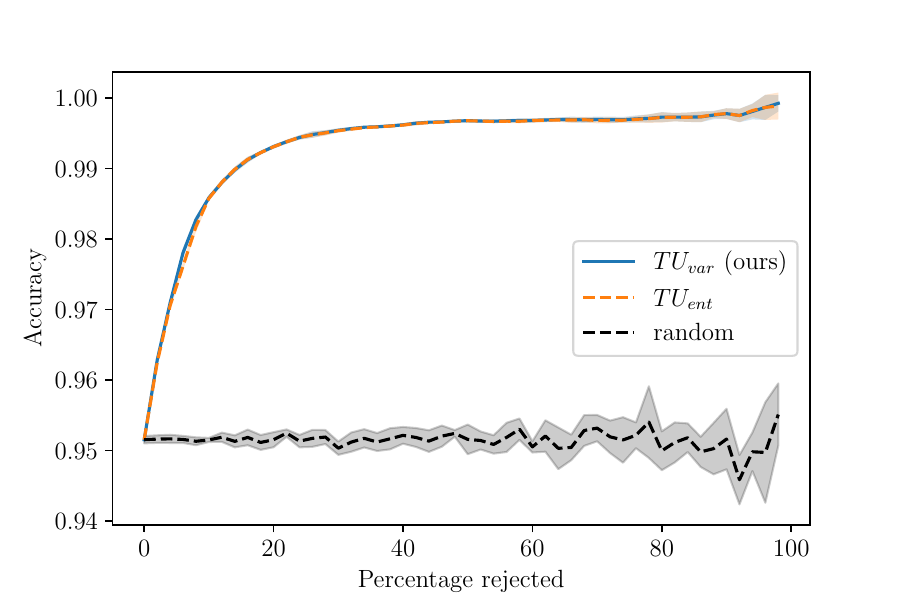}
\caption{SVHN (TU)}
\end{subfigure}

\begin{subfigure}{0.45\textwidth}
\includegraphics[width=\linewidth]{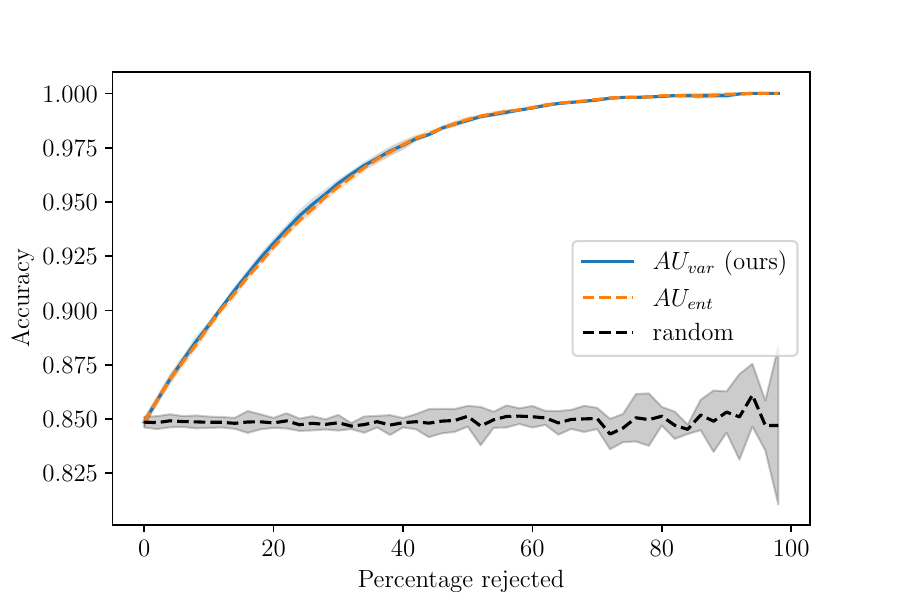}
\caption{CIFAR10 (AU)}
\end{subfigure}
\hfill
\begin{subfigure}{0.45\textwidth}
\includegraphics[width=\linewidth]{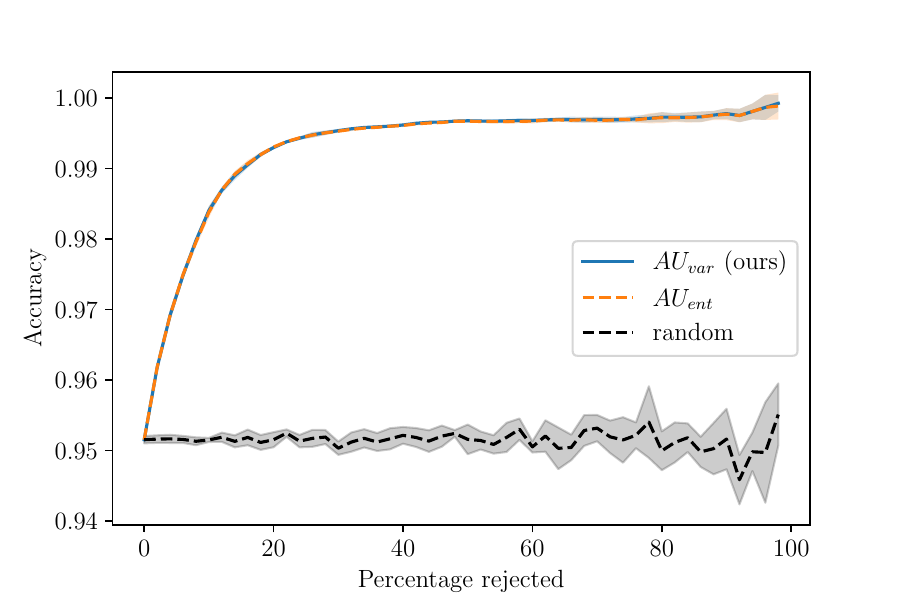}
\caption{SVHN (AU)}
\end{subfigure}

\begin{subfigure}{0.45\textwidth}
\includegraphics[width=\linewidth]{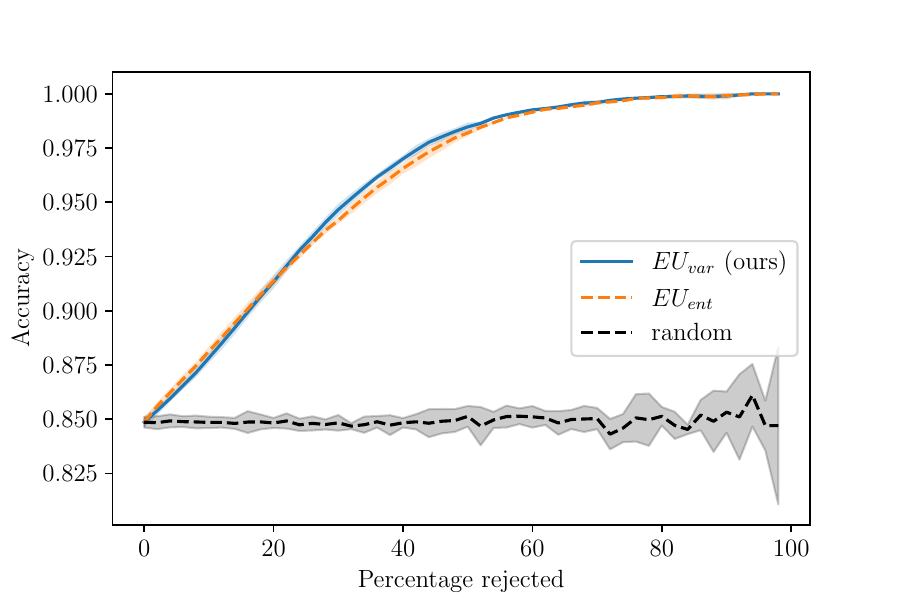}
\caption{CIFAR10 (EU)}
\end{subfigure}
\hfill
\begin{subfigure}{0.45\textwidth}
\includegraphics[width=\linewidth]{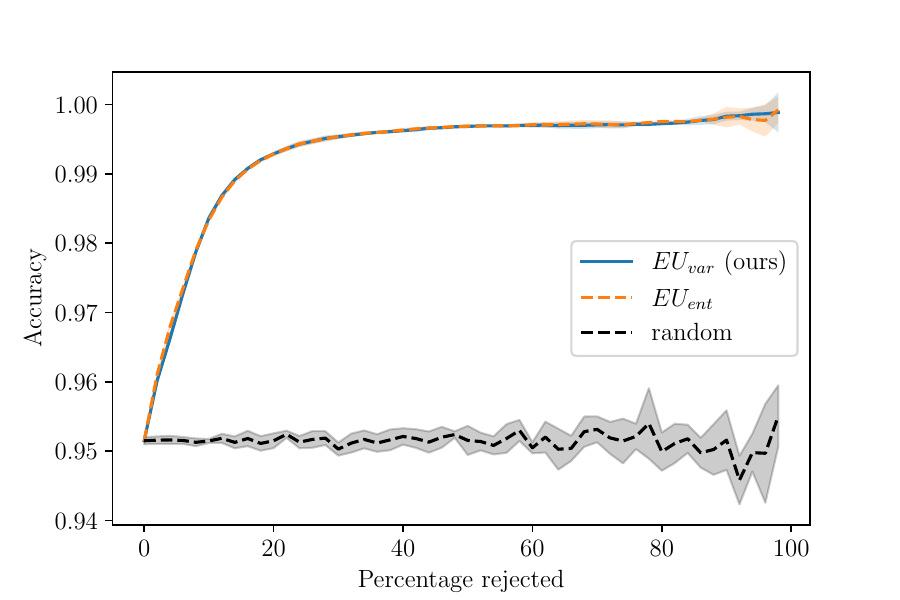}
\caption{SVHN (EU)}
\end{subfigure}

\caption{Accuracy-rejection curves on CIFAR10 (left), and SVHN (right).}
\label{fig:arcs}
\end{figure}

\subsection{Correct/Incorrect Predictions}
\label{sub:correct}
In addition to the accuracy-rejection curves we also present histograms of the values of the total uncertainty for all test instances. The instances are distinguished by whether the model has classified them correctly or incorrectly. Naturally, with a proper uncertainty measure, the higher the total uncertainty associated with an instance, the more likely this instance is to be misclassified. Conversely, an instance with low total uncertainty should be classified correctly. Figure~\ref{fig:hists} shows the histograms for the correct and incorrect predictions of the models described in Section \ref{subsec:arcs} trained on CIFAR10 (top) and SVHN (bottom). Appendix \ref{app:hist} contains the histograms of the aleatoric and epistemic uncertainties on the same data sets.
\begin{figure}[htbp]
\centering

\begin{subfigure}{0.45\textwidth}
\includegraphics[width=\linewidth]{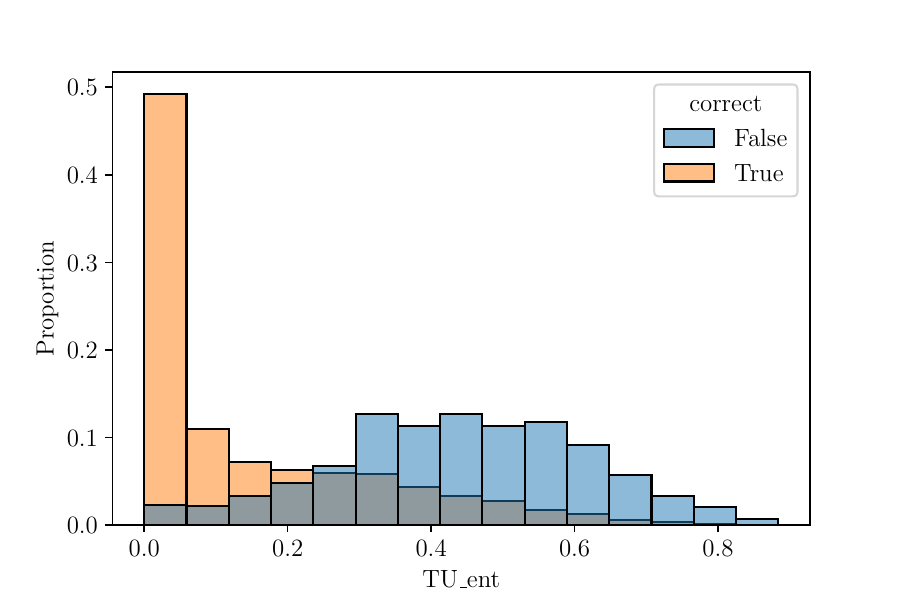}
\caption{CIFAR10 $(\TU_{\textnormal{ent}})$}
\end{subfigure}
\hfill
\begin{subfigure}{0.45\textwidth}
\includegraphics[width=\linewidth]{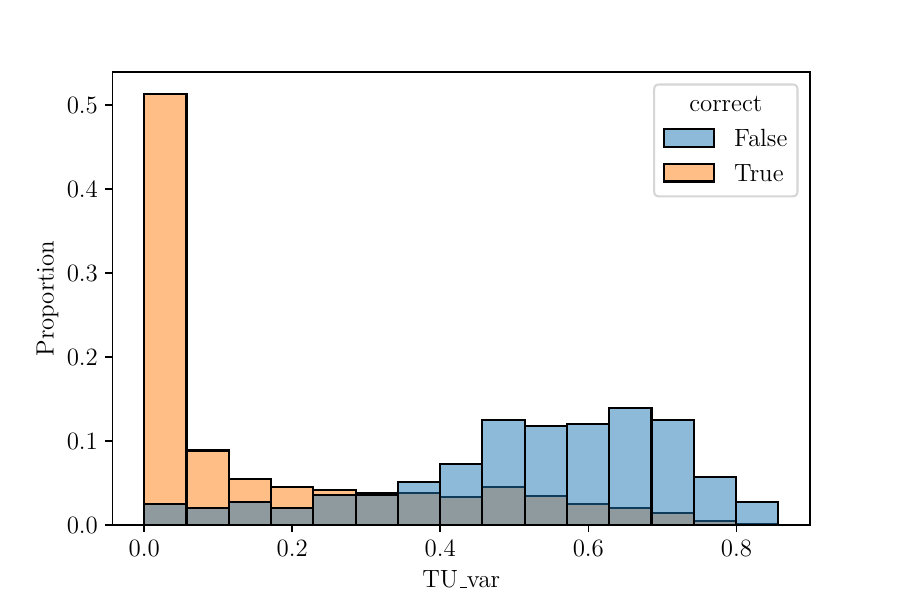}
\caption{CIFAR10 $(\TU_{\textnormal{var}})$}
\end{subfigure}
\begin{subfigure}{0.45\textwidth}
\includegraphics[width=\linewidth]{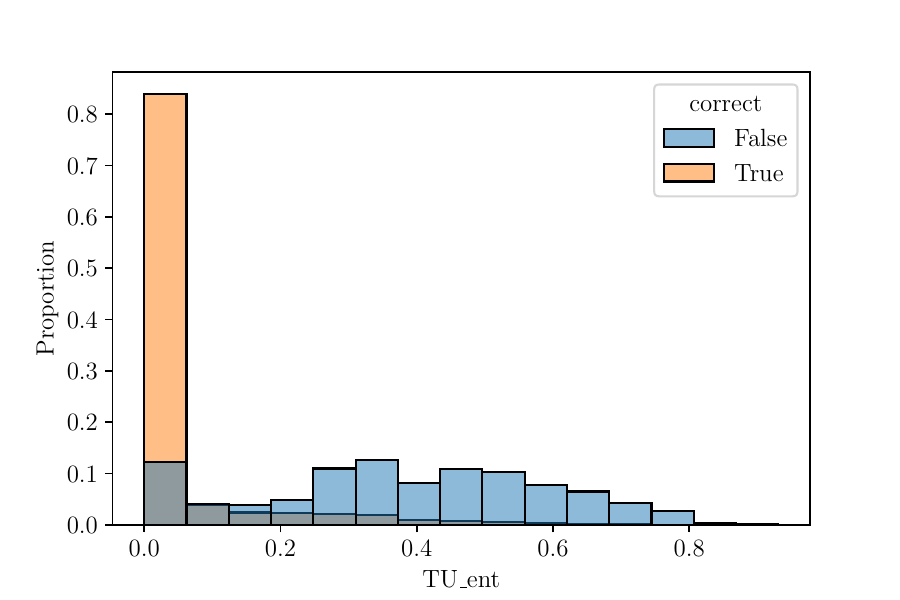}
\caption{SVHN $(\TU_{\textnormal{ent}})$}
\end{subfigure}
\hfill
\begin{subfigure}{0.45\textwidth}
\includegraphics[width=\linewidth]{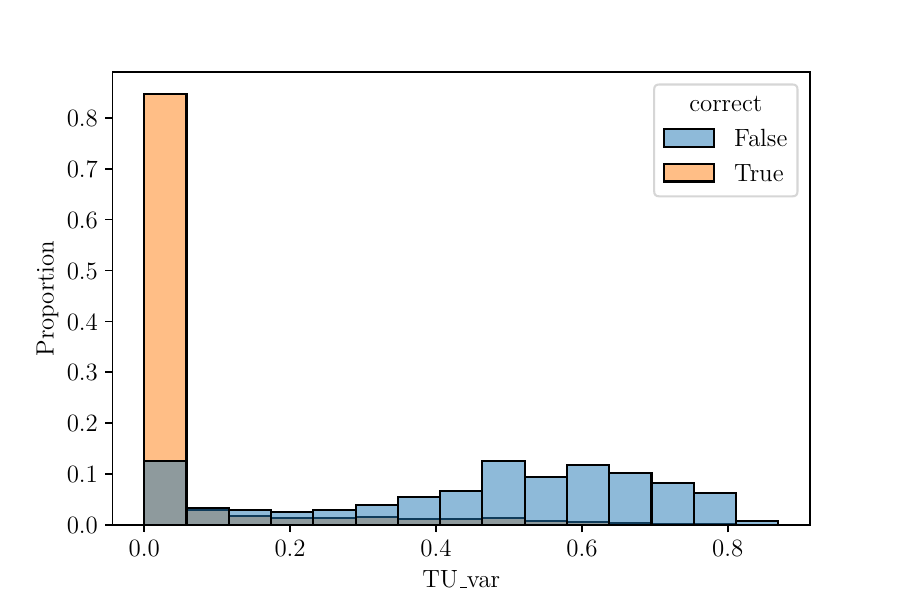}
\caption{SVHN $(\TU_{\textnormal{var}})$}
\end{subfigure}

\caption{Histograms of $\TU$ values on CIFAR10 (top), and SVHN (bottom).}
\label{fig:hists}
\end{figure}
Although the histograms for the entropy-based (left) and variance-based (right) measures look similar, there is less overlap between the total uncertainties of the correctly and incorrectly classified instances for the variance-based measures than for the entropy-based measures. However, we can not make any statements regarding the faithfulness of the uncertainty measures based on the total uncertainty due to the lack of ground truth uncertainties. 
\newpage

\subsection{Out-of-Distribution Detection}
To assess and compare measures of epistemic uncertainty, we conduct out-of-distribution (OoD) experiments.
We train a model on an in-distribution (ID) data set, and compute epistemic uncertainty values on instances of the ID test set. Subsequently, the model is exposed to data from an out-of-distribution (OoD) data set, and we similarly assess the epistemic uncertainty for these new instances.
The model, which has not previously encountered the OoD data, is expected to exhibit increased epistemic uncertainty for these instances. To determine the effectiveness in distinguishing epistemic uncertainties between ID and OoD instances, we calculate the AUROC.
Our approach involves training an ensemble of five neural networks on the FashionMNIST data set (ID), using MNIST \citep{lecun1998} and KMNIST \citep{clanuwat2018deep} as our chosen OoD data sets. We calculate both the average and the standard deviation of the AUROC across five iterations. Similarly, we conduct OoD experiments for  CIFAR10 (ID) with SVHN \citep{netzer2011reading} and CIFAR10.2 \citep{luHarderDifferentCloser} as OoD data sets. 
Table~\ref{table:ood} shows the results for the networks trained on FashionMNIST or CIFAR10. 

\begin{table}[h!]
\begin{center}
\begin{tabular}{@{}llllll@{}}\toprule
    &        & \multicolumn{2}{c}{FashionMNIST} & \multicolumn{2}{c}{CIFAR10} \\ \midrule
    \multicolumn{2}{c}{Measures} & MNIST          & KMNIST          & SVHN       & CIFAR10.2      \\ \midrule
    \multirow{1}{*}{Ours}  
    & $\EU_{\text{var}}$     &  .882 $\pm$ .018 & .959 $\pm$ .005 &      \textbf{.761 $\pm$ .022} & \textbf{.999 $\pm$ .001}            \\ \midrule
    \multirow{1}{*}{Baseline} & $\EU_{\text{ent}}$     & \textbf{.895 $\pm$ .017} & \textbf{.969 $\pm$ .004} &  .760 $\pm$ .026 & .998 $\pm$ .002             \\ \bottomrule
    \end{tabular}
    \vspace{0.2cm}
    \caption{OoD detection performance. AUROC and standard deviation over five runs are reported. $\EU_{\text{ent}}$ denotes mutual information, and $\EU_{\text{var}}$ the variance-based measure. Best performance is in \textbf{bold}.}
    \label{table:ood}
\end{center}
\end{table}
The results displayed in Table \ref{table:ood} demonstrate that our proposed variance-based measure shows competitive performance in OoD experiments across different data sets. 
It is particularly noteworthy that our measure stands on par with, and in certain cases surpasses, the entropy-based measures, which are hard to outperform in such applications \citep{smithGal2018}. 
Thus, our findings suggest that while entropy-based measures have been the \textit{de facto} standard for uncertainty quantification in classification tasks, variance-based measures perform equally well in downstream tasks, and are sometimes even more effective. Furthermore, the variance-based measures accommodate a more nuanced label-wise uncertainty perspective.

\subsection{Label-wise Uncertainty Quantification}
\label{sub:classwise}
To further illustrate the usefulness of a label-wise uncertainty measure, Figure \ref{fig:mnistmain} shows six representative high-uncertainty instances from the MNIST data set. In Appendix \ref{app:label} we provide further examples of images with the greatest total, aleatoric, and epistemic uncertainty.
\begin{figure}[htbp]
\centering
\includegraphics[width=\linewidth]{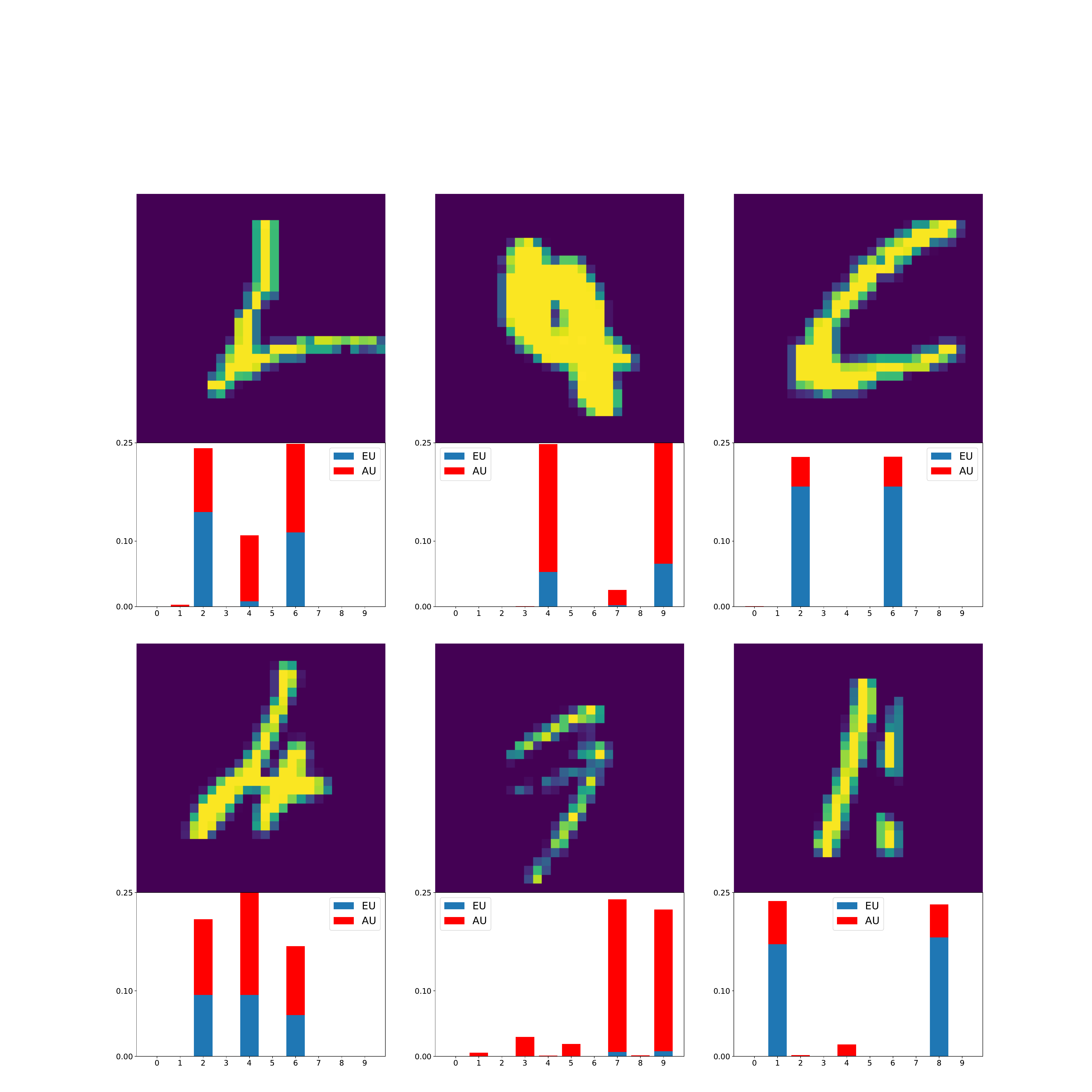}
\caption{MNIST instances with high uncertainty and their respective label-wise uncertainties.}
\label{fig:mnistmain}
\end{figure}
Based on these examples, we discuss a number of interesting qualitative behaviors to highlight how label-wise uncertainty quantification can aid in more nuanced decision-making. The top-left image represents a situation with high aleatoric uncertainty and high epistemic uncertainty. Using global uncertainty measures, one would need to conclude that the aleatoric uncertainty estimates cannot be trusted with so much epistemic uncertainty present. However, using label-wise uncertainty allows one to distinguish different levels of epistemic uncertainty for different classes. Thus, we can say that the estimate for the aleatoric uncertainty on class 4 is good, while for the classes 2 and 6 we need to reduce the epistemic uncertainty, for example by collecting more images of digits 2 and 6, to give a reasonable estimate for the aleatoric uncertainty. A similar situation can be found in the top-middle image, where there is hardly any (epistemic) uncertainty regarding the aleatoric uncertainty associated with class 7 but some epistemic uncertainty regarding the aleatoric uncertainty on classes 4 and 9.

\section{Concluding Remarks}
We proposed variance-based measures for uncertainty quantification in the context of second-order distributions. Addressing criticisms in the current literature and problems of the commonly used entropy measures, we showed that the proposed measures satisfy a set of desirable properties. Besides, our measures allow for reasoning about the uncertainty associated with individual classes, something that can aid in decision-making in situations where incorrect predictions have different consequences for different classes. We also presented empirical results highlighting the practical usefulness of these measures. All in all, we hope that this work presents a first step towards a more interpretable representation of uncertainty, which will be beneficial for safety-critical applications.

Let us conclude with a remark on future work related to the loss-based view of the uncertainty measures as briefly touched on in Section 3:
\begin{itemize}
\item Total uncertainty is the expected loss of the risk-minimizing prediction $\hat{\theta}$ given knowledge of the second-order distribution $Q$ (i.e., knowing that $\theta$ is drawn from $Q$ and then $Y$ according to $\theta$).
\item Aleatoric uncertainty is the expected loss of the risk-minimizing prediction $\hat{\theta}$ given knowledge about the true $\theta$ (sampled from $Q$). 
\item Epistemic uncertainty is the difference between these two, i.e., the extra loss that is caused by the lack of knowledge about the true $\theta$. 
\end{itemize}
The common entropy-based approach is an instantiation of this decomposition for the case of log-loss, the variance-based approach for the case of squared-error loss. Thus, an obvious idea is to elaborate on instantiations for other loss functions, too, maybe in combination with the label-wise decomposition of the uncertainty measures. Natural candidates for such losses are (strictly) proper scoring rules \citep{gnei_sp05}, which have the meaningful property that the risk-minimizer $\hat{\theta}$ given $\Theta = \theta$ coincides with $\theta$ itself\,---\,both log-loss and squared-error loss are examples of proper scoring rules.

\subsection*{Acknowledgments}
Yusuf Sale and Lisa Wimmer are supported by the DAAD program Konrad Zuse Schools of Excellence in Artificial Intelligence, sponsored by the Federal Ministry of Education and Research.

\bibliography{references}

\begin{thebibliography}{}

\bibitem[Augustin et~al., 2014]{augustin2014introduction}
Augustin, T., Coolen, F.~P., De~Cooman, G., and Troffaes, M.~C. (2014).
\newblock {\em Introduction to Imprecise Probabilities}.
\newblock John Wiley \& Sons.

\bibitem[Bronevich and Klir, 2008]{bronevich2008axioms}
Bronevich, A. and Klir, G.~J. (2008).
\newblock Axioms for uncertainty measures on belief functions and credal sets.
\newblock In {\em Annual Meeting of the North American Fuzzy Information
  Processing Society {(NAFIPS)}}, pages 1--6. IEEE.

\bibitem[Charpentier et~al., 2022]{Charpentier2022NaturalPN}
Charpentier, B., Borchert, O., Zugner, D., Geisler, S., and G\"unnemann, S.
  (2022).
\newblock Natural posterior network: Deep {B}ayesian predictive uncertainty for
  exponential family distributions.
\newblock In {\em Proc.\ {ICLR}, 10th International Conference on Learning
  Representations}.

\bibitem[Clanuwat et~al., 2018]{clanuwat2018deep}
Clanuwat, T., Bober-Irizar, M., Kitamoto, A., Lamb, A., Yamamoto, K., and Ha,
  D. (2018).
\newblock Deep learning for classical japanese literature.

\bibitem[Corani et~al., 2012]{corani2012bayesian}
Corani, G., Antonucci, A., and Zaffalon, M. (2012).
\newblock {Bayesian networks with imprecise probabilities: Theory and
  application to classification}.
\newblock {\em Data Mining: Foundations and Intelligent Paradigms: Volume 1:
  Clustering, Association and Classification}, pages 49--93.

\bibitem[Cover and Thomas, 1999]{cover1999elements}
Cover, T. and Thomas, J.~A. (1999).
\newblock {\em Elements of Information Theory}.
\newblock John Wiley \& Sons.

\bibitem[Depeweg et~al., 2018]{depeweg2018decomposition}
Depeweg, S., Hernandez-Lobato, J.-M., Doshi-Velez, F., and Udluft, S. (2018).
\newblock Decomposition of uncertainty in {B}ayesian deep learning for
  efficient and risk-sensitive learning.
\newblock In {\em Proc.\ ICML, 35th International Conference on Machine
  Learning}, pages 1184--1193. PMLR.

\bibitem[Duan et~al., 2023]{duan2023evidential}
Duan, R., Caffo, B., Bai, H.~X., Sair, H.~I., and Jones, C. (2023).
\newblock Evidential uncertainty quantification: A variance-based perspective.
\newblock {\em arXiv preprint arXiv:2311.11367}.

\bibitem[Gal, 2016]{gal_2016_UncertaintyDeepLearning}
Gal, Y. (2016).
\newblock {\em Uncertainty in {{Deep Learning}}}.
\newblock PhD thesis, {University of Cambridge}.

\bibitem[Gelman et~al., 2013]{gelman2013bayesian}
Gelman, A., Carlin, J.~B., Stern, H.~S., Dunson, D.~B., Vehtari, A., and Rubin,
  D.~B. (2013).
\newblock {\em Bayesian Data Analysis}.
\newblock CRC Press.

\bibitem[Gneiting and Raftery, 2005]{gnei_sp05}
Gneiting, T. and Raftery, A. (2005).
\newblock Strictly proper scoring rules, prediction, and estimation.
\newblock Technical Report 463R, Department of Statistics, University of
  Washington.

\bibitem[Gruber et~al., 2023]{gruber2023sources}
Gruber, C., Schenk, P.~O., Schierholz, M., Kreuter, F., and Kauermann, G.
  (2023).
\newblock Sources of uncertainty in machine learning -- a statisticians' view.
\newblock {\em arXiv preprint arXiv:2305.16703}.

\bibitem[He et~al., 2016]{heResnet2016}
He, K., Zhang, X., Ren, S., and Sun, J. (2016).
\newblock Deep residual learning for image recognition.
\newblock In {\em 2016 {IEEE} Conference on Computer Vision and Pattern
  Recognition, {CVPR} 2016, Las Vegas, NV, USA, June 27-30, 2016}, pages
  770--778. {IEEE} Computer Society.

\bibitem[Hoffer et~al., 2023]{HOFFER2023109279}
Hoffer, J., Ranftl, S., and Geiger, B. (2023).
\newblock Robust bayesian target value optimization.
\newblock {\em Computers \& Industrial Engineering}, 180:109279.

\bibitem[Houlsby et~al., 2011]{houlsby_2011_BayesianActiveLearning}
Houlsby, N., Huszár, F., Ghahramani, Z., and Lengyel, M. (2011).
\newblock {Bayesian active learning for classification and preference
  learning}.
\newblock {\em arXiv preprint arXiv:1112.5745}.

\bibitem[H\"uhn and H{\"{u}}llermeier, 2009]{huhn2009}
H\"uhn, J.~C. and H{\"{u}}llermeier, E. (2009).
\newblock {FR3:} {A} fuzzy rule learner for inducing reliable classifiers.
\newblock {\em {IEEE} Trans. Fuzzy Syst.}, 17(1):138--149.

\bibitem[H{\"u}llermeier et~al., 2022]{hullermeier2022quantification}
H{\"u}llermeier, E., Destercke, S., and Shaker, M.~H. (2022).
\newblock Quantification of credal uncertainty in machine learning: A critical
  analysis and empirical comparison.
\newblock In {\em Uncertainty in Artificial Intelligence}, pages 548--557.
  PMLR.

\bibitem[H\"ullermeier et~al., 2022]{eyke_new}
H\"ullermeier, E., Destercke, S., and Shaker, M.~H. (2022).
\newblock Quantification of credal uncertainty in machine learning: A critical
  analysis and empirical comparison.
\newblock In Cussens, J. and Zhang, K., editors, {\em Proceedings of the
  Thirty-Eighth Conference on Uncertainty in Artificial Intelligence}, volume
  180 of {\em Proceedings of Machine Learning Research}, pages 548--557. PMLR.

\bibitem[H{\"u}llermeier and Waegeman, 2021]{hullermeier2021aleatoric}
H{\"u}llermeier, E. and Waegeman, W. (2021).
\newblock Aleatoric and epistemic uncertainty in machine learning: An
  introduction to concepts and methods.
\newblock {\em Machine Learning}, 110(3):457--506.

\bibitem[Kendall and Gal, 2017]{kendall2017uncertainties}
Kendall, A. and Gal, Y. (2017).
\newblock What uncertainties do we need in {B}ayesian deep learning for
  computer vision?
\newblock In {\em Proc.\ {NeurIPS}, 30th Advances in Neural Information
  Processing Systems}, volume~30, pages 5574--5584.

\bibitem[Kingma and Ba, 2015]{kingmaAdam2015}
Kingma, D.~P. and Ba, J. (2015).
\newblock Adam: {A} method for stochastic optimization.
\newblock In Bengio, Y. and LeCun, Y., editors, {\em 3rd International
  Conference on Learning Representations, {ICLR} 2015, San Diego, CA, USA, May
  7-9, 2015, Conference Track Proceedings}.

\bibitem[Krizhevsky et~al., 2009]{krizhevsky2009learning}
Krizhevsky, A., Hinton, G., et~al. (2009).
\newblock Learning multiple layers of features from tiny images.

\bibitem[Kullback and Leibler, 1951]{kullback1951information}
Kullback, S. and Leibler, R.~A. (1951).
\newblock On information and sufficiency.
\newblock {\em The annals of mathematical statistics}, 22(1):79--86.

\bibitem[Lambrou et~al., 2010]{lambrou2010reliable}
Lambrou, A., Papadopoulos, H., and Gammerman, A. (2010).
\newblock Reliable confidence measures for medical diagnosis with evolutionary
  algorithms.
\newblock {\em IEEE Transactions on Information Technology in Biomedicine},
  15(1):93--99.

\bibitem[LeCun et~al., 1998]{lecun1998}
LeCun, Y., Bottou, L., Bengio, Y., and Haffner, P. (1998).
\newblock Gradient-based learning applied to document recognition.
\newblock {\em Proc. {IEEE}}, 86(11):2278--2324.

\bibitem[Lu et~al., 2020]{luHarderDifferentCloser}
Lu, S., Nott, B., Olson, A., Todeschini, A., Vahabi, H., Carmon, Y., and
  Schmidt, L. (2020).
\newblock Harder or {Different}?{A} {Closer} {Look} at {Distribution} {Shift}
  in {Dataset} {Reproduction}.

\bibitem[Michelmore et~al.,
  2018]{michelmore_2018_EvaluatingUncertaintyQuantification}
Michelmore, R., Kwiatkowska, M., and Gal, Y. (2018).
\newblock Evaluating {{Uncertainty Quantification}} in {{End-to-End Autonomous
  Driving Control}}.

\bibitem[Mobiny et~al., 2021]{mobiny_2021_DropConnectEffectiveModeling}
Mobiny, A., Yuan, P., Moulik, S.~K., Garg, N., Wu, C.~C., and Van~Nguyen, H.
  (2021).
\newblock {{DropConnect}} is effective in modeling uncertainty of {{Bayesian}}
  deep networks.
\newblock {\em Scientific Reports}, 11:5458.

\bibitem[Netzer et~al., 2011]{netzer2011reading}
Netzer, Y., Wang, T., Coates, A., Bissacco, A., Wu, B., and Ng, A.~Y. (2011).
\newblock Reading digits in natural images with unsupervised feature learning.

\bibitem[Nguyen et~al., 2022]{nguyen2022measure}
Nguyen, V.-L., Shaker, M.~H., and H{\"u}llermeier, E. (2022).
\newblock How to measure uncertainty in uncertainty sampling for active
  learning.
\newblock {\em Machine Learning}, 111(1):89--122.

\bibitem[Pal et~al., 1993]{pal1993uncertainty}
Pal, N.~R., Bezdek, J.~C., and Hemasinha, R. (1993).
\newblock Uncertainty measures for evidential reasoning {II}: {A} new measure
  of total uncertainty.
\newblock {\em International Journal of Approximate Reasoning}, 8(1):1--16.

\bibitem[Paszke et~al., 2019]{pytorch2019}
Paszke, A., Gross, S., Massa, F., Lerer, A., Bradbury, J., Chanan, G., Killeen,
  T., Lin, Z., Gimelshein, N., Antiga, L., Desmaison, A., K{\"{o}}pf, A., Yang,
  E.~Z., DeVito, Z., Raison, M., Tejani, A., Chilamkurthy, S., Steiner, B.,
  Fang, L., Bai, J., and Chintala, S. (2019).
\newblock Pytorch: An imperative style, high-performance deep learning library.
\newblock In Wallach, H.~M., Larochelle, H., Beygelzimer, A.,
  d'Alch{\'{e}}{-}Buc, F., Fox, E.~B., and Garnett, R., editors, {\em Advances
  in Neural Information Processing Systems 32: Annual Conference on Neural
  Information Processing Systems 2019, NeurIPS 2019, December 8-14, 2019,
  Vancouver, BC, Canada}, pages 8024--8035.

\bibitem[Sale et~al., 2023a]{sale2023secondorder}
Sale, Y., Bengs, V., Caprio, M., and Hüllermeier, E. (2023a).
\newblock Second-order uncertainty quantification: A distance-based approach.
\newblock {\em arXiv preprint arXiv:2312.00995}.

\bibitem[Sale et~al., 2023b]{sale2023volume}
Sale, Y., Caprio, M., and H{\"u}llermeier, E. (2023b).
\newblock Is the volume of a credal set a good measure for epistemic
  uncertainty?
\newblock In {\em Proc.\ UAI, 39th Conference on Uncertainty in Artificial
  Intelligence}, pages 1795--1804. PMLR.

\bibitem[Senge et~al., 2014]{senge_2014_ReliableClassificationLearning}
Senge, R., Bösner, S., Dembczyński, K., Haasenritter, J., Hirsch, O.,
  Donner-Banzhoff, N., and Hüllermeier, E. (2014).
\newblock Reliable classification: {{Learning}} classifiers that distinguish
  aleatoric and epistemic uncertainty.
\newblock {\em Information Sciences}, 255:16--29.

\bibitem[Shannon, 1948]{shannon1948mathematical}
Shannon, C.~E. (1948).
\newblock A mathematical theory of communication.
\newblock {\em The Bell System Technical Journal}, 27(3):379--423.

\bibitem[Shelmanov et~al., 2021]{shelmanov_2021_ActiveLearningSequence}
Shelmanov, A., Puzyrev, D., Kupriyanova, L., Belyakov, D., Larionov, D.,
  Khromov, N., Kozlova, O., Artemova, E., Dylov, D.~V., and Panchenko, A.
  (2021).
\newblock Active {{Learning}} for {{Sequence Tagging}} with {{Deep Pre-trained
  Models}} and {{Bayesian Uncertainty Estimates}}.
\newblock In {\em Proceedings of the 16th {{Conference}} of the {{EACL}}},
  pages 1698--1712. {Association for Computational Linguistics}.

\bibitem[Smith and Gal, 2018]{smithGal2018}
Smith, L. and Gal, Y. (2018).
\newblock Understanding measures of uncertainty for adversarial example
  detection.
\newblock In Globerson, A. and Silva, R., editors, {\em Proceedings of the
  Thirty-Fourth Conference on Uncertainty in Artificial Intelligence, {UAI}
  2018, Monterey, California, USA, August 6-10, 2018}, pages 560--569. {AUAI}
  Press.

\bibitem[Stanton et~al., 2023]{stanton2023bayesian}
Stanton, S., Maddox, W., and Wilson, A.~G. (2023).
\newblock Bayesian optimization with conformal prediction sets.
\newblock In {\em International Conference on Artificial Intelligence and
  Statistics}, pages 959--986. PMLR.

\bibitem[Ulmer et~al., 2023]{UlmerHF23}
Ulmer, D., Hardmeier, C., and Frellsen, J. (2023).
\newblock Prior and posterior networks: {A} survey on evidential deep learning
  methods for uncertainty estimation.
\newblock {\em Transaction of Machine Learning Research}.

\bibitem[Varshney, 2016]{varshney2016engineering}
Varshney, K.~R. (2016).
\newblock Engineering safety in machine learning.
\newblock In {\em 2016 Information Theory and Applications Workshop (ITA)},
  pages 1--5. IEEE.

\bibitem[Varshney and Alemzadeh, 2017]{varshney2017safety}
Varshney, K.~R. and Alemzadeh, H. (2017).
\newblock On the safety of machine learning: Cyber-physical systems, decision
  sciences, and data products.
\newblock {\em Big Data}, 5(3):246--255.

\bibitem[Walley, 1991]{walley1991}
Walley, P. (1991).
\newblock {\em Statistical Reasoning with Imprecise Probabilities}.
\newblock Chapman \& Hall.

\bibitem[Wimmer et~al., 2023]{wimmer2023quantifying}
Wimmer, L., Sale, Y., Hofman, P., Bischl, B., and H{\"u}llermeier, E. (2023).
\newblock Quantifying aleatoric and epistemic uncertainty in machine learning:
  Are conditional entropy and mutual information appropriate measures?
\newblock In {\em Proc.\ UAI, 39th Conference on Uncertainty in Artificial
  Intelligence}, pages 2282--2292. PMLR.

\bibitem[Yang et~al., 2009]{yang2009using}
Yang, F., Wang, H.-Z., Mi, H., Lin, C.-D., and Cai, W.-W. (2009).
\newblock Using random forest for reliable classification and cost-sensitive
  learning for medical diagnosis.
\newblock {\em BMC Bioinformatics}, 10(1):1--14.

\end{thebibliography}

\newpage
\appendix
\section{Proofs}
\label{appendix:proofs}

\begin{proof}[Proof of Theorem \ref{thm:axioms}]
We prove that the variance-based uncertainty measures \eqref{tu:variance}, \eqref{au:variance}, and \eqref{eu:variance} satisfy Axioms A0, A1, and Axioms A3--A7 of Section \ref{subsection:axioms}.
\begin{itemize}
    \item[A0:] This property holds trivially true. 
    \item[A1:] Let $Q = \delta_{\vtheta} \in \ksimplextwo$ be a Dirac-measure on $\vtheta \in \ksimplex$ and $w_1,\dots,w_K \geq 0$. Then $\EU[\delta_{\vtheta}] = 0$ holds trivially true, since $\Var_{Q_k}[\Theta_k] = 0$ for all $k \in \{1, \dots, K\}$. The other direction follows similarly.   
    \item[A3:] Let $Q^{\prime} \in \ksimplextwo$ be a mean-preserving spread of $Q \in \ksimplextwo$, i.e., let $\vec{X} \sim Q, \vec{X}^\prime \sim Q^\prime$ be two random variables such that $\vec{X}^\prime \overset{d}{=} \vec{X} + \vec{Z}$, for some random variable $\vec{Z}$ with $\mathbb{E}[\vec{Z} \given \sigma(\vec{X})] = 0$ almost surely. Then, we have the following:
    \begin{align}
    \EU(Q^{\prime}) &=  \sum_{k=1}^{K} w_k \Var(X_k + Z_k) \\
           &= \sum_{k=1}^{K} w_k \left(\Var(X_k) + \Var(Z_k)+ 2 \Cov(X_k, Z_k) \right)         \\
           &= \EU(Q) + \sum_{k=1}^{K} w_k \Var(Z_k) 
           \label{mps:ineq} \\
           &> \EU(Q) \label{mps:ineq2}
    \end{align}
    Note that the equality \eqref{mps:ineq} holds true, since we know that $\Cov(X_k, Z_k) = 0$. To see this, observe that we have $\mathbb{E}[X_k \cdot Z_k] = \mathbb{E}[\mathbb{E}[X_k \cdot Z_k \given \sigma(X_k)]] = 0$ due to the mean-preserving spread assumption. Similarly, we know that $\mathbb{E}[Z_k] = 0$, such that we have $\Cov(X_k, Z_k) = \mathbb{E}[X_k \cdot Z_k] - \mathbb{E}[X_k] \cdot \mathbb{E}[Z_k] = 0.$  The inequality \eqref{mps:ineq2} is strict since by assumption $\max_k \Var(Z_k) > 0$ and $\min w_k > 0$.
    \item[A4:] 
    This is a special case of Corollary \ref{cor:A4} with equal weights.
    \item[A5:] Let $\vec{\Theta} \sim Q$, and $(\vec{\Theta} + \vec{z}) \sim Q^{\prime}$, where $\vec{z} \neq \vec{0}$ is a constant. Then, we observe
\begin{align*}
    \EU(Q^{\prime}) &= \sum_{k = 1}^{K} w_k \Var(\Theta_k + z_k) \\[0.1cm]
    &= \sum_{k = 1}^{K}  w_k \, \mathbb{E}[((\Theta_k + z_k) - \mathbb{E}[\Theta_k + z_k] )^2 ] \\[0.1cm]
    &= \sum_{k = 1}^{K} w_k \mathbb{E}[(\Theta_k - \mathbb{E}[\Theta_k])^2] \\[0.1cm]
    &= \EU(Q). 
\end{align*}
    \item[A6:] With $\delta_m \in \Delta_{\delta_m}$ we have 
    \begin{align*}
        \AU(\delta_m) &= \sum_{k = 1}^{K} w_k \mathbb{E}[\Theta_k \cdot (1 - \Theta_k)] \\[0.2cm]
        &= \sum_{k = 1}^{K} w_k[(\lambda_k \cdot (1 - 1) + (1 - \lambda_k) \cdot (0 - 0)] \\[0.2cm]
        &= 0.
    \end{align*}
    \item[A7:] Let $Q \in \ksimplextwo$ and $w_1,\dots,w_K \geq 0$. Further denote by $Q_{|\lab_1}$ and $Q_{|\lab_2}$ the corresponding marginalized distribution, where $\lab_1$ and $\lab_2$ are partitions of $\lab$. It holds 
    \begin{align*}
        \TU_{\lab}(Q) = \sum_{k \in \lab} w_k \Var(Y_k) &= \sum_{k \in \lab_1} w_k \Var(Y_k) + \sum_{k \in \lab_2} w_k \Var(Y_k) \\[0.2cm]
        &= \TU_{\lab_1}(Q_{|\lab_1}) + \TU_{\lab_2}(Q_{|\lab_2}),
    \end{align*}
    similarly the same holds for $\AU$. Due to the additive decomposition the claim is also true for $\EU$.
\end{itemize}
This concludes the proof.
\end{proof}

\begin{proof}[Proof of Lemma \ref{lem:concavity}]    It holds 
    \begin{align*}
    \nabla^2 V(\vec{\theta}) = -2 \mathrm{diag}(w_1, \dots, w_K ),
    \end{align*}
    which is negative definite. To find the minimizer, consider the Lagrangian dual:
    \begin{align*}
        \max_{\vec{\theta} \in \Delta_K, \lambda} V^*(\vec{\theta}, \lambda) =  \max_{\vec{\theta} \in \Delta_K, \lambda} V(\vec{\theta}) + \lambda \biggl(\sum_{k = 1}^K \theta_k - 1\biggr).
    \end{align*}
    The first-order conditions are
    \begin{align*}
        \frac{\partial  V^*(\vec{\theta}, \lambda)}{\partial \theta_k} =  w_k(1 - 2 \theta_k) + \lambda = 0, k = 1, \dots, K \quad \text{and} \quad 
        \sum_{k = 1}^K \theta_k = 1,
    \end{align*}
    which are solved by $\vec{\theta} = \vec{\beta}$ defined in the statement of the lemma and $\lambda = -(K - 2) / (\sum_{k = 1}^K 1/w_k)$.
\end{proof}

\begin{prop}\label{corrigendum}
Let $Q^{\prime} \in \ksimplextwo$ be a mean-preserving spread of $Q \in \ksimplextwo$, i.e., let $X \sim Q, X^\prime \sim Q^\prime$ be two random variables such that $X^\prime \overset{d}{=} X + Z$, for some random variable $Z$ with $\mathbb{E}[Z \given \sigma(X)] = 0$ almost surely. \\[0.2cm]
Now, define 
\begin{align}
\EU(Q) = \mathbb{E}_Q[D_{KL}(\vec{X} \| \bar{\vec{x}})], 
\end{align}
where \(D_{KL}(\vec{X} \| \bar{\vec{x}})\) denotes the Kullback-Leibler (KL) divergence of \(\vec{X}\) from its mean \(\bar{\vec{x}}\). Then the claim is that $\EU(Q') \geq \EU(Q).$    
\end{prop}

\begin{proof}
First, note that $D_{KL}(\vec{X} \| \bar{\vec{x}})$ is a convex function of $\vec{X}$ since $\bar{\vec{x}} \in \ksimplex$ is a constant. Given that $\vec{X}^{\prime} \overset{d}{=} \vec{X} + \vec{Z}$ and $\mathbb{E}[\vec{Z} \given \sigma(\vec{X})] = 0$ almost surely, it follows that $\mathbb{E}_{Q^{\prime}}[\vec{X}^{\prime}] = \mathbb{E}_{Q}[\vec{X}]$, preserving the mean. \\[0.1cm]
Jensen's inequality states that for a convex function \(f\) and a random variable $\vec{Y}$, we have $\mathbb{E}[f(\vec{Y})] \geq f(\mathbb{E}[\vec{Y}])$. Then, we know that Jensen's inequality implies:
\begin{align*}
\mathbb{E}_{Q'}[D_{KL}(\vec{X}' \| \bar{\vec{x}}) \given \sigma(\vec{X})] \geq \underbrace{D_{KL}(\mathbb{E}_{Q'}[\vec{X}'\given \sigma(\vec{X})] \| \bar{\vec{x}})}_{= D_{KL}(\vec{X} \| \bar{\vec{x}})} \quad \textnormal{a.s}.
\end{align*}
By this we get $\mathbb{E}_Q[\mathbb{E}_{Q'}[D_{KL}(\vec{X}' \| \bar{\vec{x}}) \given \sigma(\vec{X})]] \geq \mathbb{E}_{Q}[D_{KL}(\vec{X} \| \bar{\vec{x}})]$. Law of total expectation yields $\mathbb{E}_{Q'}[D_{KL}(\vec{X}' \| \bar{\vec{x}})]$ for the left hand-side of the inequality. This concludes the proof. 
\end{proof}

\newpage
\section{Experimental Details}
\label{appendix:exp_details}
In this section, we provide a detailed overview of the experimental setup to allow reproduction of the results. In addition, the code is available in a public GitHub repository\footnote{\href{https://github.com/YSale/uq-variance.git}{Here} is the Repository.}. 
\paragraph{Setup} The code is written in \texttt{Python 3.9} using the \texttt{PyTorch} \citep{pytorch2019} library. 
\paragraph{Data sets} For all data sets, we use the respective dedicated train-test splits. We only use pre-processing for the CIFAR10 data set. Each image is normalized using the mean and standard deviation per channel of the training set. Additionally, the training images are cropped randomly (while adding 4 pixels of padding on every border, and randomly flipped horizontally). 
\paragraph{Ensembles} The ensembles are built using the two base models: a Convolutional Neural Network (CNN) and a \texttt{ResNet18} \citep{heResnet2016}. The \texttt{CNN} has two convolutional layers followed by two fully connect layers. The convolutional layers have 32 and 64 filters of 5 by 5 and the fully connected layers have 512 and 10 neurons, respectively. The layers have \texttt{ReLU} activations and the last layer uses a softmax function to output probabilities. The \texttt{ResNet18} model has a fully connected last layer of 10 units and a softmax function to generate probabilities for 10 classes. The output of the ensemble is generated by averaging over the outputs of the individual ensemble members.

\subsection*{Accuracy-Rejection Curves}
We train 5 \texttt{CNNs} on FMNIST, MNIST and KMNIST and 5 \texttt{ResNets} on CIFAR10 and SVHN. We use \texttt{Adam} \citep{kingmaAdam2015} with the default hyper-parameters to train the \texttt{CNNs} in 10 epochs for MNIST and 20 epochs for FMNIST and KMNIST using a batch size of 256. We train the \texttt{ResNets} using stochastic gradient descent with weight decay set to $10^{-4}$, momentum at 0.9 and a learning rate of 0.1, setting the learning rate to 0.001 at epoch 20 and to 0.0001 at epoch 25. The models are trained for 30 epochs in total. 
The ARCs are then generated using the test set.

\subsection*{Out-of-Distribution Detection}
We, again, train 5 \texttt{CNNs} on FashionMNIST and 5 \texttt{ResNets} on CIFAR10 using the same setup as for the ARCs. The epistemic uncertainty is computed on the test sets of the respective data sets without applying any data augmentation. 
\newpage
\section{Additional Results}
\label{appendix:exp}
In this section, we report on experiments that we perform in addition to the ones in the main paper.
\subsection*{Accuracy-Rejection Curves}
\label{app:arc}
We train an ensemble of 5 neural networks on the data sets using the setup outlined in Section \ref{appendix:exp_details}. Figure \ref{fig:arcs_supp} shows the accuracy-rejection curves for the MNIST data set (left) and the FMNIST data set (right). The accuracies are reported as the mean over five runs, and the standard deviation is depicted by the shaded area.
\begin{figure}[htbp]
\centering
\begin{subfigure}{0.45\textwidth}
\includegraphics[width=\linewidth]{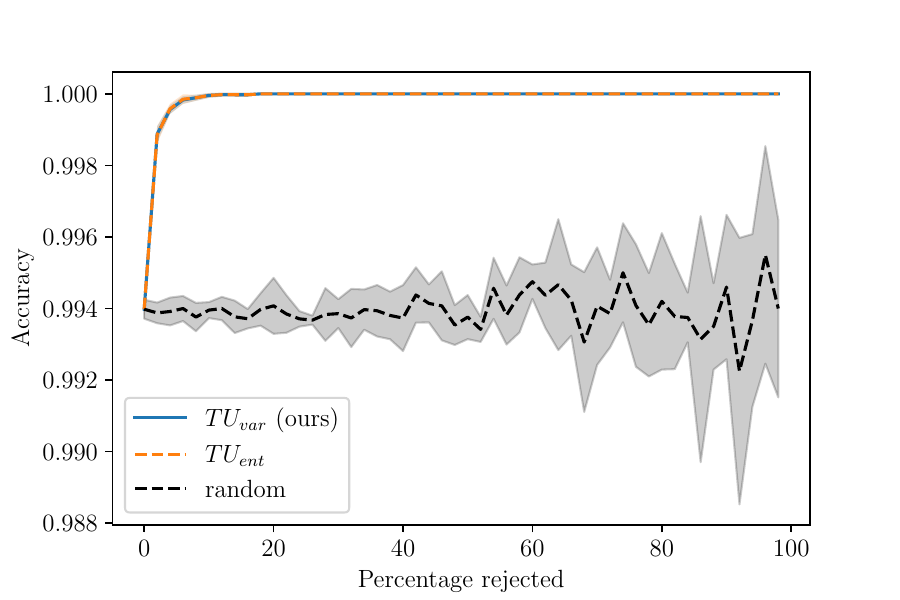}
\caption{MNIST (TU)}
\end{subfigure}
\hfill
\begin{subfigure}{0.45\textwidth}
\includegraphics[width=\linewidth]{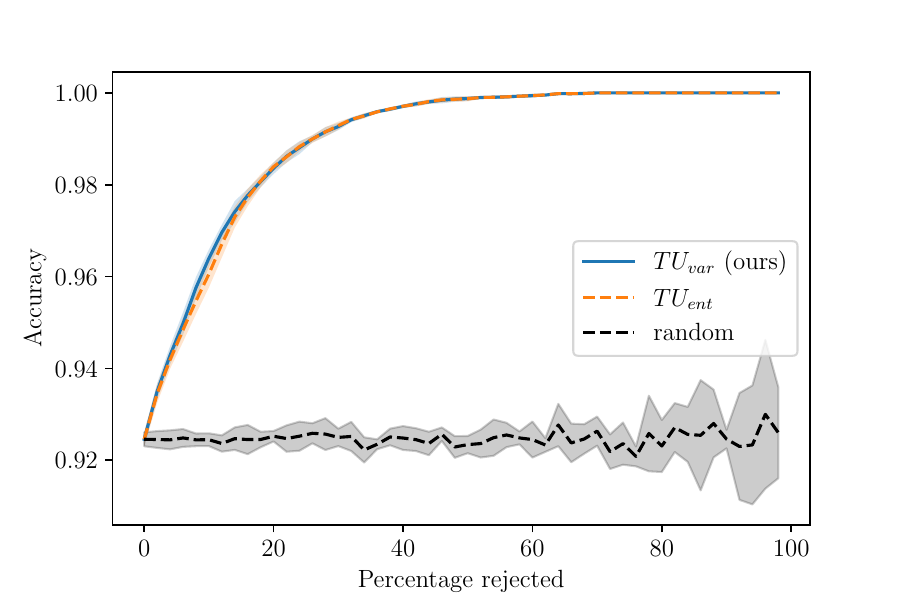}
\caption{FMNIST (TU)}
\end{subfigure}

\begin{subfigure}{0.45\textwidth}
\includegraphics[width=\linewidth]{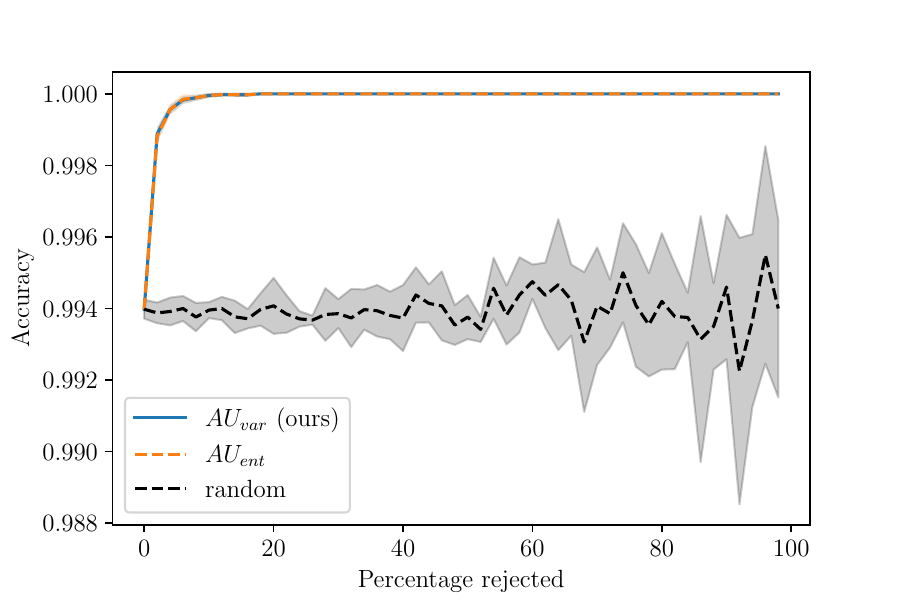}
\caption{MNIST (AU)}
\end{subfigure}
\hfill
\begin{subfigure}{0.45\textwidth}
\includegraphics[width=\linewidth]{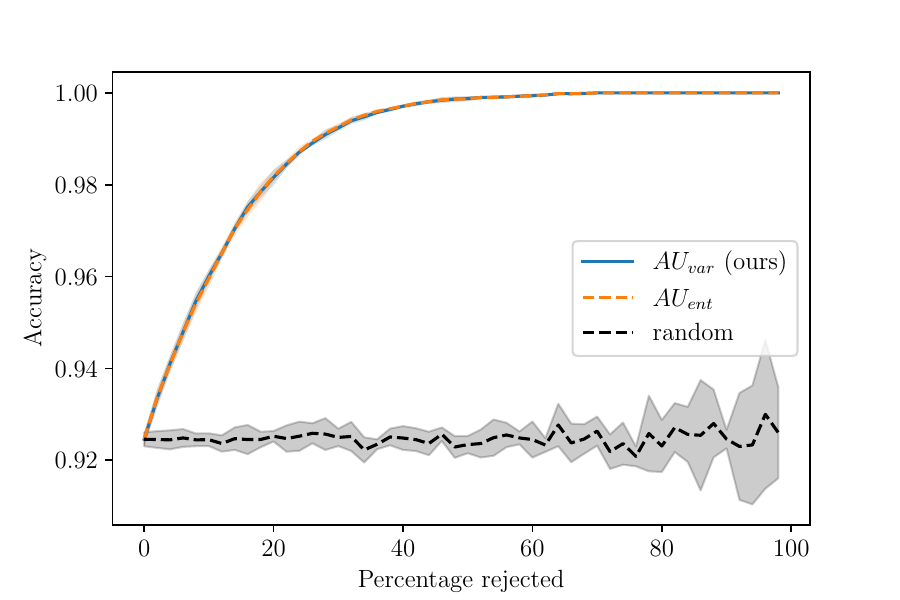}
\caption{FMNIST (AU)}
\end{subfigure}

\begin{subfigure}{0.45\textwidth}
\includegraphics[width=\linewidth]{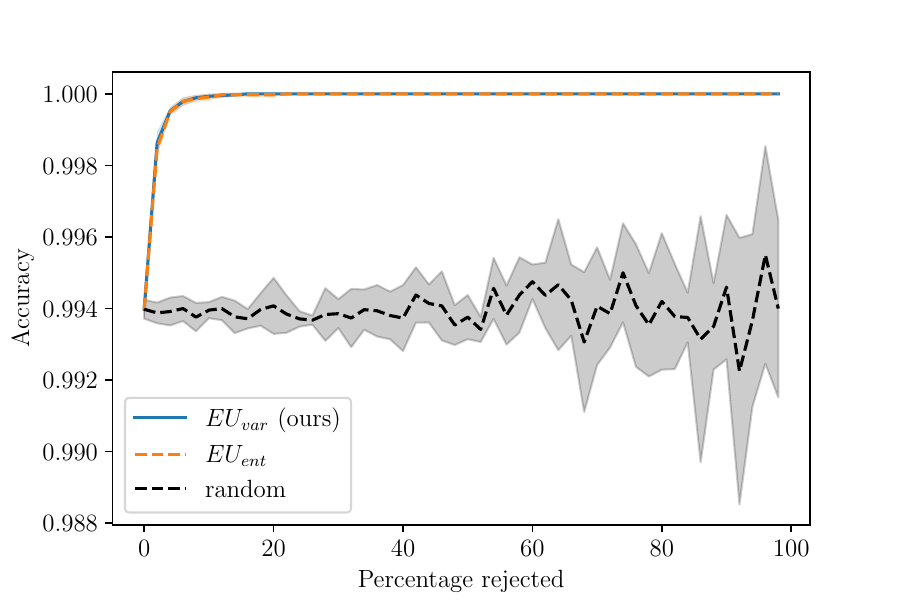}
\caption{MNIST (EU)}
\end{subfigure}
\hfill
\begin{subfigure}{0.45\textwidth}
\includegraphics[width=\linewidth]{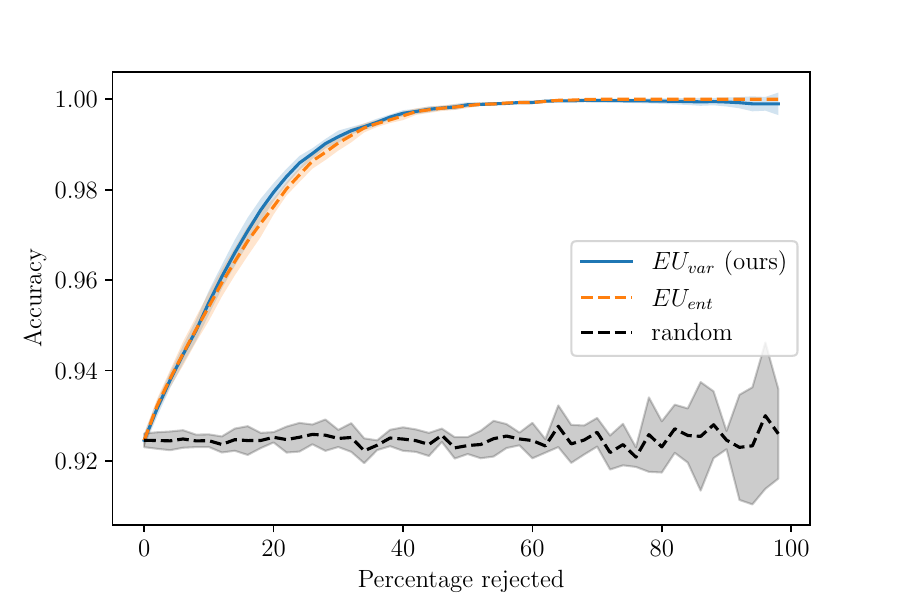}
\caption{FMNIST (EU)}
\end{subfigure}

\caption{Accuracy-rejection curves on MNIST (left), and FMNIST (right).}
\label{fig:arcs_supp}
\end{figure}
We observe similar results as the ones presented in the main paper. The entropy-based and variance-based measures show very similar behavior, increasing monotonically for almost every measure. 

\newpage
\subsection*{Correct/Incorrect Predictions}
\label{app:hist}
We plot the histograms of the aleatoric (Figure \ref{fig:hist_supp_au}) and epistemic (Figure \ref{fig:hist_supp_eu}) uncertainties for the models trained on CIFAR10 and SVHN as described in Section \ref{sub:correct}.
\begin{figure}[htbp]
\centering

\begin{subfigure}{0.35\textwidth}
\includegraphics[width=\linewidth]{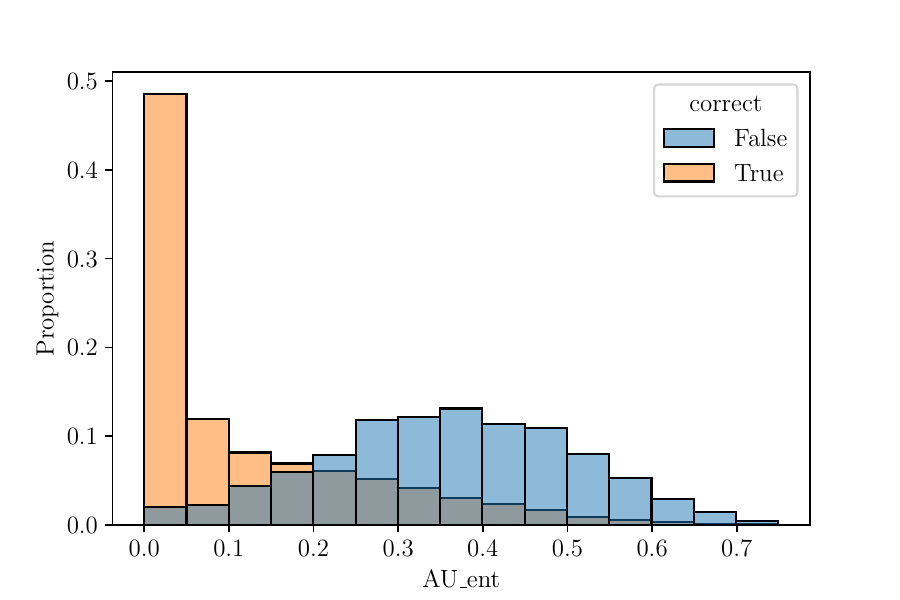}
\caption{CIFAR10 $(\AU_{\textnormal{ent}})$}
\end{subfigure}
\hspace{2.5cm}
\begin{subfigure}{0.35\textwidth}
\includegraphics[width=\linewidth]{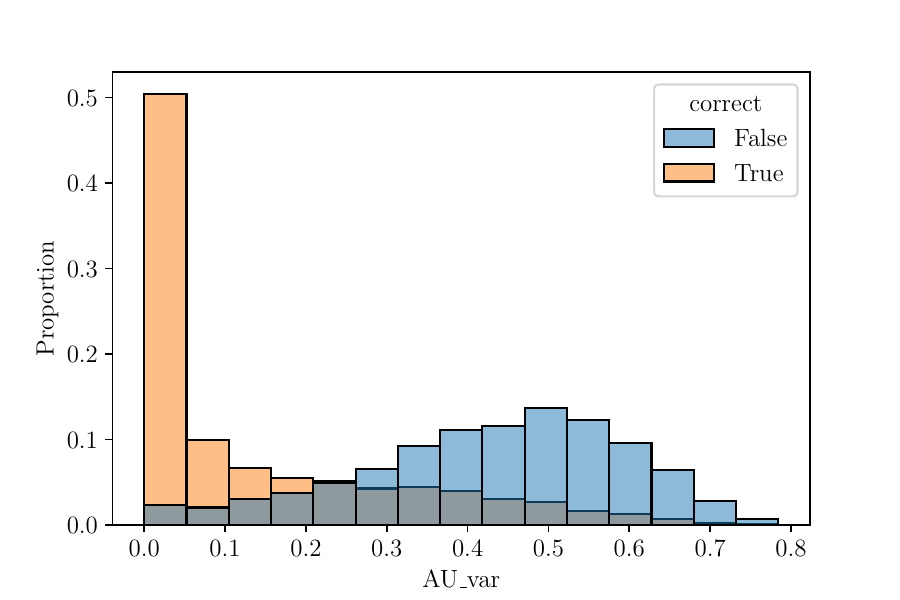}
\caption{CIFAR10 $(\AU_{\textnormal{var}})$}
\end{subfigure}

\begin{subfigure}{0.35\textwidth}
\includegraphics[width=\linewidth]{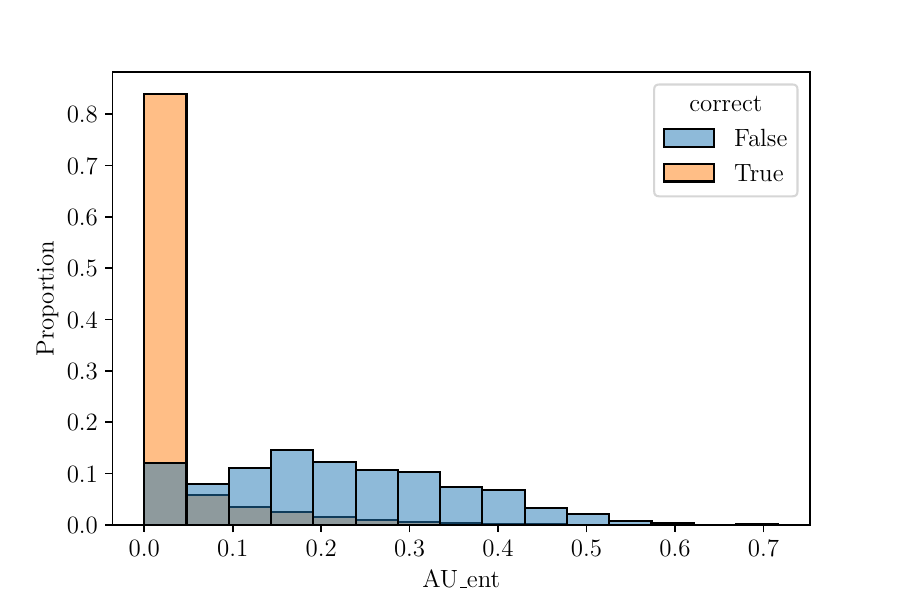}
\caption{SVHN $(\AU_{\textnormal{ent}})$}
\end{subfigure}
\hspace{2.5cm}
\begin{subfigure}{0.35\textwidth}
\includegraphics[width=\linewidth]{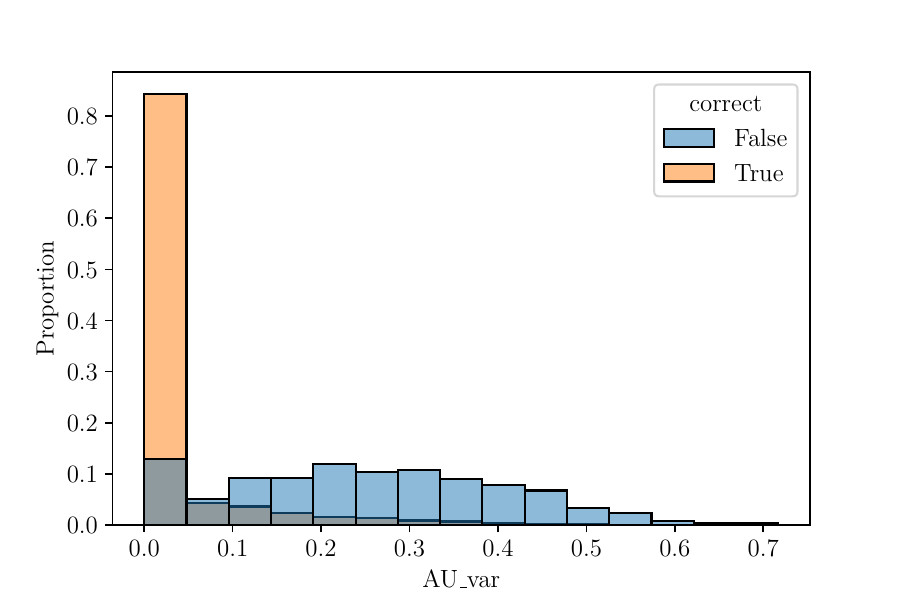}
\caption{SVHN $(\AU_{\textnormal{var}})$}
\end{subfigure}
\caption{Histograms of $\AU$ values on CIFAR10 (top) and SVHN (bottom).}
\label{fig:hist_supp_au}
\end{figure}

\begin{figure}[htbp]
\centering

\begin{subfigure}{0.35\textwidth}
\includegraphics[width=\linewidth]{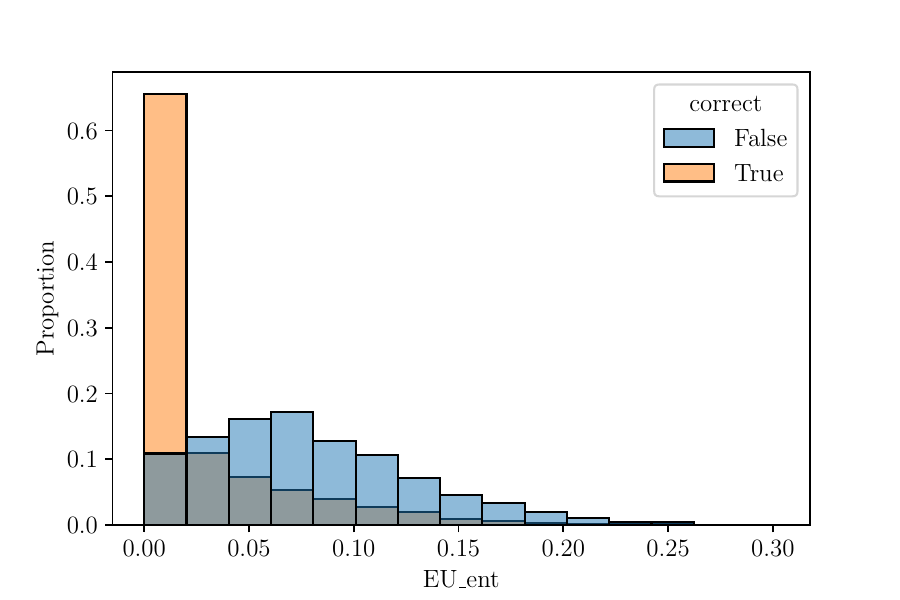}
\caption{CIFAR10 $(\EU_{\textnormal{ent}})$}
\end{subfigure}
\hspace{2.5cm}
\begin{subfigure}{0.35\textwidth}
\includegraphics[width=\linewidth]{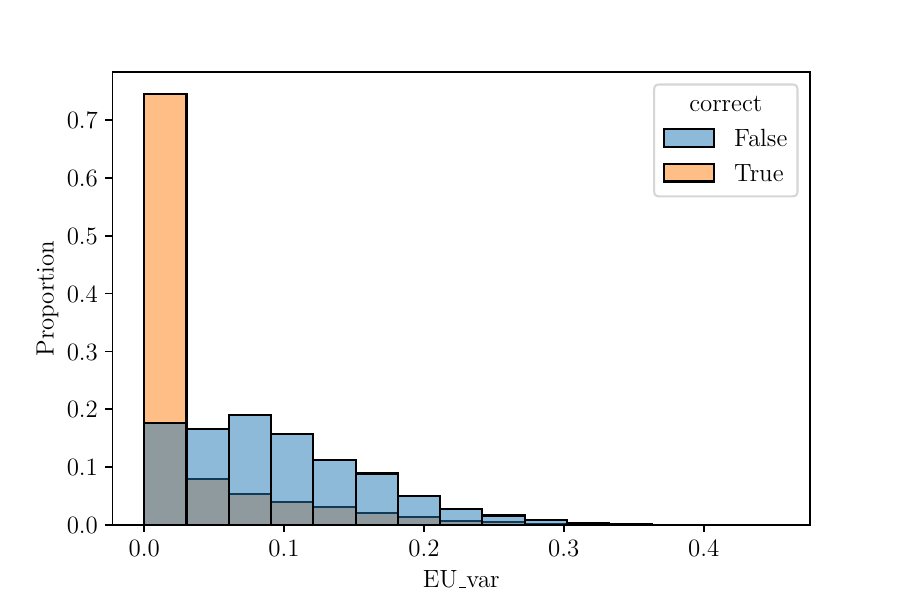}
\caption{CIFAR10 $(\EU_{\textnormal{var}})$}
\end{subfigure}

\begin{subfigure}{0.35\textwidth}
\includegraphics[width=\linewidth]{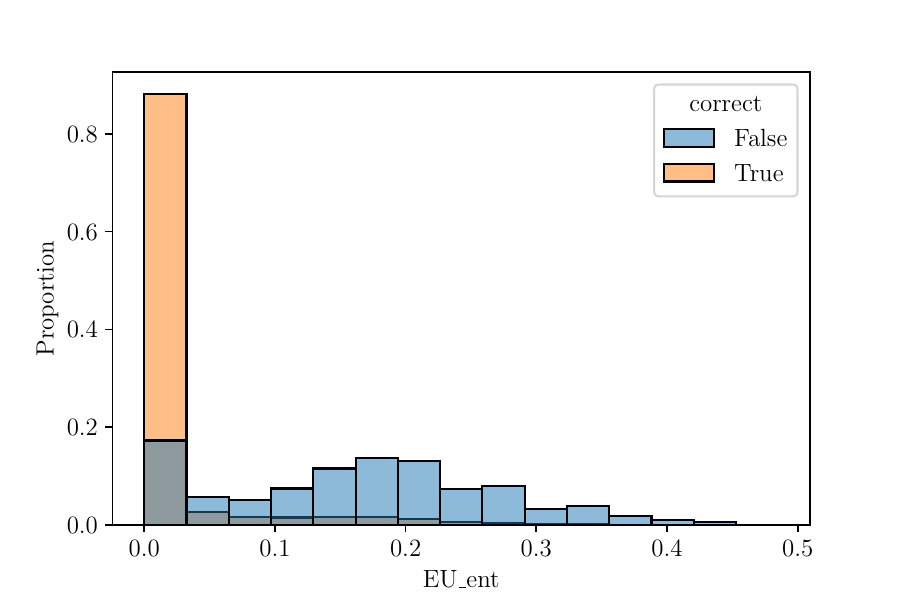}
\caption{SVHN $(\EU_{\textnormal{ent}})$}
\end{subfigure}
\hspace{2.5cm}
\begin{subfigure}{0.35\textwidth}
\includegraphics[width=\linewidth]{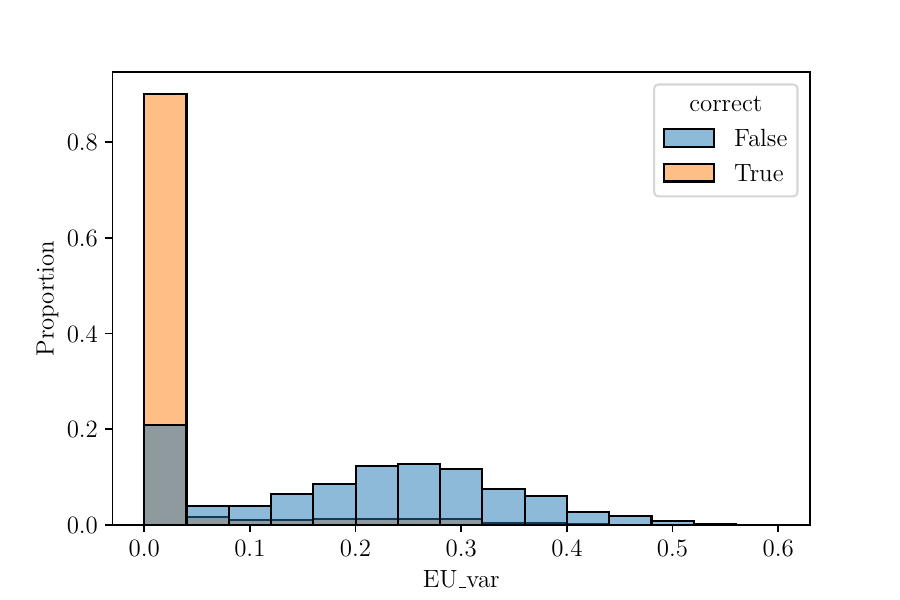}
\caption{SVHN $(\EU_{\textnormal{var}})$}
\end{subfigure}

\caption{Histograms of $\EU$ values on CIFAR10 (top) and SVHN (bottom).}
\label{fig:hist_supp_eu}
\end{figure}
Both aleatoric and epistemic uncertainty show very similar results to the ones in the main paper. The variance-based measure incorrect distributions seem to be shifted to the right a bit more than the entropy-based distributions.

\newpage
\subsection*{Label-wise Uncertainty Quantification}
\label{app:label}
We give additional examples of the images with the highest total (Figure \ref{fig:tu_add}), aleatoric (Figure \ref{fig:au_add}), and epistemic (Figure \ref{fig:eu_add}) uncertainties for the MNIST data set.  
\begin{figure}[htbp]
    \centering

    \begin{subfigure}[b]{0.24\textwidth}
        \includegraphics[width=\textwidth]{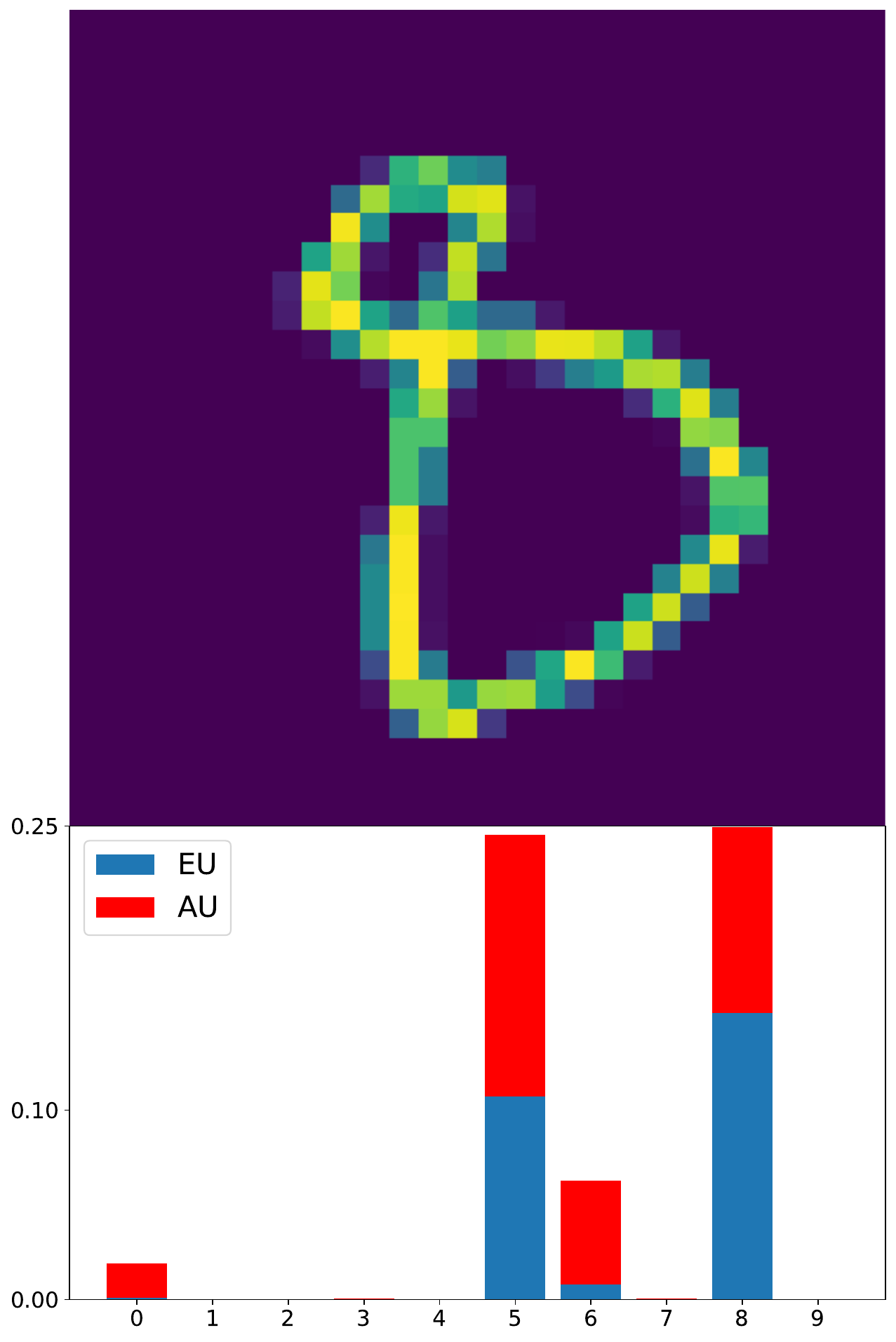}
    \end{subfigure}
    \hspace{1.5cm}
    \begin{subfigure}[b]{0.24\textwidth}
        \includegraphics[width=\textwidth]{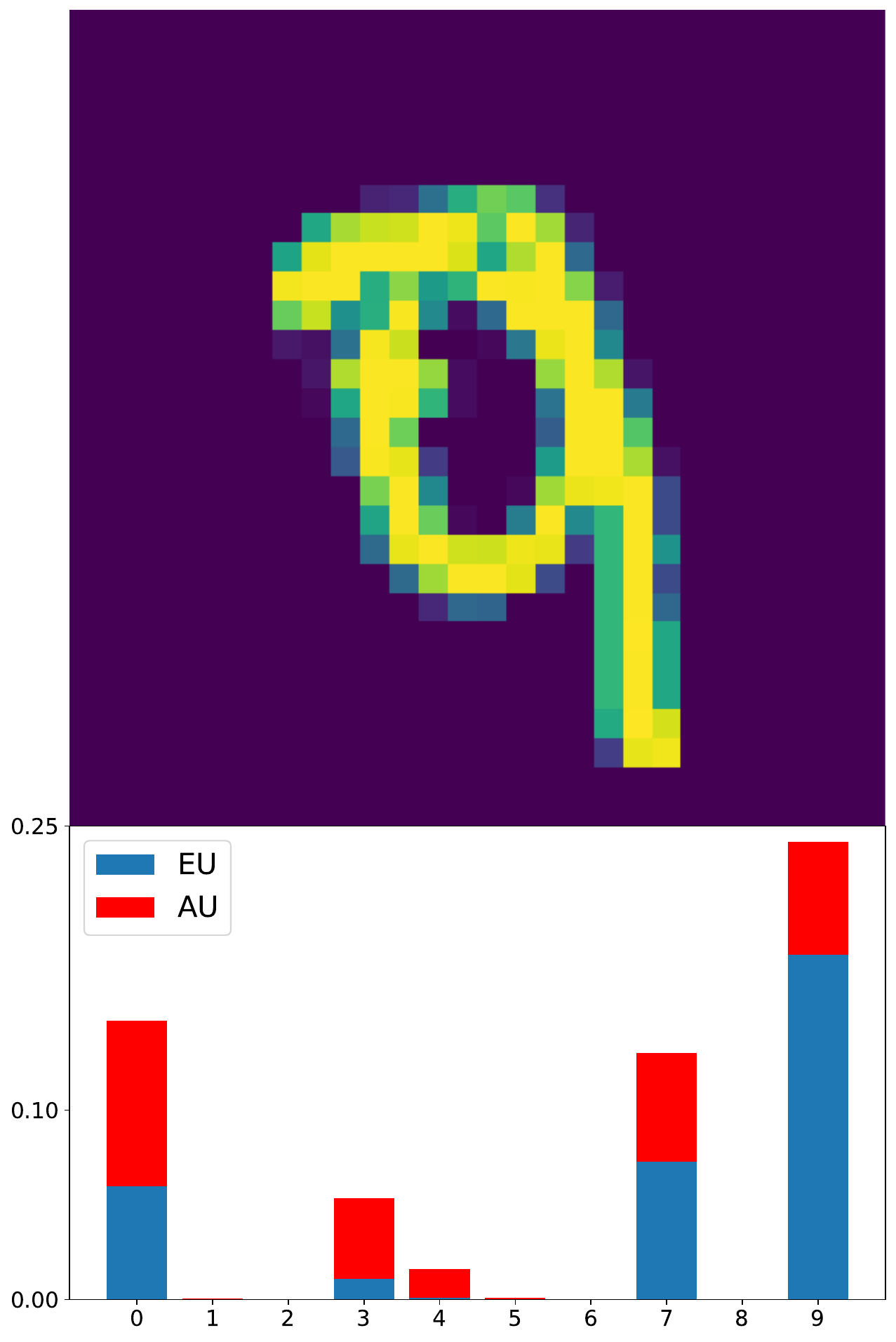}
    \end{subfigure}
    \hspace{1.5cm}
    \begin{subfigure}[b]{0.24\textwidth}
        \includegraphics[width=\textwidth]{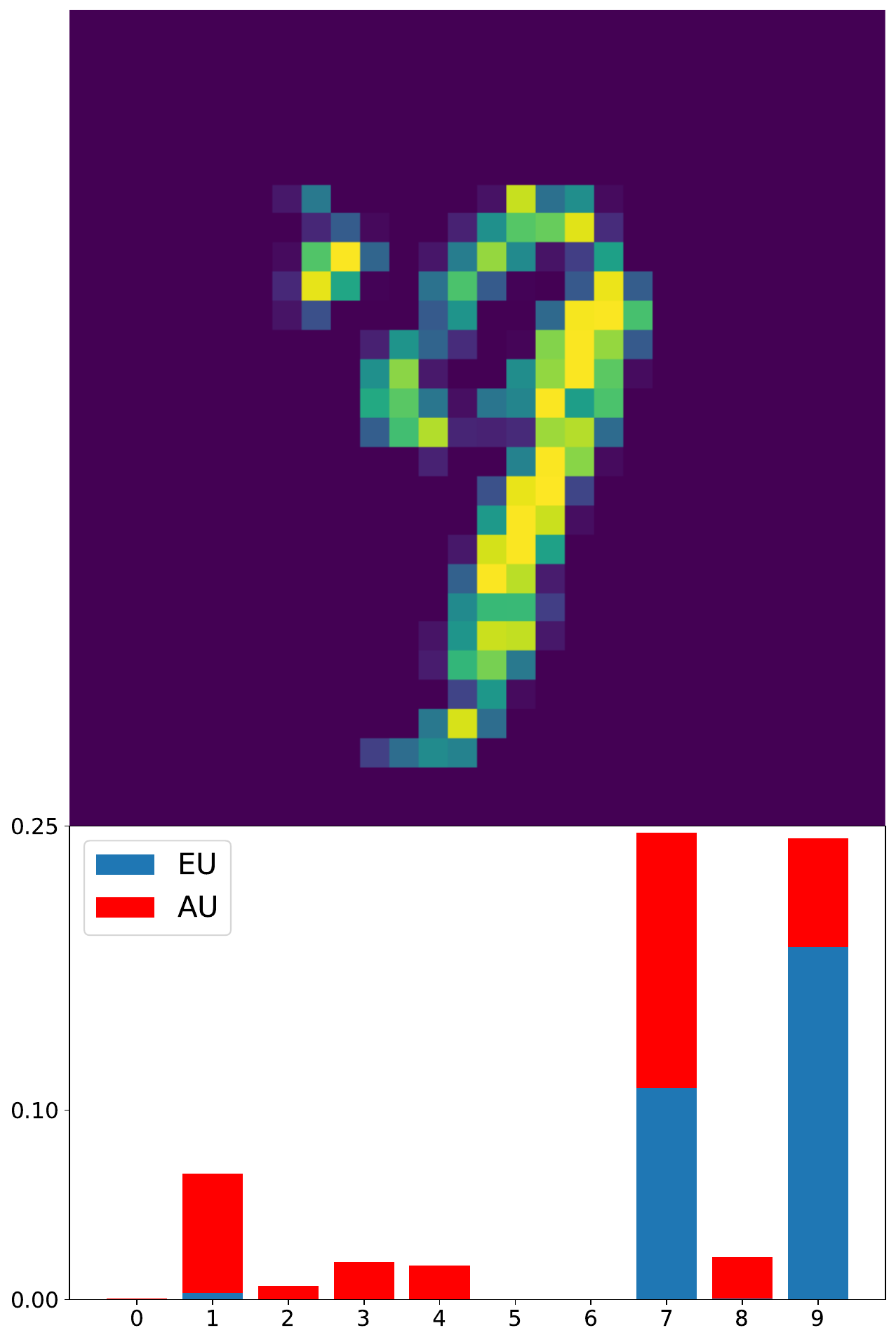}
    \end{subfigure}

    \begin{subfigure}[b]{0.24\textwidth}
        \includegraphics[width=\textwidth]{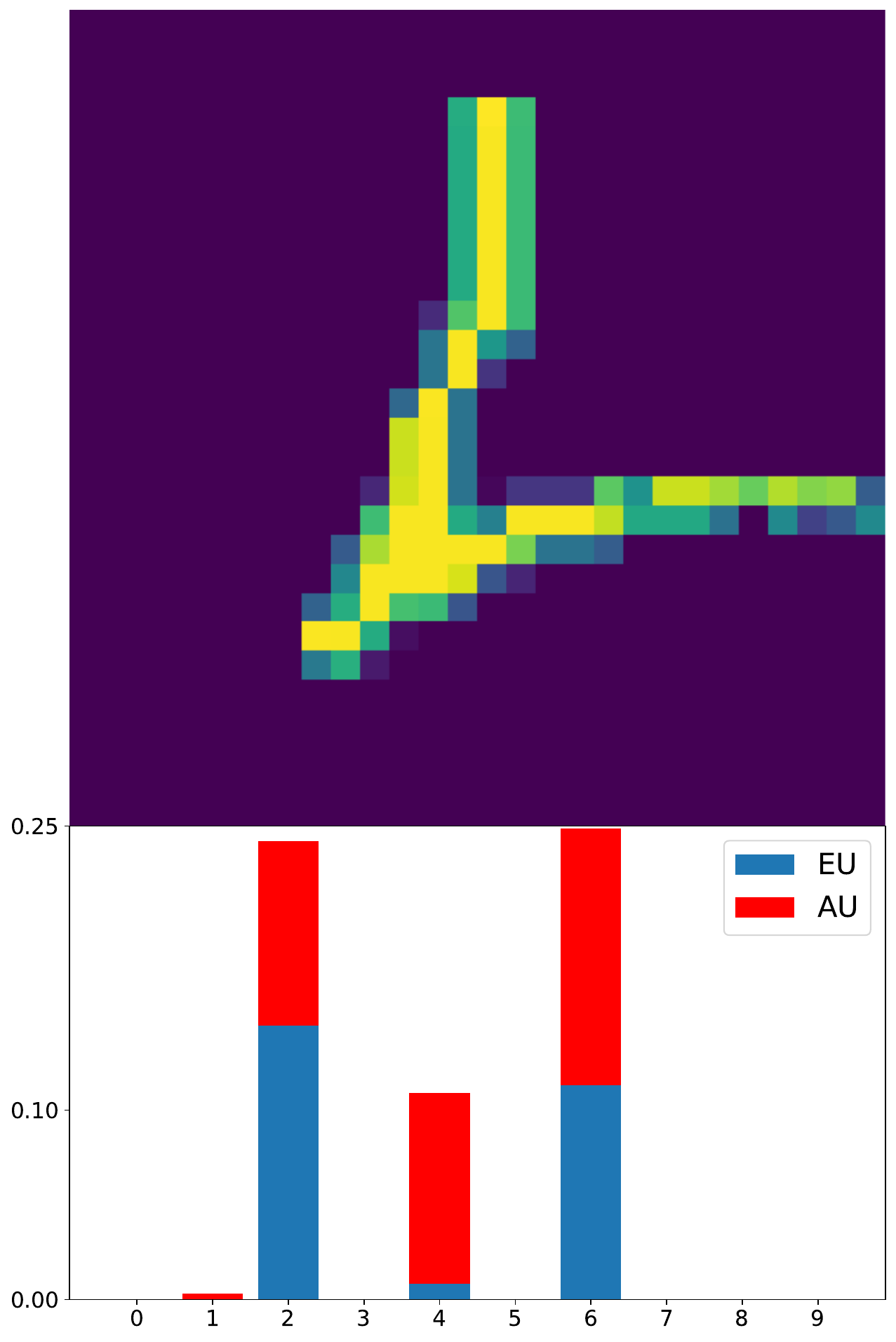}
    \end{subfigure}
    \hspace{1.5cm}
    \begin{subfigure}[b]{0.24\textwidth}
        \includegraphics[width=\textwidth]{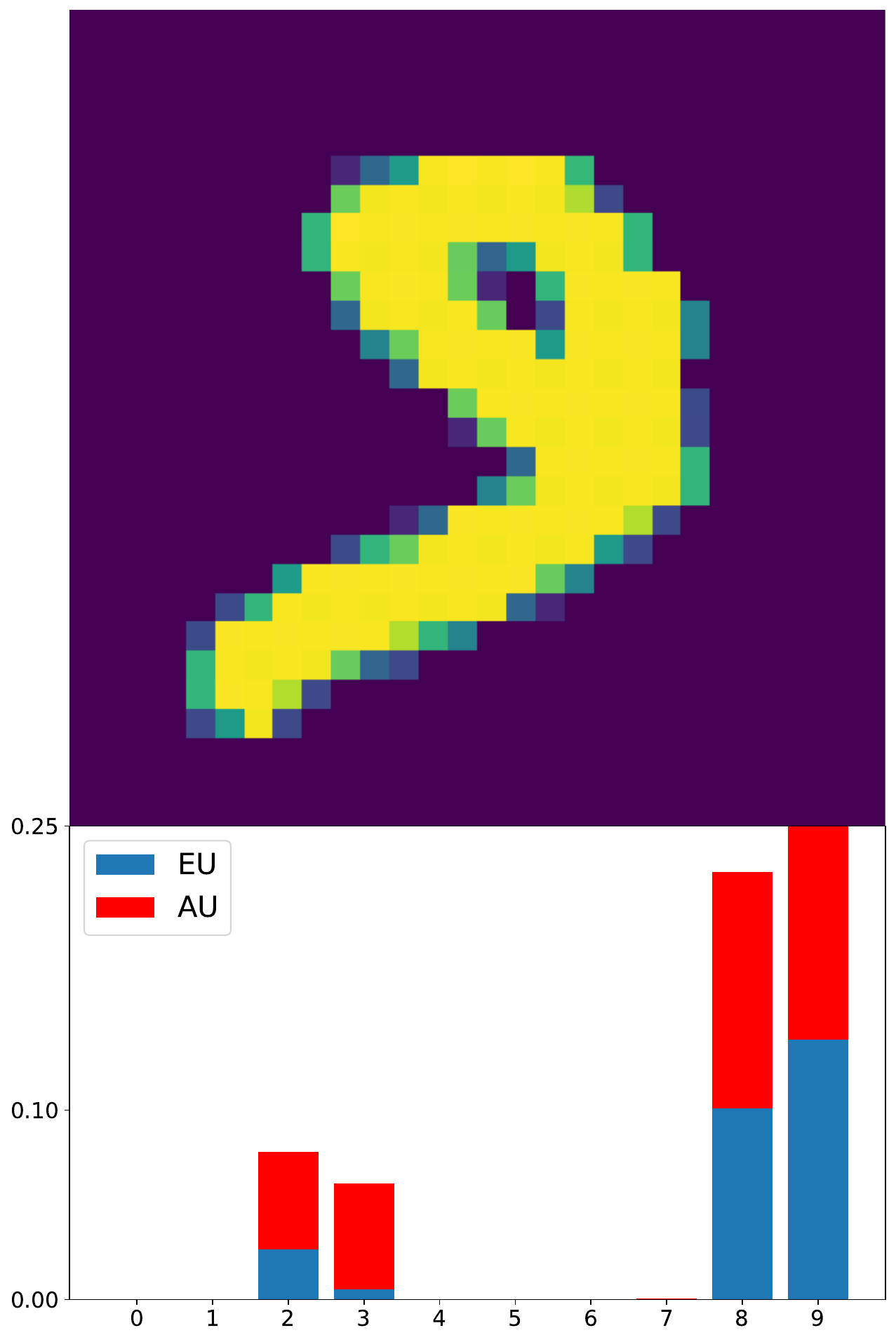}
    \end{subfigure}
    \hspace{1.5cm}
    \begin{subfigure}[b]{0.24\textwidth}
        \includegraphics[width=\textwidth]{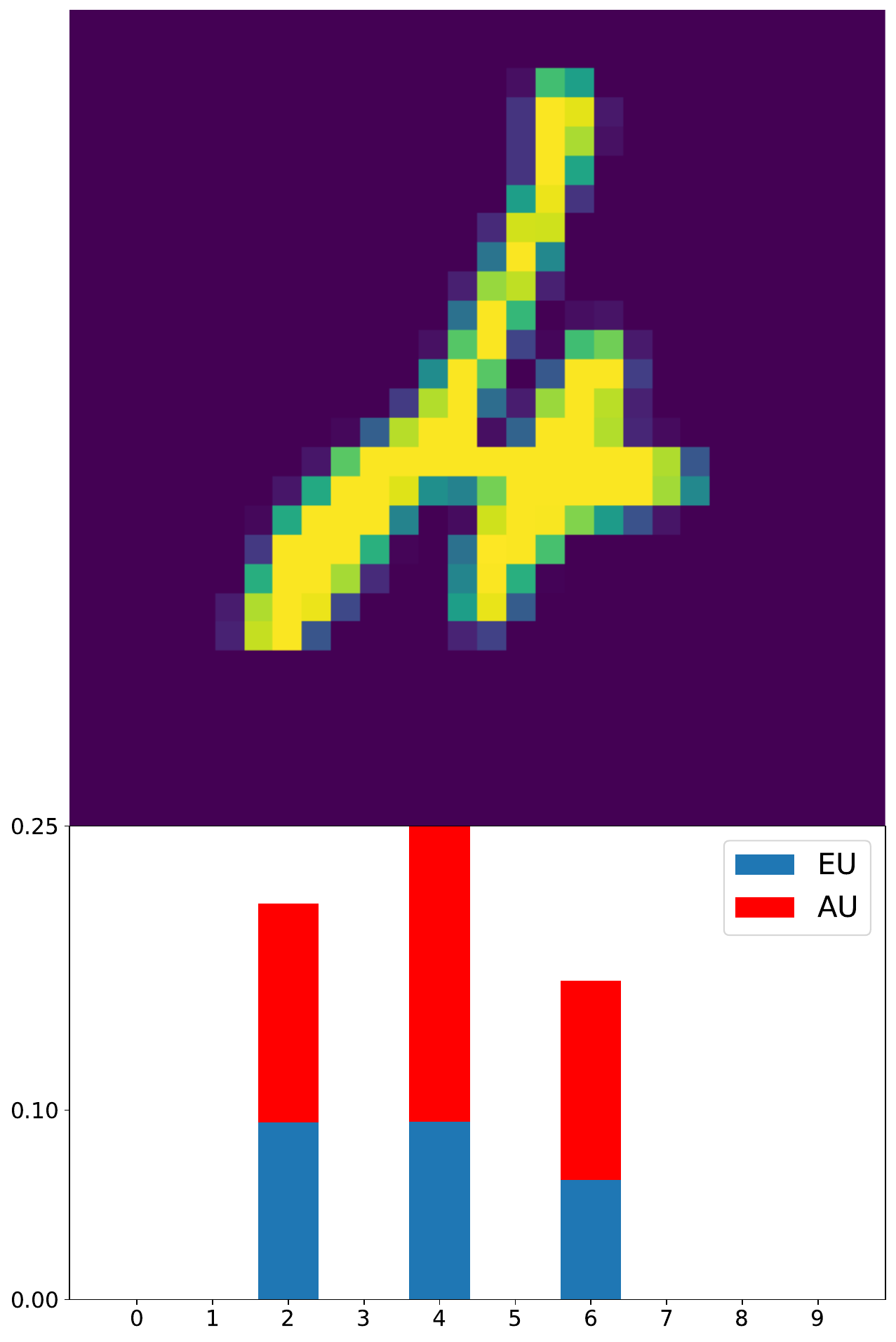}
    \end{subfigure}

    \begin{subfigure}[b]{0.24\textwidth}
        \includegraphics[width=\textwidth]{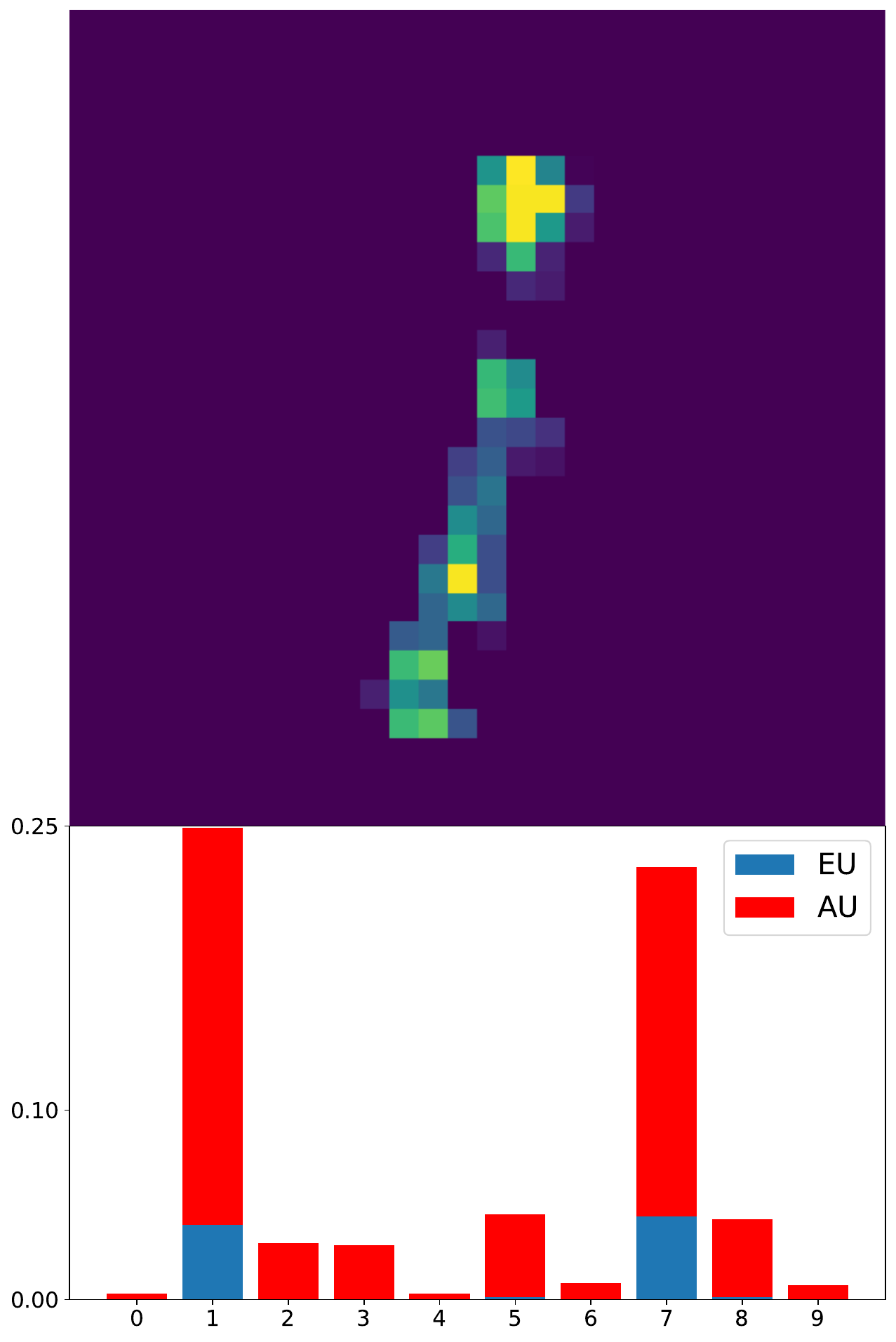}
    \end{subfigure}
    \hspace{1.5cm}
    \begin{subfigure}[b]{0.24\textwidth}
        \includegraphics[width=\textwidth]{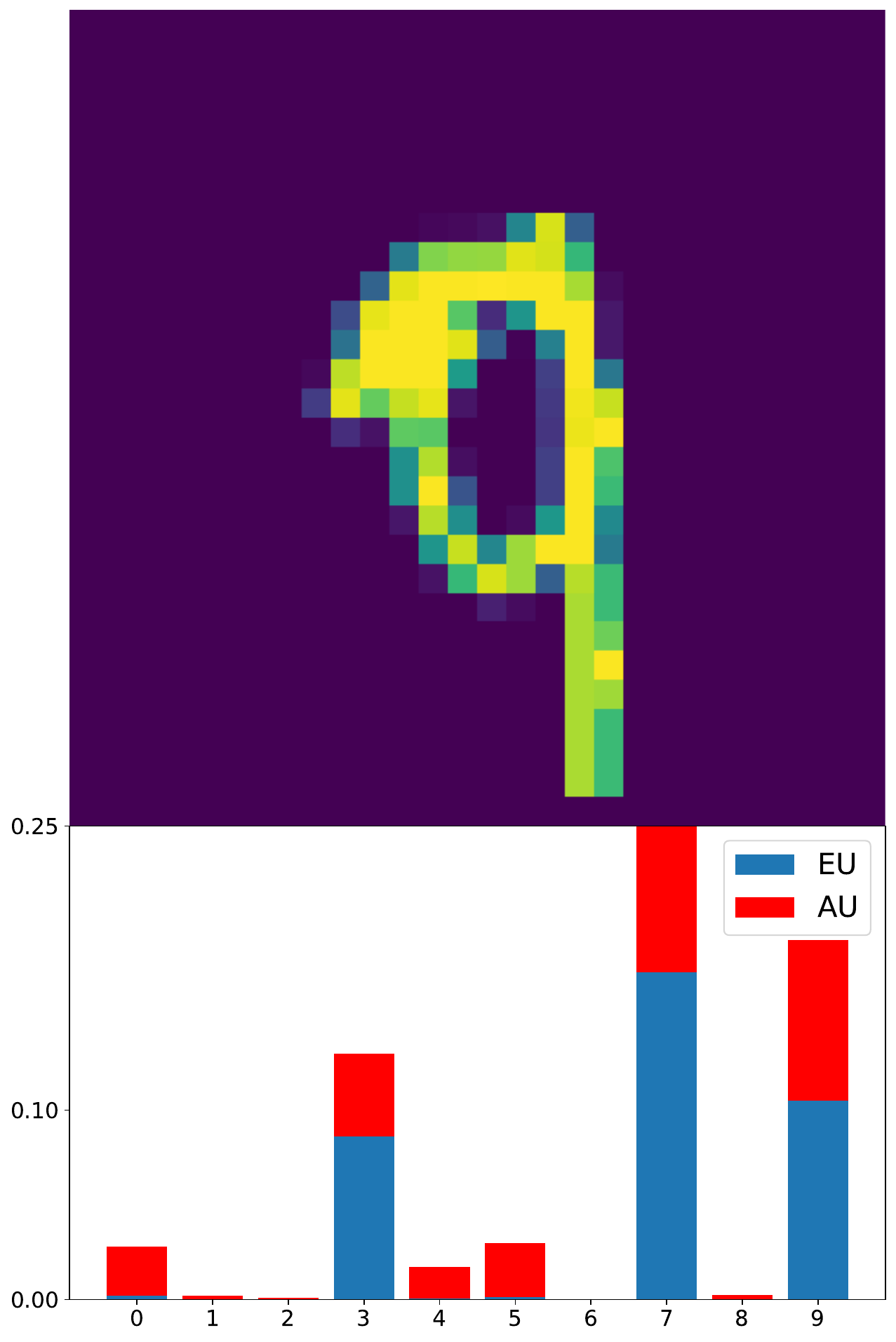}
    \end{subfigure}
    \hspace{1.5cm}
    \begin{subfigure}[b]{0.24\textwidth}
        \includegraphics[width=\textwidth]{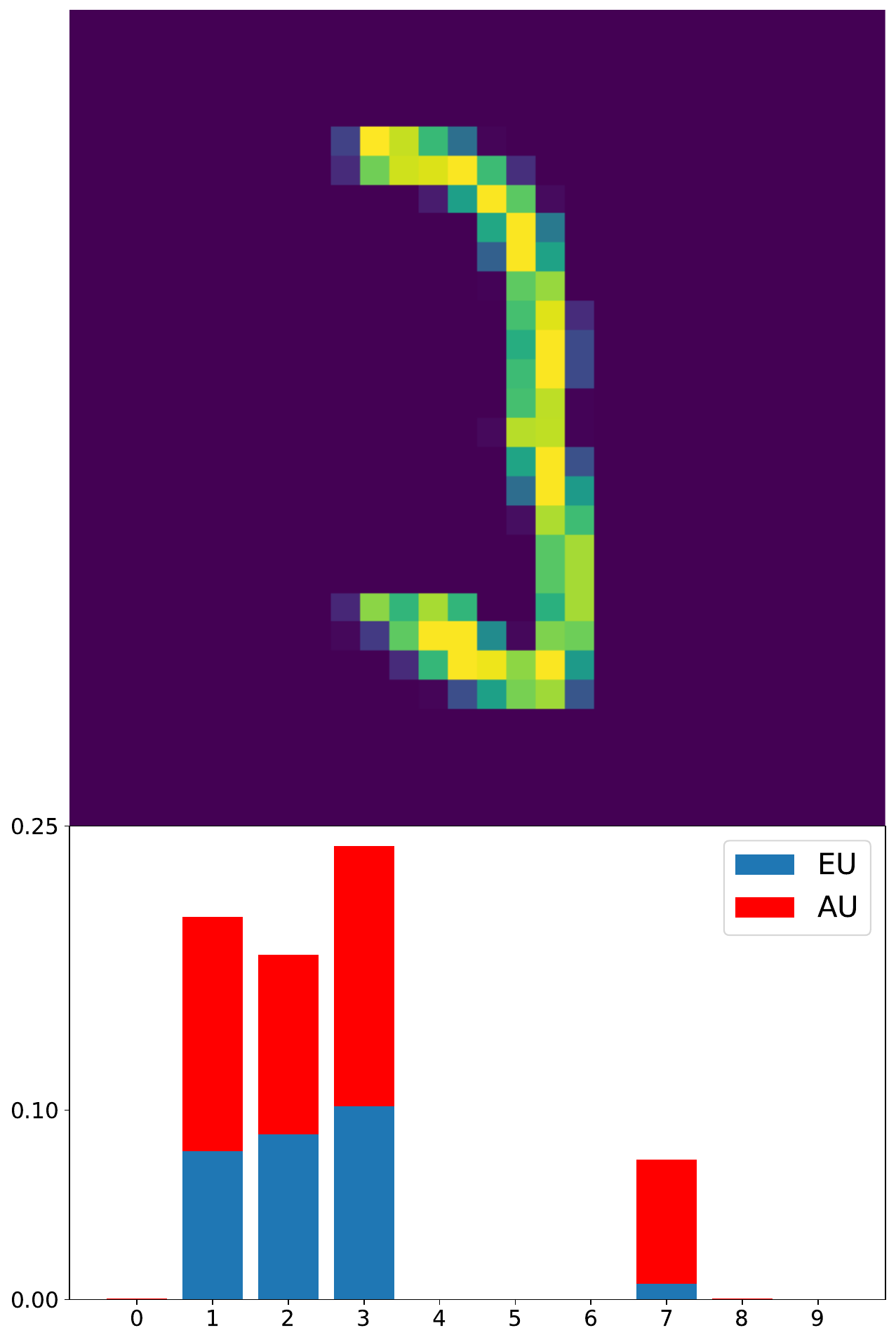}
    \end{subfigure}

    \caption{MNIST instances with greatest total uncertainty and their respective label-wise uncertainties.}
    \label{fig:tu_add}
\end{figure}

\begin{figure}[htbp]
    \centering

    \begin{subfigure}[b]{0.24\textwidth}
        \includegraphics[width=\textwidth]{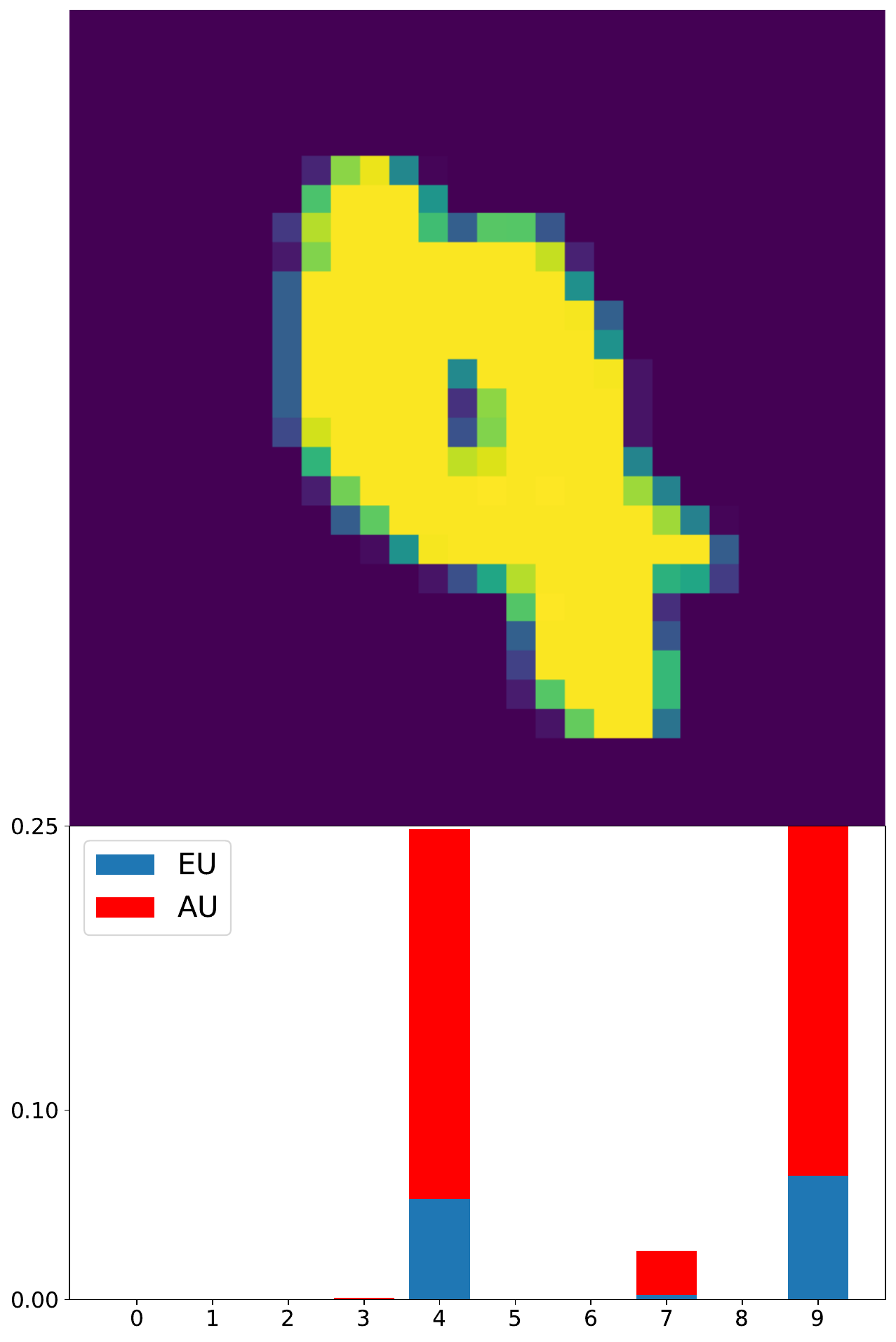}
    \end{subfigure}
    \hspace{1.5cm}
    \begin{subfigure}[b]{0.24\textwidth}
        \includegraphics[width=\textwidth]{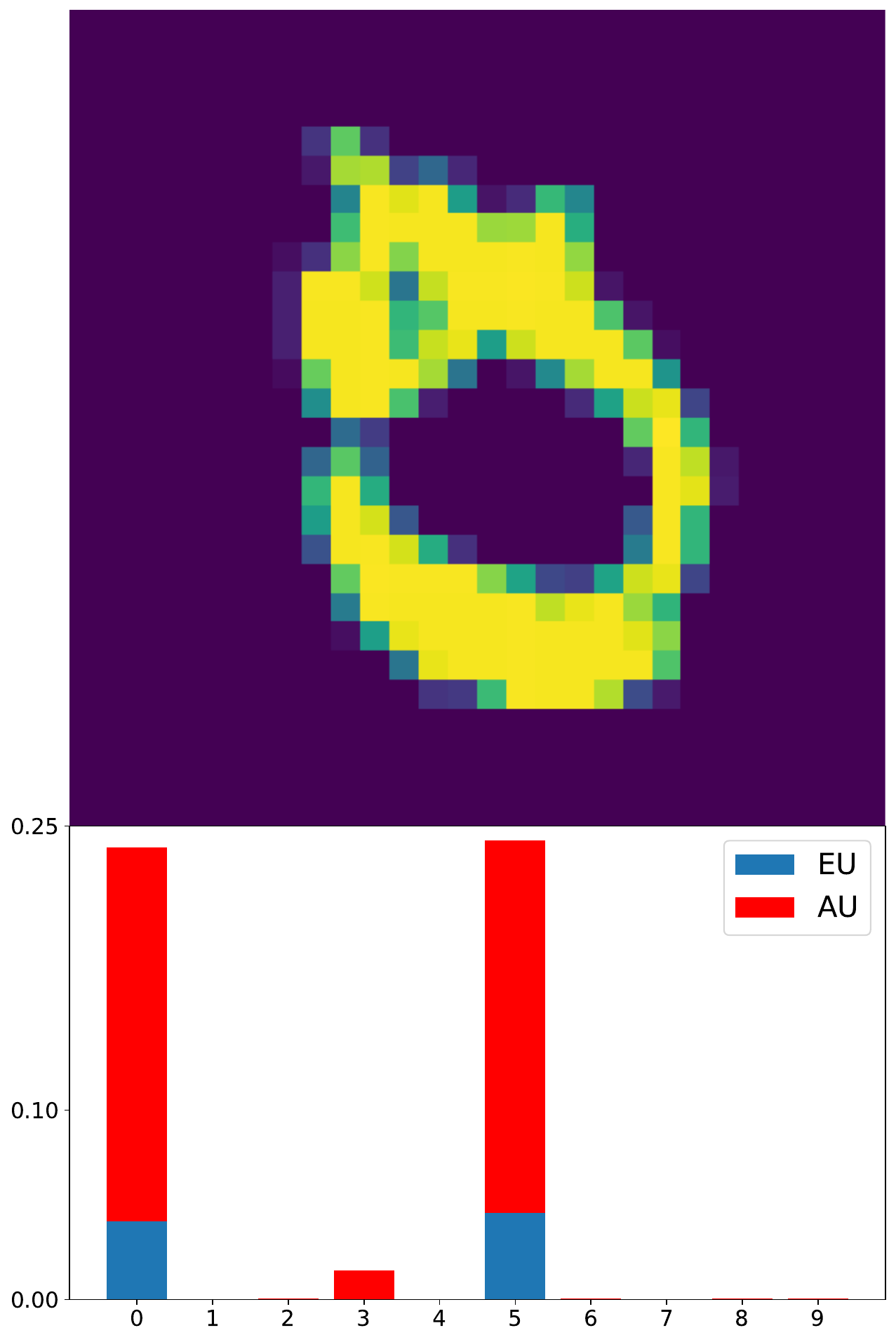}
    \end{subfigure}
    \hspace{1.5cm}
    \begin{subfigure}[b]{0.24\textwidth}
        \includegraphics[width=\textwidth]{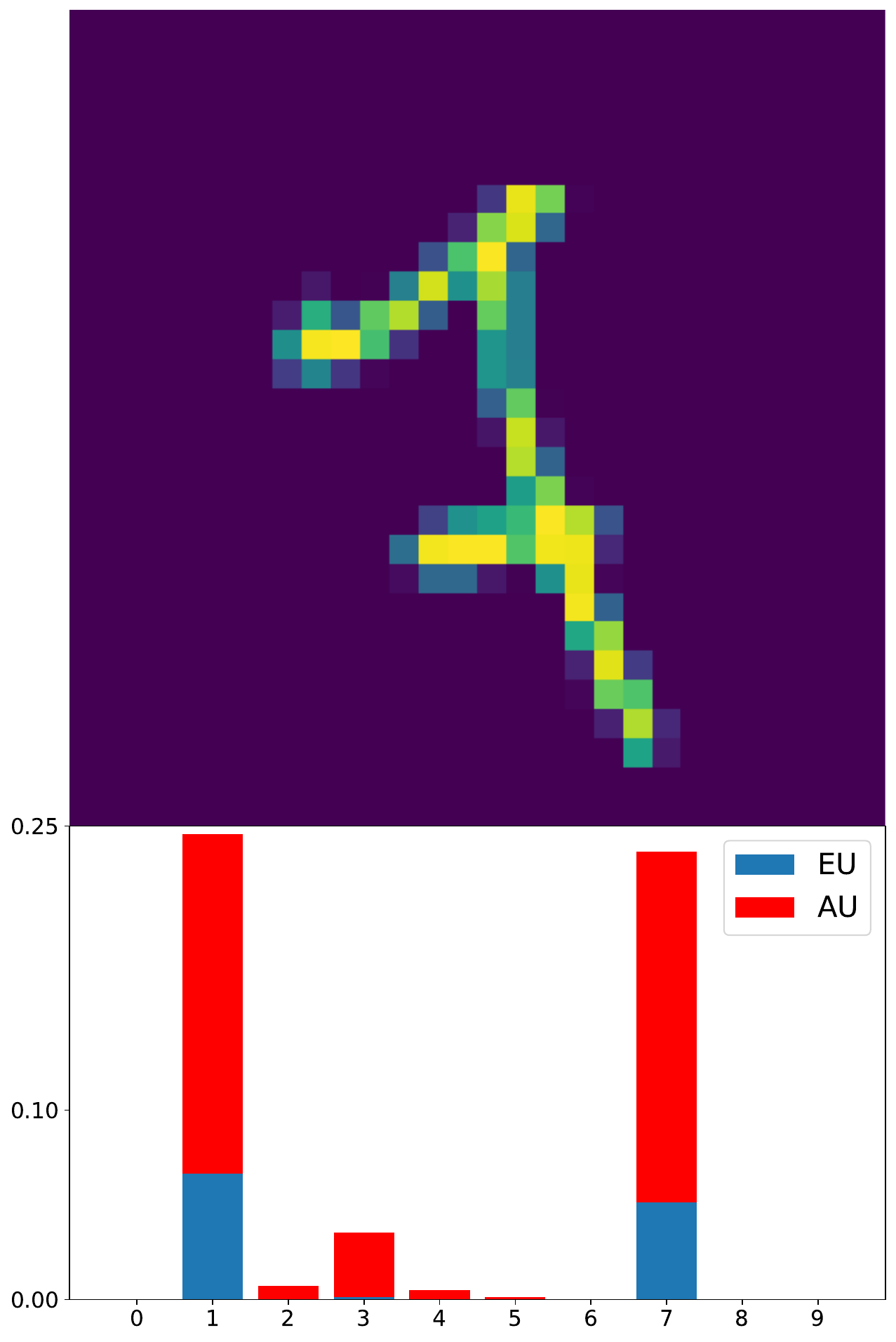}
    \end{subfigure}

    \begin{subfigure}[b]{0.24\textwidth}
        \includegraphics[width=\textwidth]{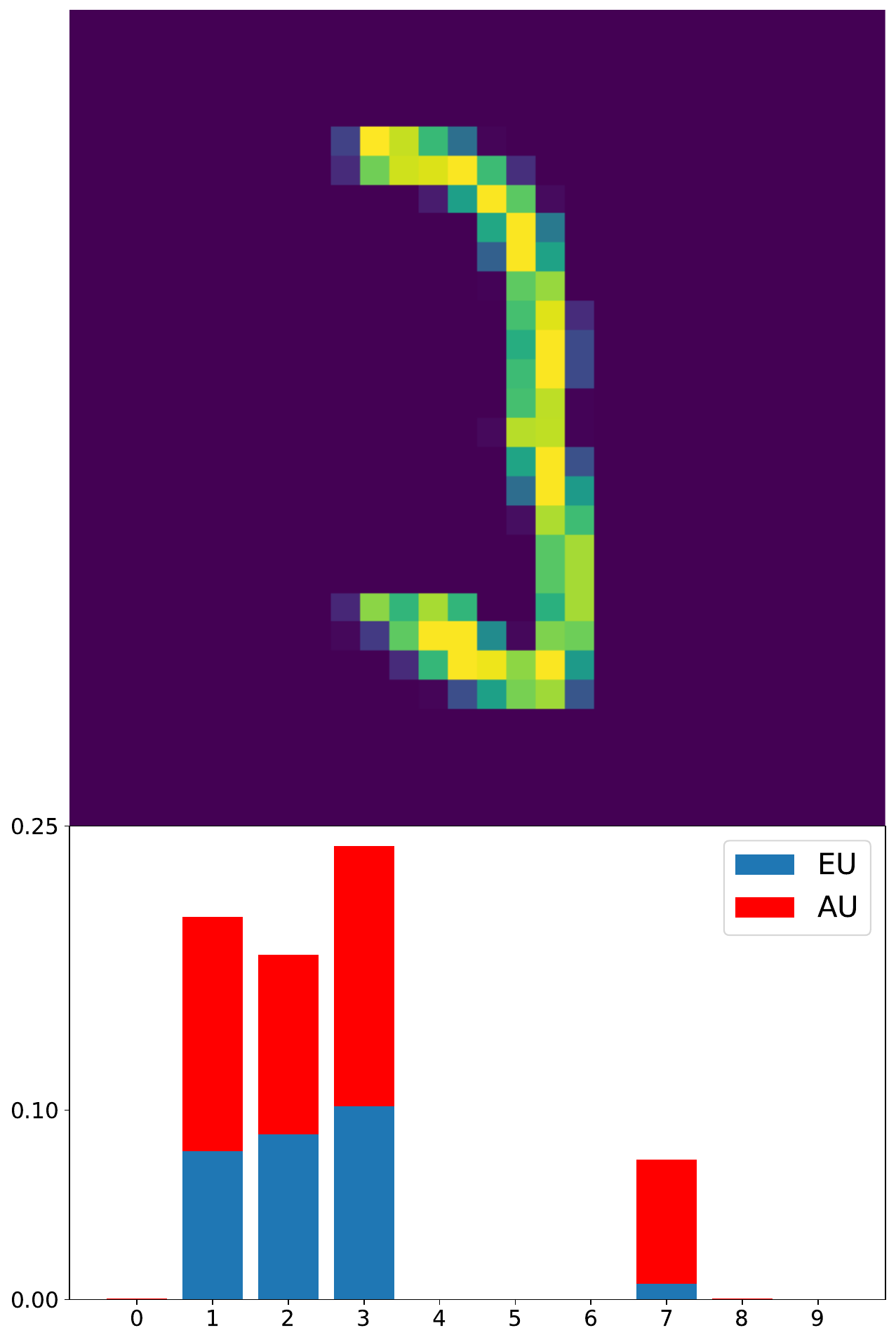}
    \end{subfigure}
    \hspace{1.5cm}
    \begin{subfigure}[b]{0.24\textwidth}
        \includegraphics[width=\textwidth]{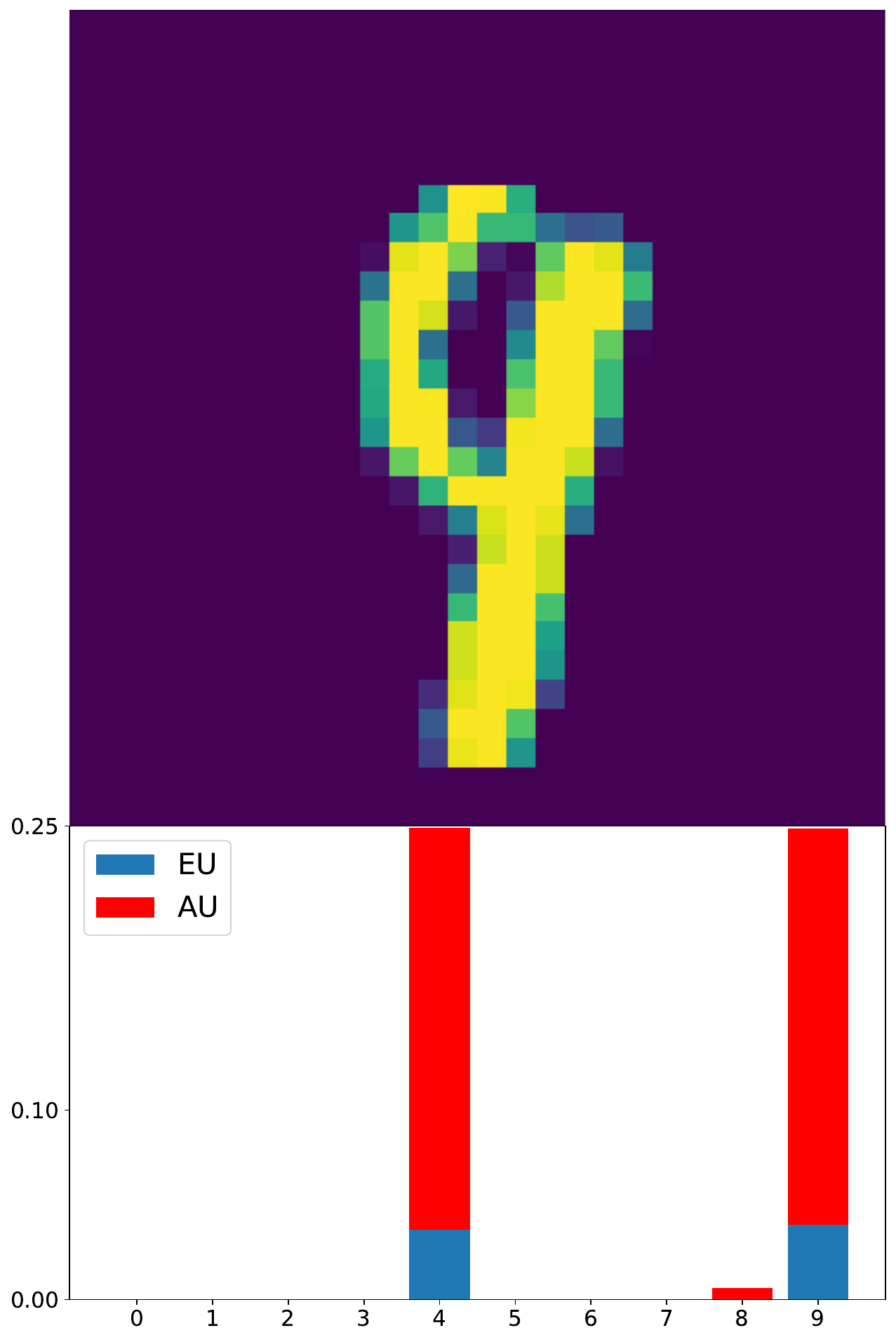}
    \end{subfigure}
    \hspace{1.5cm}
    \begin{subfigure}[b]{0.24\textwidth}
        \includegraphics[width=\textwidth]{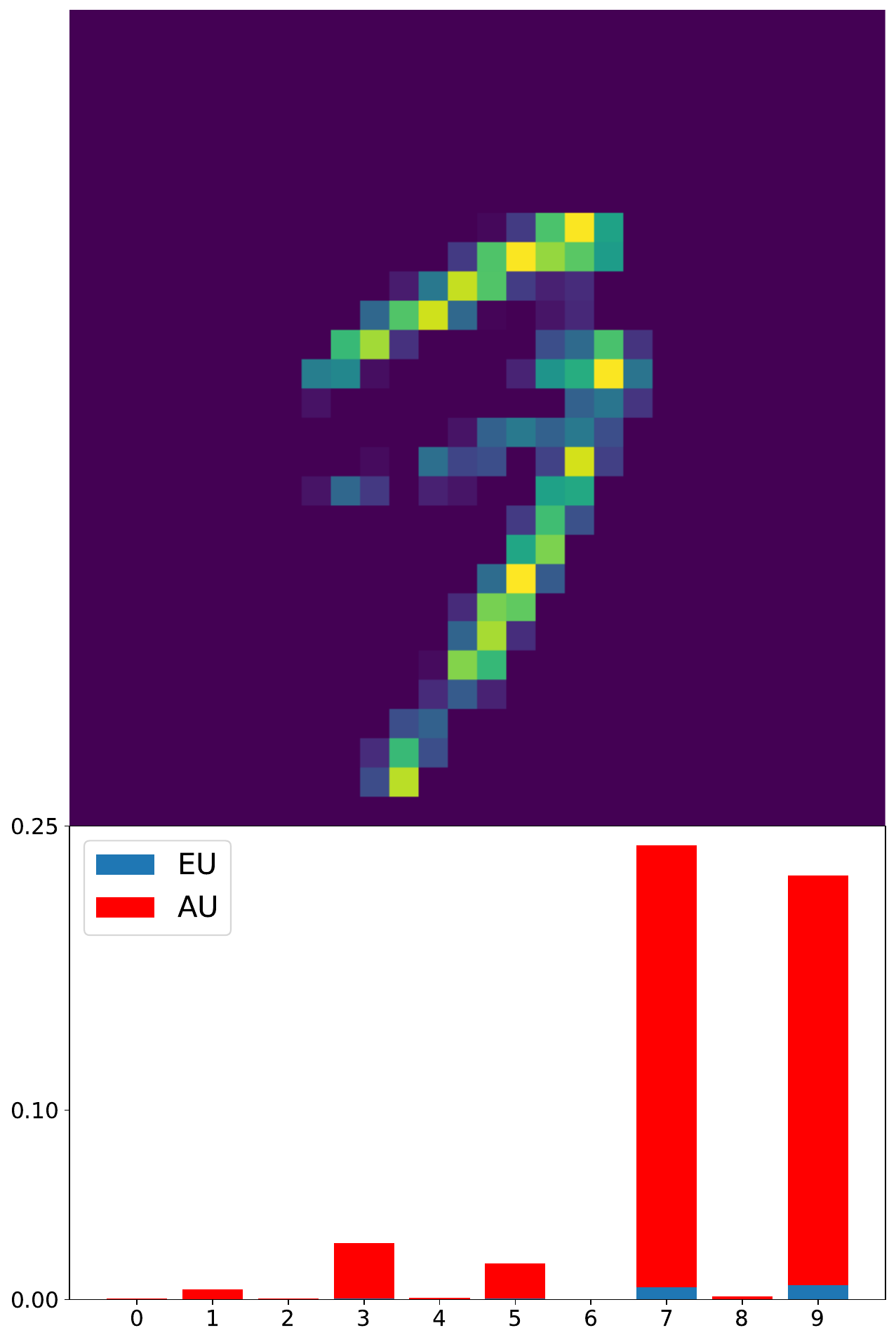}
    \end{subfigure}

    \begin{subfigure}[b]{0.24\textwidth}
        \includegraphics[width=\textwidth]{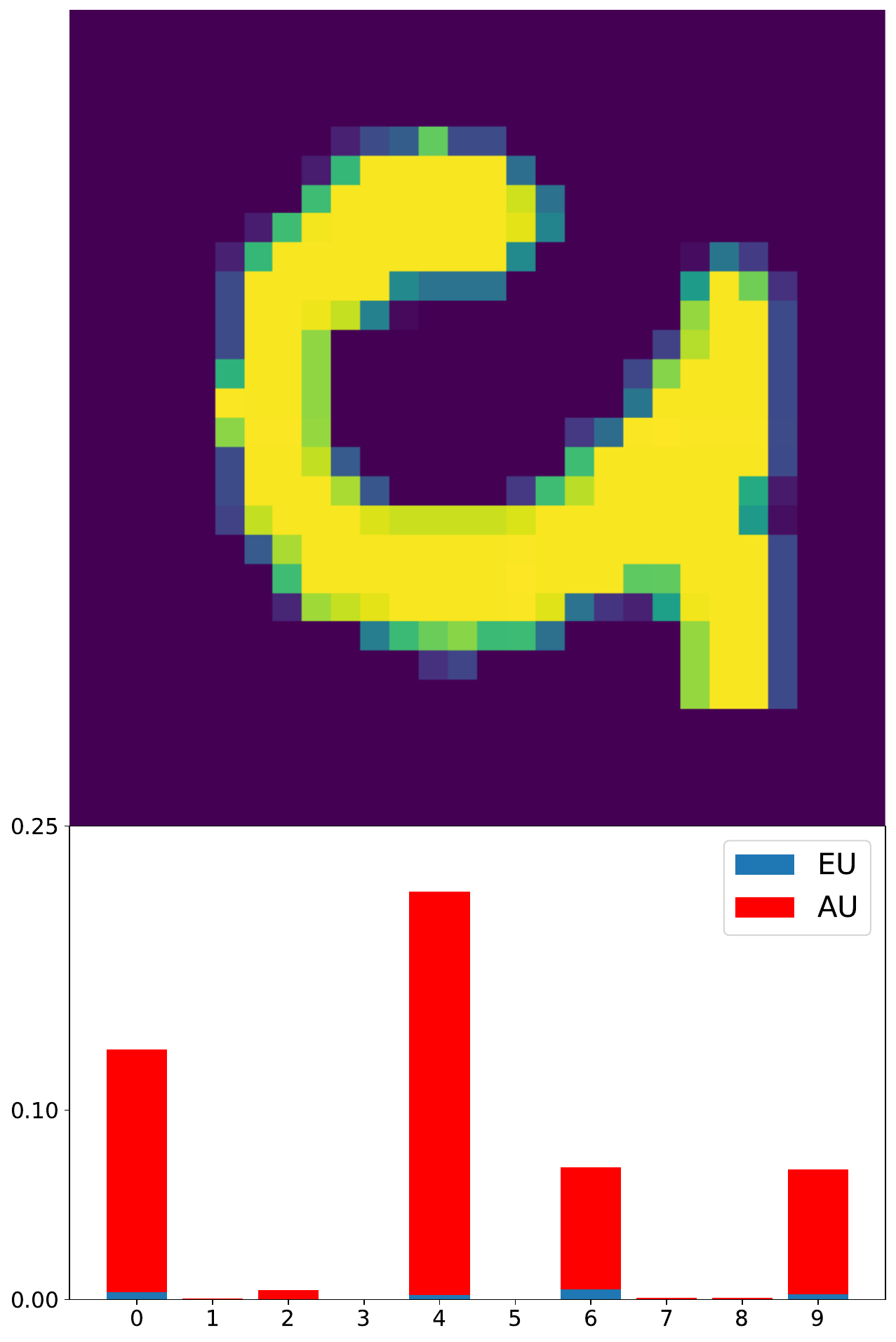}
    \end{subfigure}
    \hspace{1.5cm}
    \begin{subfigure}[b]{0.24\textwidth}
        \includegraphics[width=\textwidth]{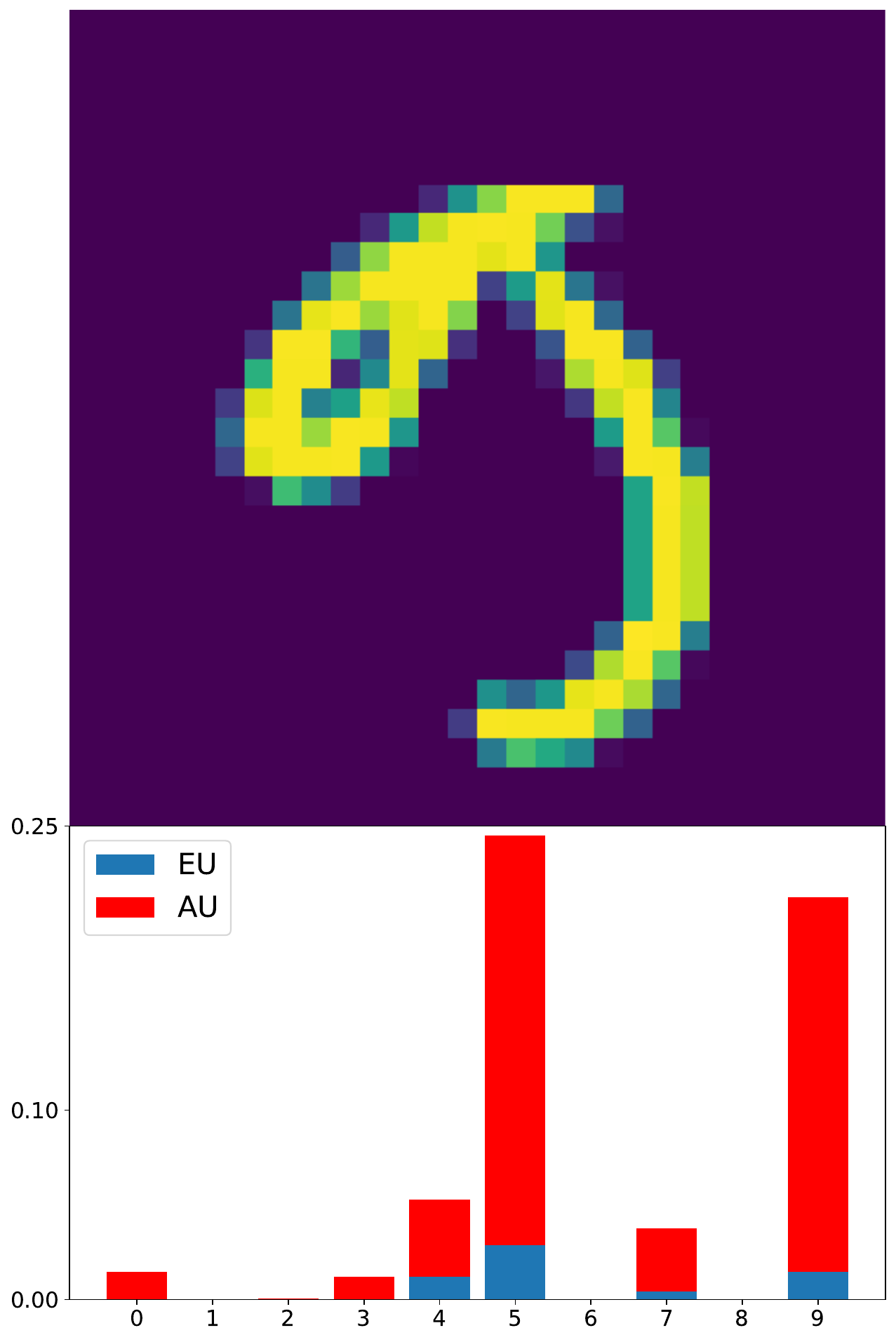}
    \end{subfigure}
    \hspace{1.5cm}
    \begin{subfigure}[b]{0.24\textwidth}
        \includegraphics[width=\textwidth]{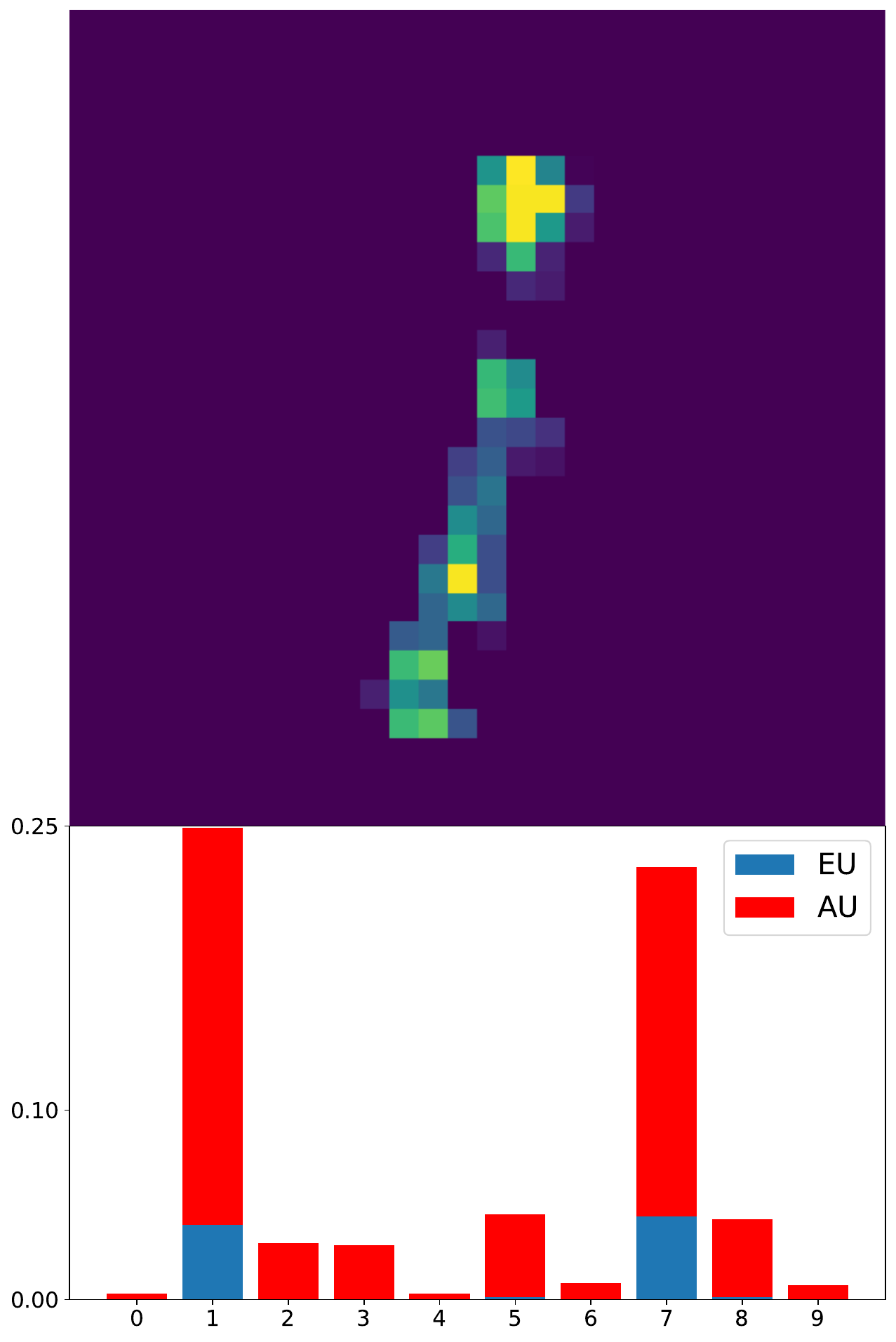}
    \end{subfigure}

    \caption{MNIST instances with greatest aleatoric uncertainty and their respective label-wise uncertainties.}
    \label{fig:au_add}
\end{figure}

\begin{figure}[htbp]
    \centering

    \begin{subfigure}[b]{0.24\textwidth}
        \includegraphics[width=\textwidth]{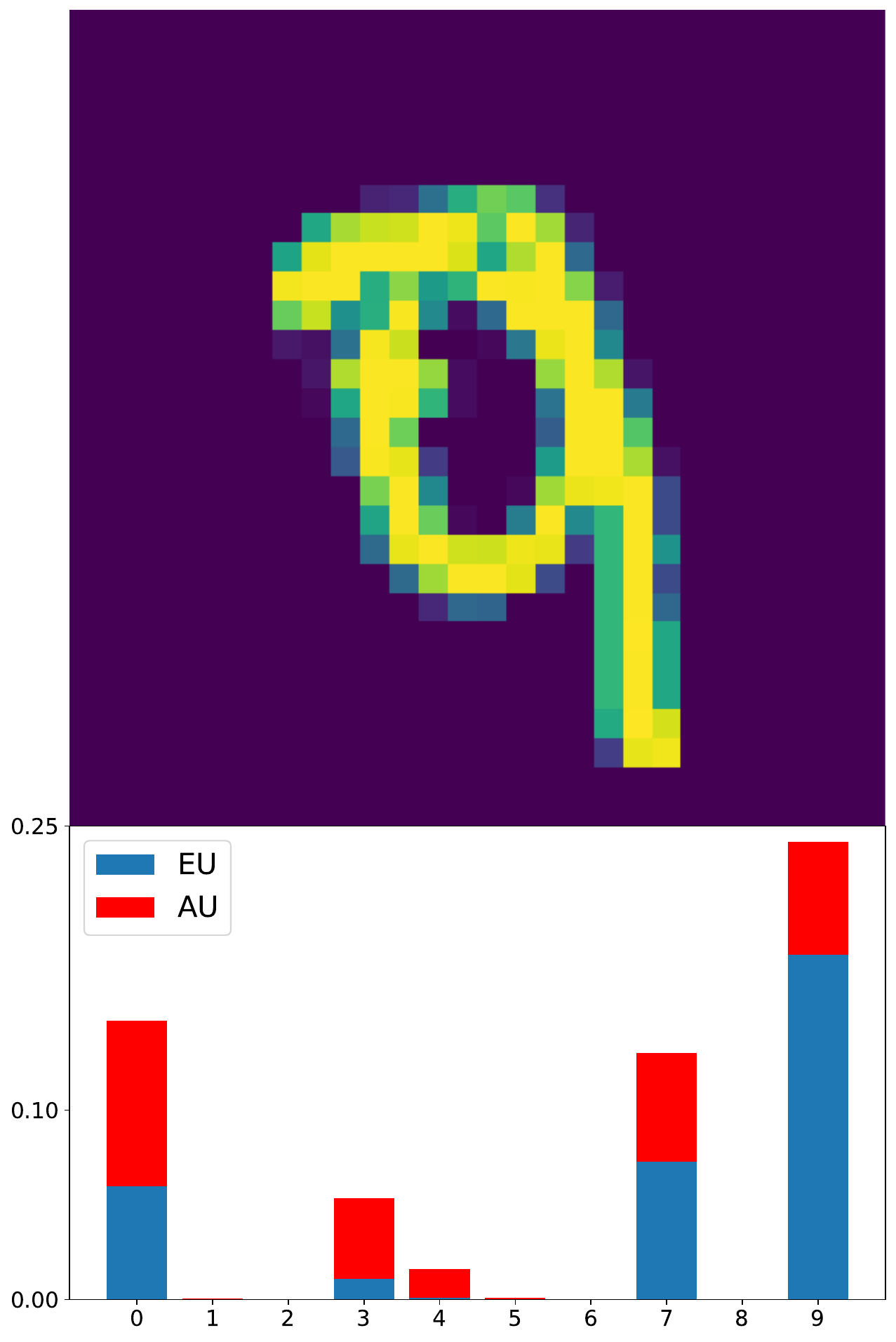}
    \end{subfigure}
    \hspace{1.5cm}
    \begin{subfigure}[b]{0.24\textwidth}
        \includegraphics[width=\textwidth]{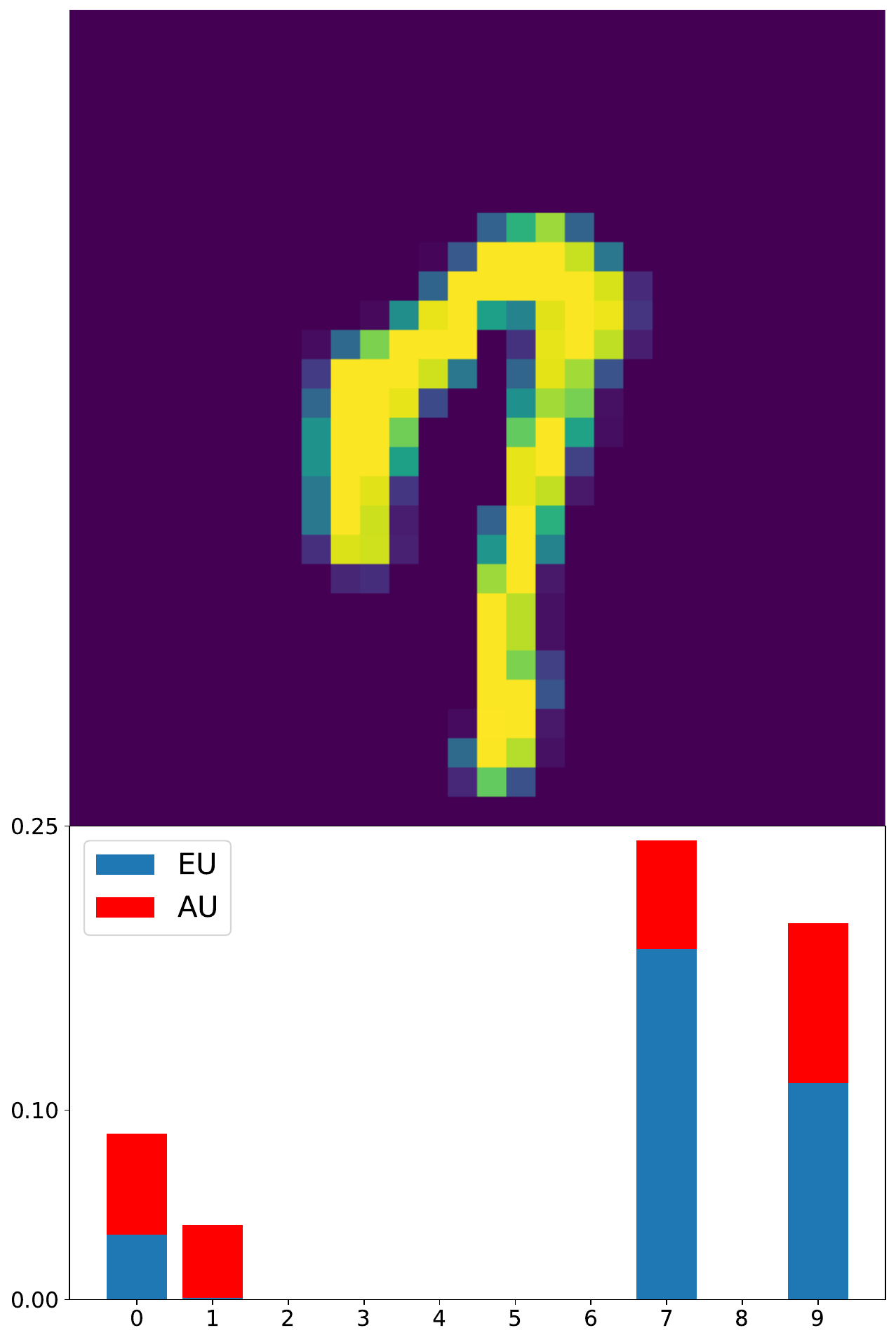}
    \end{subfigure}
    \hspace{1.5cm}
    \begin{subfigure}[b]{0.24\textwidth}
        \includegraphics[width=\textwidth]{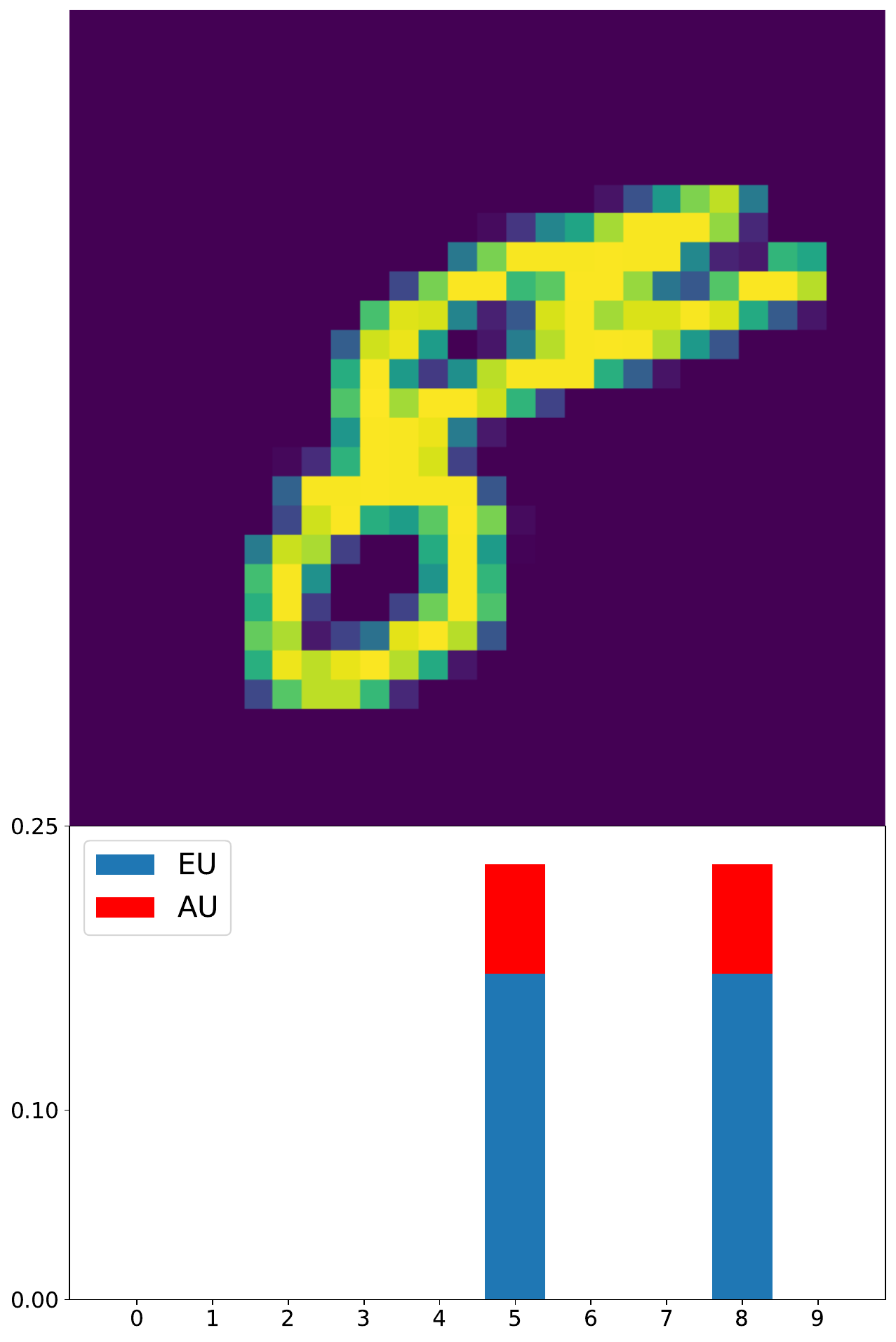}
    \end{subfigure}

    \begin{subfigure}[b]{0.24\textwidth}
        \includegraphics[width=\textwidth]{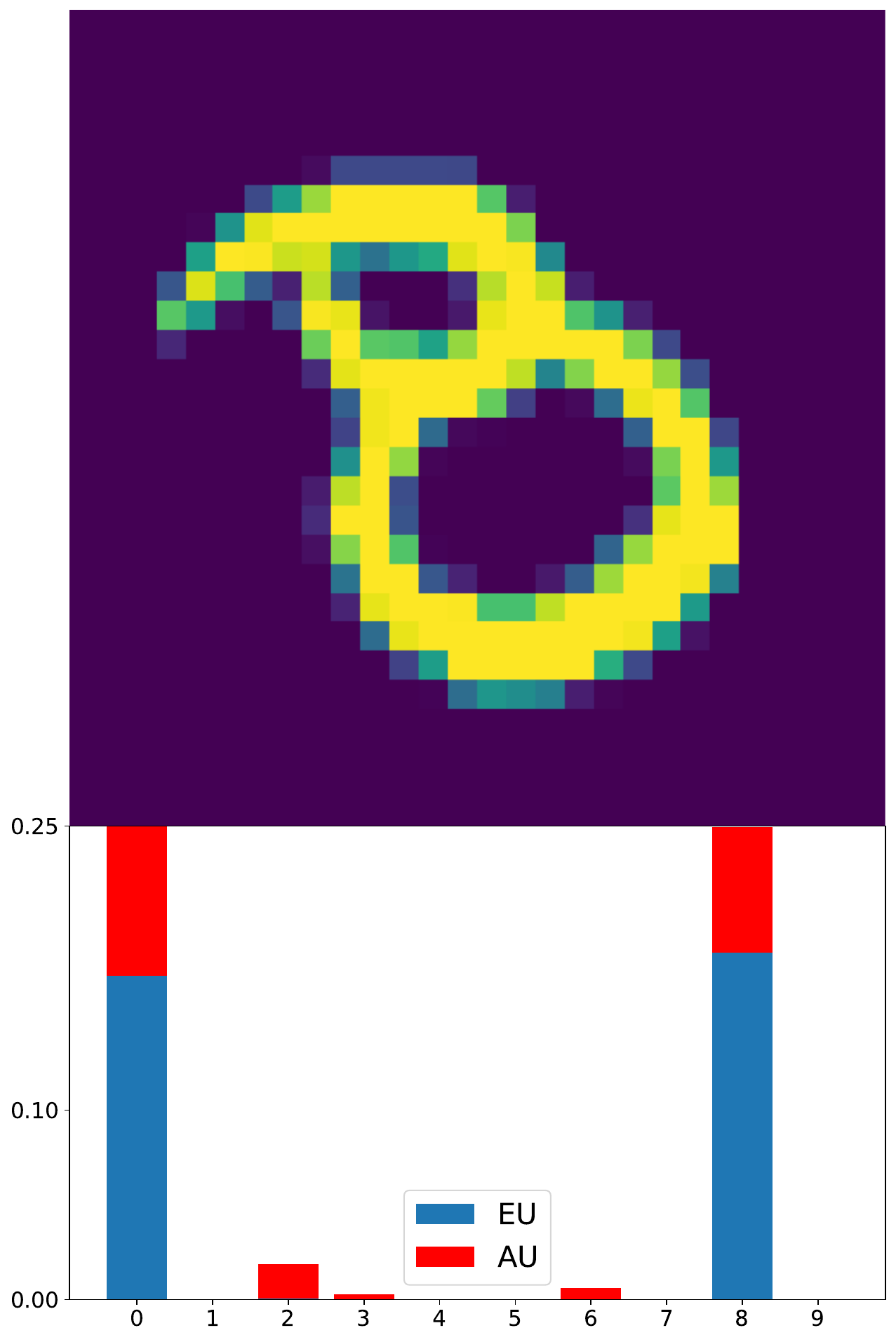}
    \end{subfigure}
    \hspace{1.5cm}
    \begin{subfigure}[b]{0.24\textwidth}
        \includegraphics[width=\textwidth]{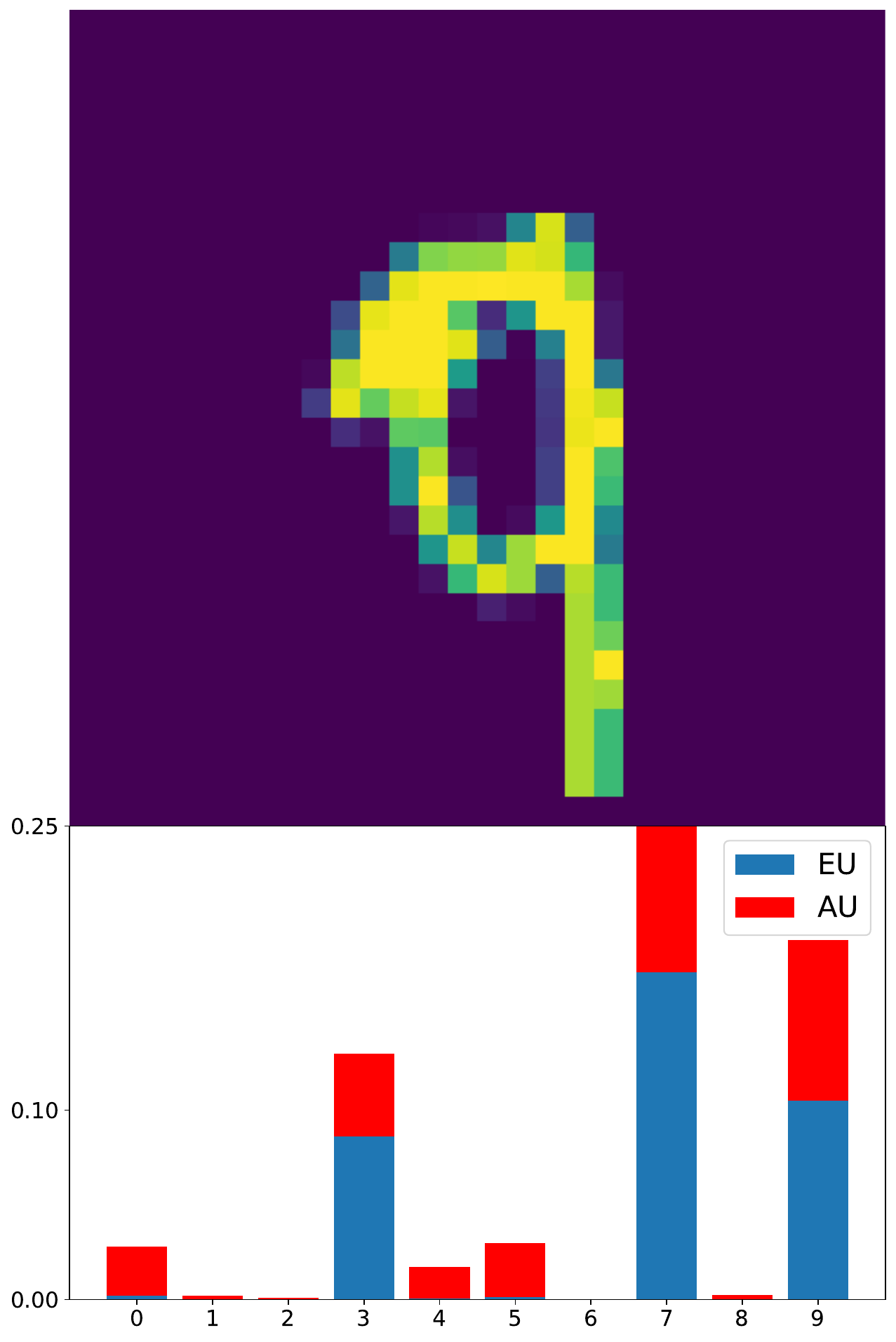}
    \end{subfigure}
    \hspace{1.5cm}
    \begin{subfigure}[b]{0.24\textwidth}
        \includegraphics[width=\textwidth]{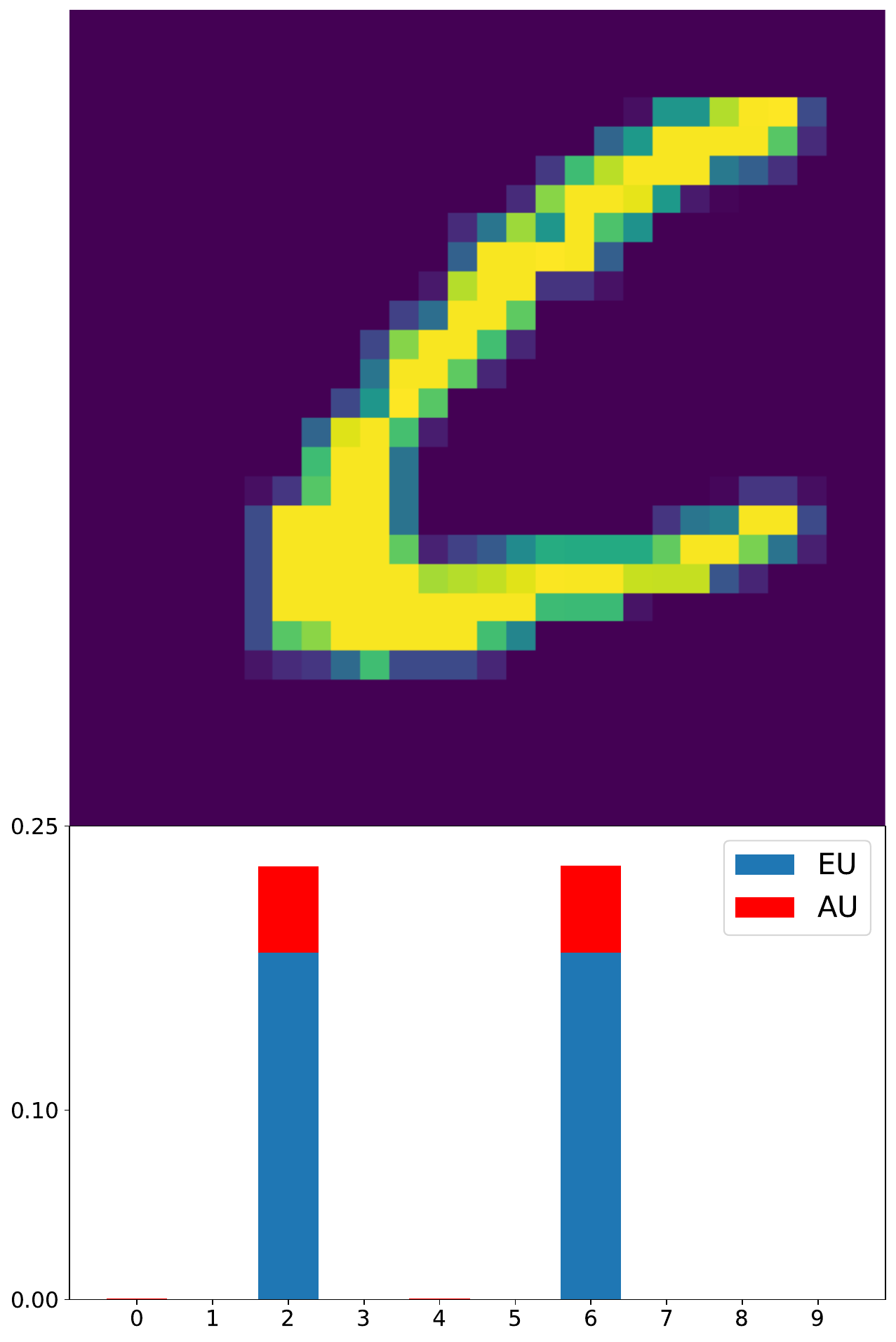}
    \end{subfigure}

    \begin{subfigure}[b]{0.24\textwidth}
        \includegraphics[width=\textwidth]{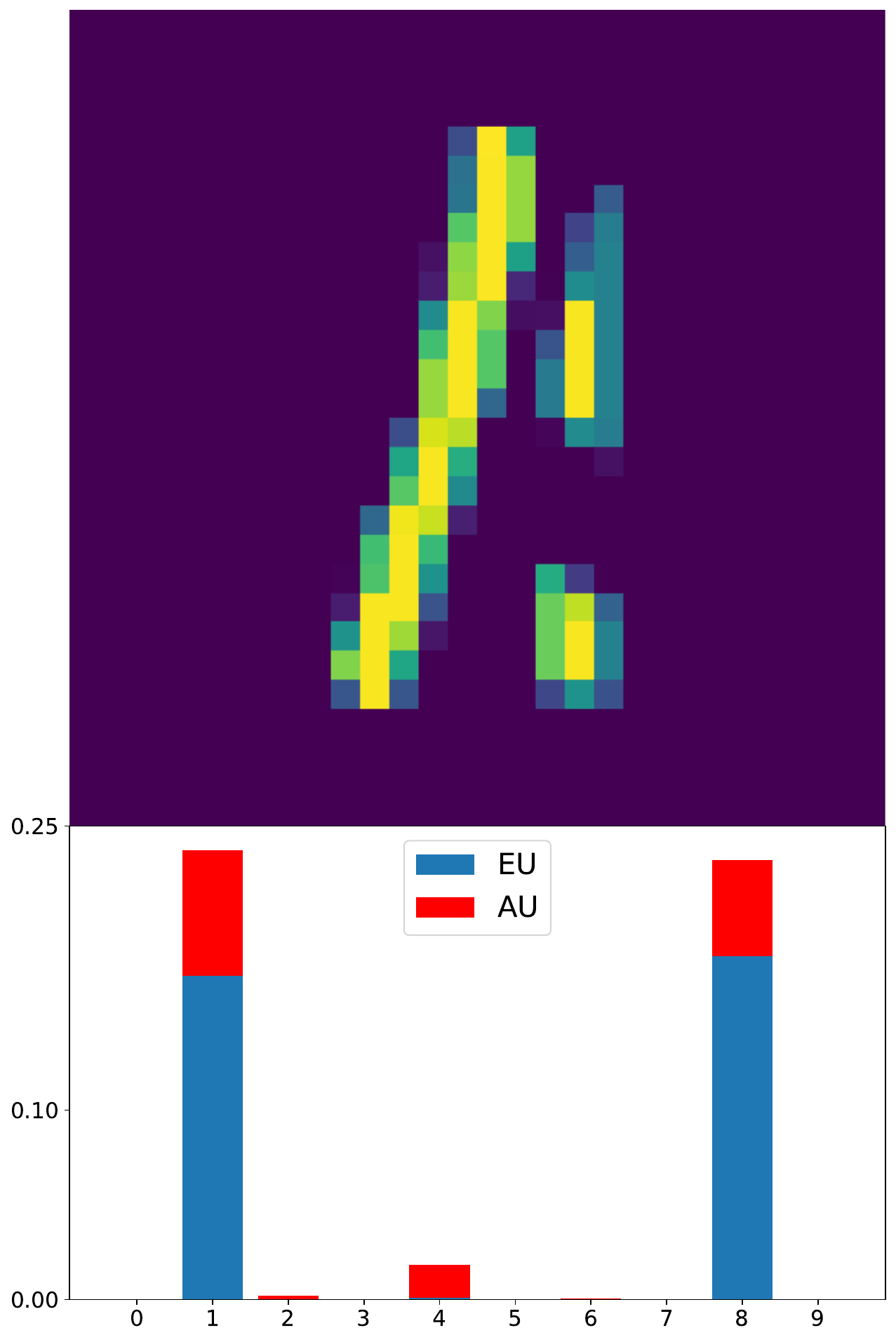}
    \end{subfigure}
    \hspace{1.5cm}
    \begin{subfigure}[b]{0.24\textwidth}
        \includegraphics[width=\textwidth]{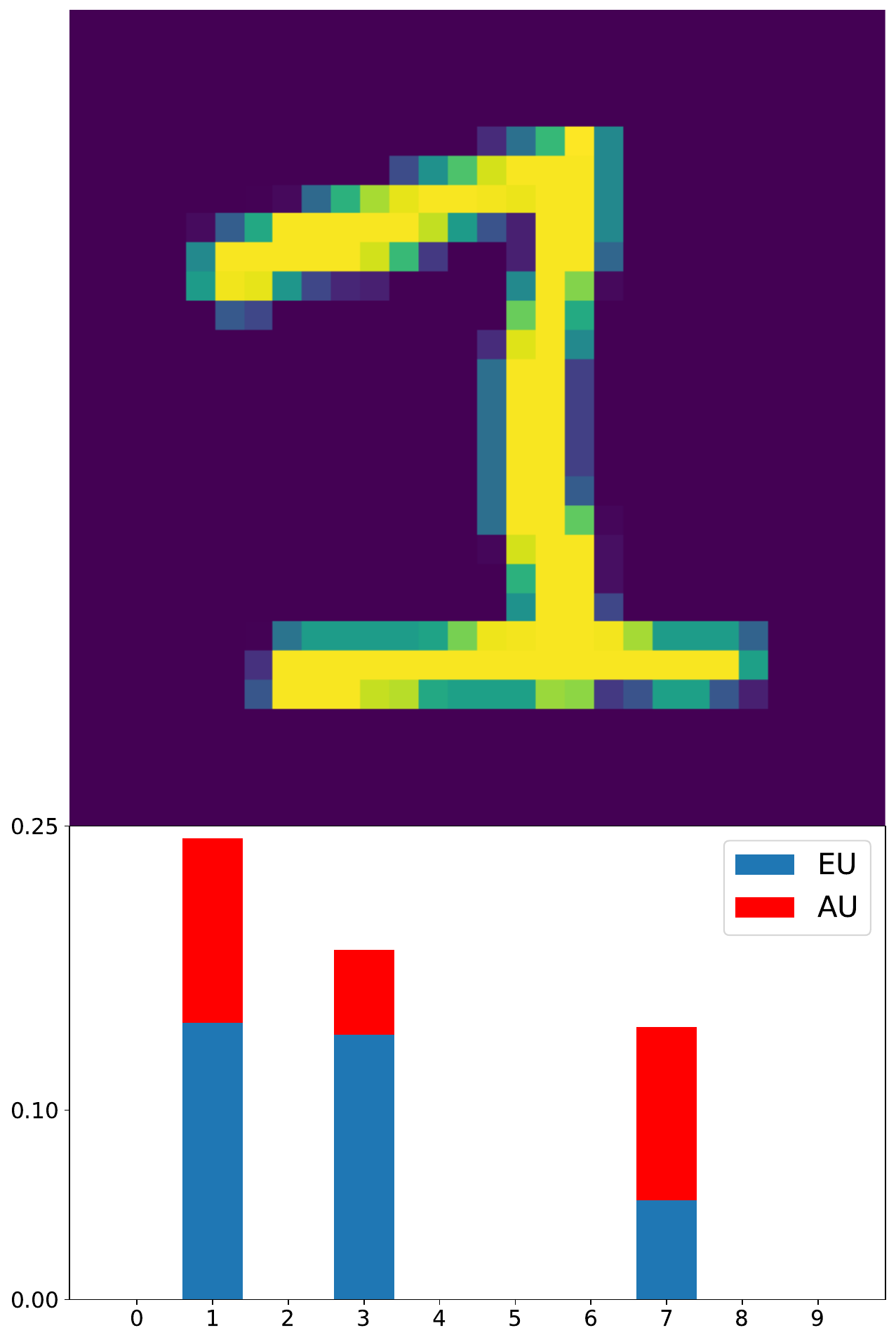}
    \end{subfigure}
    \hspace{1.5cm}
    \begin{subfigure}[b]{0.24\textwidth}
        \includegraphics[width=\textwidth]{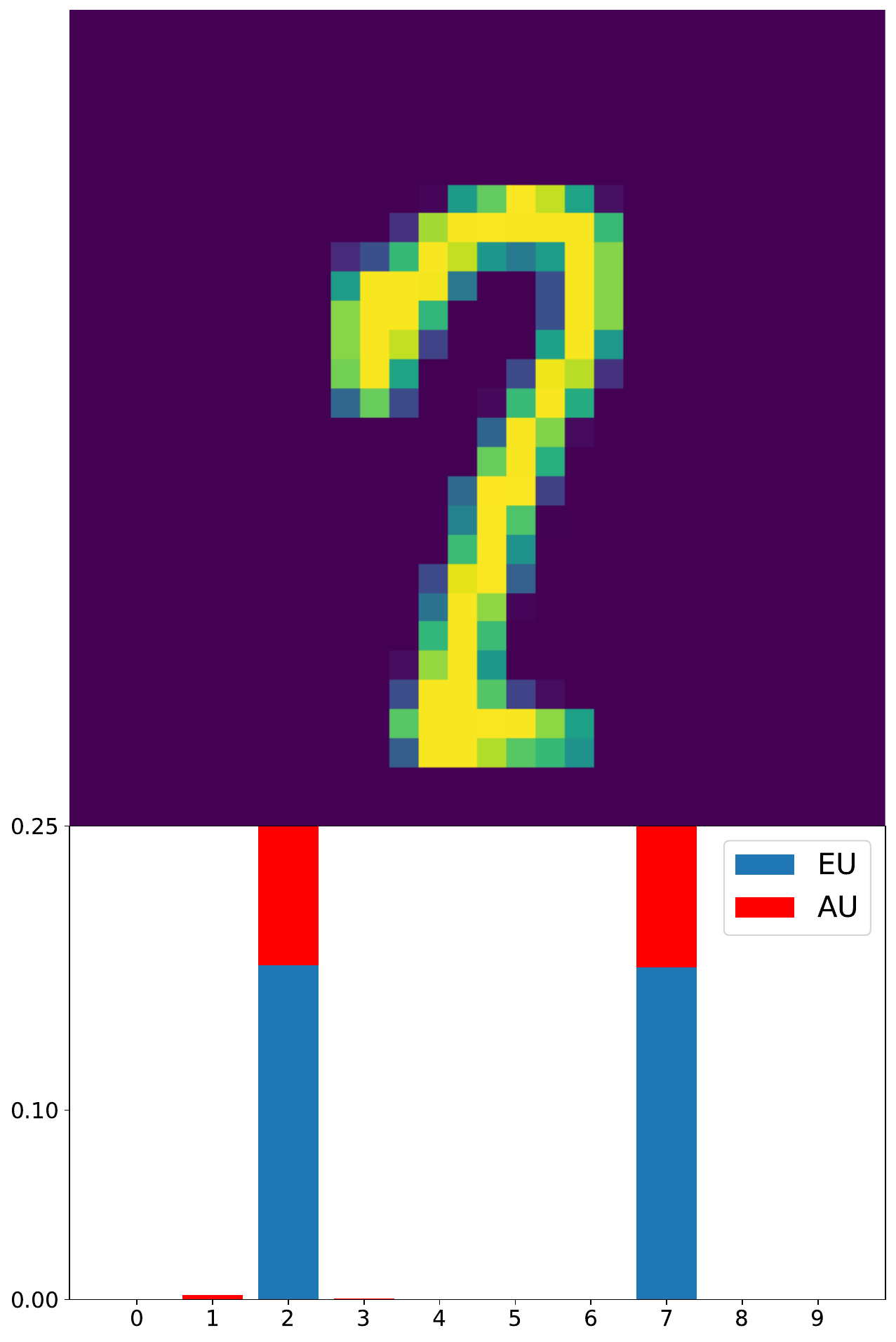}
    \end{subfigure}

    \caption{MNIST instances with greatest epistemic uncertainty and their respective label-wise uncertainties.}
    \label{fig:eu_add}
\end{figure}

\end{document}